\documentclass[AMA,Times1COL]{WileyNJDv5} 

\articletype{}%

\received{}
\revised{}
\accepted{}
\journal{}
\volume{}
\copyyear{}
\startpage{1}

\usepackage{graphicx}
\usepackage[binary-units=true]{siunitx}
\usepackage{times}
\usepackage{amsmath,amssymb}
\usepackage{graphicx}
\usepackage{bm}

\usepackage[inline]{enumitem}
\usepackage{array}
\usepackage{multicol}
\usepackage{multirow}
\usepackage{tabulary}
\usepackage{longtable}
\usepackage{booktabs}
\usepackage{caption}
\usepackage{subcaption}
\usepackage{hhline}

\usepackage{todonotes}
\usepackage{algorithm}
\usepackage{algpseudocode}

\usepackage{tcolorbox}
\usepackage{mathtools}
\usepackage{tikz}
\usetikzlibrary{arrows}
\usetikzlibrary{calc}
\usetikzlibrary{arrows,decorations.markings}
\usepackage{tkz-euclide}
\usepackage{nomencl}
\usepackage{pgfplots}
\usetikzlibrary{arrows.meta}

\DeclareMathOperator*{\mini}{min.}
\DeclareMathOperator*{\lexmin}{lexmin.}

\definecolor{light-gray}{gray}{0.95}
\definecolor{dark-gray}{gray}{0.5}
\definecolor{mygray}{gray}{0.75}
\newcommand{\BIN}{\begin{bmatrix}}
\newcommand{\BOUT}{\end{bmatrix}}

\newcommand{\hw}{\hat{w}}

\newcommand{\tK}{\tilde{K}}
\newcommand{\tL}{\tilde{\mathcal{L}}}

\newcommand{\hz}{\hat{z}}

\newcommand{\bb}{\breve{b}}

\newcommand{\tA}{\tilde{A}}

\newcommand{\mA}{{\mathcal{A}}}

\newcommand{\mI}{{\mathcal{I}}}

\newcommand{\bE}{{\mathbb{E}}}
\newcommand{\bI}{{\mathbb{I}}}

\pgfdeclarelayer{bg1}   
\pgfdeclarelayer{bg}  
\pgfsetlayers{bg1,bg,main}  

\definecolor{orange}{rgb}{0.99,0.69,0.07}
\definecolor{lightgray}{gray}{0.85}
\definecolor{light-gray}{gray}{0.95}
\definecolor{dark-gray}{gray}{0.5}

\tikzset{cross/.style={cross out, draw=black, minimum size=2*(#1-\pgflinewidth), inner sep=0pt, outer sep=0pt},
cross/.default={1pt}}

 \makeatletter
 \newcommand\fs@spaceruled{\def\@fs@cfont{\bfseries}\let\@fs@capt\floatc@ruled
   \def\@fs@pre{\vspace{5pt}\hrule height.8pt depth0pt \kern2pt}%
   \def\@fs@post{\kern2pt\hrule\relax}%
   \def\@fs@mid{\kern2pt\hrule\kern2pt}%
   \let\@fs@iftopcapt\iftrue}
 \makeatother

\DeclareMathOperator*{\argmin}{arg\,min}

\begin{document}

\title{Efficient Lexicographic Optimization for Prioritized Robot Control and Planning}

\author[1]{Kai Pfeiffer}

\author[2,3]{Abderrahmane Kheddar}

\authormark{Pfeiffer \textsc{et al.}}
\titlemark{Efficient Lexicographic Optimization for Prioritized Robot  Control and Planning}

\address[1]{\orgdiv{Schaeffler Hub for Advanced Research}, \orgname{School of Mechanical and Aerospace Engineering}, \orgaddress{\state{Nanyang Technological University}, \country{Singapore}}}

\address[2]{\orgdiv{Joint Robotics Laboratory (JRL) UMI3218/RL}, \orgname{CNRS/AIST}, \orgaddress{\state{Tsukuba}, \country{Japan}}}

\address[3]{\orgdiv{Interactive Digital Human}, \orgname{University of Montpellier}, \orgaddress{\state{CNRS, LIRMM, UMR5506}, \city{Montpellier}, \country{France}}}

\corres{Kai Pfeiffer. \email{kaipfeifferrobotics@gmail.com}}



\abstract[Abstract]{In this work, we present several tools for efficient sequential hierarchical least-squares programming (S-HLSP) for lexicographical optimization tailored to robot control and planning. As its main step, S-HLSP relies on approximations of the original non-linear hierarchical least-squares programming (NL-HLSP) to a hierarchical least-squares programming (HLSP) by the hierarchical Newton's method or the hierarchical Gauss-Newton algorithm. We present a threshold adaptation strategy for appropriate switches between the two. This ensures optimality of infeasible constraints, promotes numerical stability when solving the HLSP's and enhances optimality of lower priority levels by avoiding regularized local minima. 
	We introduce the solver  {\hbox{$\mathcal{N}$\hspace{-2pt}ADM$_2$}}, an alternating direction method of multipliers for HLSP based on nullspace projections of active constraints. The required basis of nullspace of the active constraints is provided by a computationally efficient turnback algorithm for system dynamics discretized by the Euler method. It is based on an upper bound on the bandwidth of linearly independent column subsets within the linearized constraint matrices. Importantly, an expensive initial rank-revealing matrix factorization is unnecessary. We show how the high sparsity of the basis in the fully-actuated case can be preserved in the under-actuated case.
	\hbox{$\mathcal{N}$\hspace{-2pt}ADM$_2$} consistently shows faster computations times than competing off-the-shelf solvers on NL-HLSP composed of test-functions and whole-body trajectory optimization for fully-actuated and under-actuated robotic systems. 
	We demonstrate how the inherently lower accuracy solutions of the alternating direction method of multipliers can be used to warm-start the non-linear solver for efficient computation of high accuracy solutions to non-linear hierarchical least-squares programs.}


\jnlcitation{\cname{%
\author{Pfeiffer K.},
\author{Kheddar A.}}.
\ctitle{Efficient Lexicographic Optimization for Prioritized Robot Control and Planning} \cjournal{\it J Comput Phys.} \cvol{2024;00(00):1--18}.}

\maketitle

\renewcommand\thefootnote{}
\footnotetext{\textbf{Abbreviations:} HLSP, hierarchical least-squares programming; NL-HLSP, non-linear hierarchical least-squares programming; S-HLSP, sequential hierarchical least-squares programming.}

\renewcommand\thefootnote{\fnsymbol{footnote}}
\setcounter{footnote}{1}

\section{Introduction} 

\subsection{Context and contribution}
Lexicographic multi-objective optimization (LMOO) is the hierarchical stacking of $p$ optimization problems~\cite{Sherali1983} ($\lexmin$: lexicographically minimize)
	\begin{align}
	\lexmin_{x,v}\qquad & \Vert v_{\mathbb{C}_1}\Vert_g, \dots , \Vert v_{\mathbb{C}_p}\Vert_g
	\label{eq:lmoo}\tag{LMOO}\\
	\text{s.t}\qquad& f_{\mathbb{C}_{\cup p}}(x) \leqq v_{\mathbb{C}_{\cup p}}\nonumber
\end{align}
The symbol $\leqq$ summarily describes equality and inequality constraints in the constraint set $\mathbb{C}$. The symbol $\cup$ represents the union of constraint sets from levels 1 to $p$ as $\mathbb{C}_{\cup p} \coloneqq \mathbb{C}_1 \cup \dots \mathbb{C}_{p}$. The function $f_{\mathbb{C}}(x)\in\mathbb{R}^{\vert\mathbb{C}\vert}$ in dependence of the variable vector $x\in\mathbb{R}^n$ represents the constraint set $\mathbb{C}$.
Such problems are characterized by the optimal infeasibility (slacks $v$) $\Vert v_{\mathbb{C}_{\cup l-1}}^* \Vert_g > 0$ or optimality $ v_{\mathbb{C}_{\cup l-1}}^* = 0$ of higher priority levels 1 to $l-1$, which must be preserved by the lower priority levels $l$ to $p$. The infeasibility is thereby optimal / minimal with respect to some norm $g\geq 1$. Our above formulation of~\ref{eq:lmoo} is a modification of the classical one as described in~\cite{Lai2022} to include (feasible and infeasible) inequality constraints.
 A specific form of~\ref{eq:lmoo} is non-linear hierarchical least-squares programming (NL-HLSP) with $g=2$. NL-HLSP's have been commonly utilized in instantaneous robot feed-forward~\cite{escande2014} and feed-back~\cite{Djeha2023} control. This enables an intuitive control formulation as no weights between different constraints need to be tuned. Furthermore, robot safety and physical stability is enhanced as critical constraints of different importance are strictly separated from control objectives like reaching tasks. 
Instantaneous prioritized robot control can be solved by a real-time and anytime algorithm modification of sequential hierarchical least-squares programming (S-HLSP) for NL-HLSP~\cite{pfeiffer2023}. It utilizes the current hierarchical least-squares programming (HLSP) approximation of the original NL-HLSP to deduce a new robot control step. By utilizing trust-region constraints or regularization, the validity of the approximation is maintained at the current robot state. In the recent work~\cite{pfeiffer2024}, S-HLSP has been leveraged for the resolution of NL-HLSP representing optimal control or trajectory optimization problems. Here, not only one instance, but a longer horizon of the robot control and state is considered. This enables robots to achieve a wide range of motions by reasoning in anticipating fashion about its physical and mechanical limits~\cite{meduri2022}. 
In this work, we propose a fast hierarchical least-squares programming (HLSP) solver based on the alternating direction method of multipliers (ADMM). The proposed HLSP solver is computationally more efficient than other solver methods when the number of iterations is limited. Sparse nullspace projections  are leveraged to eliminate structured constraints like dynamics equations in optimal control scenarios. The nullspace basis is based on an efficient implementation of the turnback algorithm~\cite{kaneko1982} with an upper bound on the bandwidth in case of multiple-shooting transcription by Euler integration. At the same time, we are able to handle multi-stage constraints like regularization of momentum evolution for safe robot control. This is a distinguishing factor to recursive methods like differential dynamic programming (DDP, \cite{ddp1966}). The efficiency of the proposed methods is evaluated on non-linear test-functions and robot trajectory planning.

\subsection{Non-linear programming}
Non-linear programming (NLP) is a broad classification of optimization problems and includes non-linear and smooth convex and non-convex constraints and objectives~\cite{nocedal2006}. It is a special form of~\ref{eq:lmoo} with $p=2$, typically feasible constraints $v_1=0$ and $\vert \mathbb{C}_2\vert = 1$ such that the norm notation can be omitted. Solution approaches typically involve the repeated approximation of the original non-linear program to a simpler one at the current working point. Inequality constraints are typically recast by including penalty terms in the cost function. The primal-dual interior-point method is characterized by a barrier functions with primal penalization towards the boundary of the feasible region~\cite{Forsgren2002}. Exact penalty functions like non-smooth indicator functions have been investigated for example in the context of augmented Lagrangian methods. Here, infeasibilities are infinitely penalized outside of the primal feasible region, and not penalized otherwise~\cite{Hestenes1969}. The augmented Lagrangian can for example be used within the ADMM. It consists of alternating updates of the primal and the dual~\cite{boyd2011}. 
Another solution approach can be found in sequential quadratic programming (SQP)~\cite{boggs1995}. The original non-linear optimality conditions are approximated to second order by Newton's method. The resulting quadratic program (QP) sub-problem is then iteratively solved for a primal and dual sub-step. 
In all methods above, convergence relies on globalization methods to direct the approximate sub-steps in terms of infeasibility reduction and optimality. Filter methods with trust region constraints are popular in SQP~\cite{fletcher2002b}. The trust-region constraint maintains validity of the QP sub-problems by limiting the step-size. On the other hand, line search methods directly curtail the resulting step in order to fulfill for example Armijo's condition~\cite{Armijo1966}. Line search in combination with a filter method has been proposed for the interior-point method~\cite{ipopt}.

\subsection{Prioritized robot control and  planning}
NL-HLSP's can exhibit sparsity patterns, for example resulting from discrete optimal control problems. 
We consider discretization by direct transcription methods, namely numerical integration by the Euler method.
Due to stage-wise variable dependency of the constraints, the resulting optimality conditions exhibit block-diagonal structure. Exploiting this sparsity is critical in order to maintain linear complexity in the control horizon length~\cite{wangboyd2010}. 
Differential dynamic programming methods~\cite{ddp1966} are a popular tool to leverage such sparse problem formulations. They are very efficient due to low bandwidth only dependent on the number of the input controls. Originally developed for unconstrained systems, in recent years many developments have been proposed for constrained ones. These include interior-point method~\cite{Pavlov2021} and augmented Lagrangian~\cite{jallet2023} based approaches.
Prioritized trajectory optimization has been treated for example in~\cite{Geisert2017} which solves a hierarchy of quadratic programs each projected into the nullspace of the previous level. Sparsity of the constraints is exploited by leveraging DDP. The approach in~\cite{Wang2024} explicitly considers the active constraints in order to preserve the hierarchical ordering. Using principles from time-delay systems, linear and robust controls result from quadratic program solutions at each instance of a model-predictive controller. Both approaches can only handle input and state limit constraints. This is in contrast to LMOO as proposed in~\cite{Tazaki2014}, which is based on prioritized Pareto efficiency. However, sparsity is not exploited.
In this work, we propose a sparse solver for prioritized trajectory optimization cast as NL-HLSP under multi-stage equality and inequality constraints. Constraints can involve variables from several stages, for example regularization of the momentum evolution for safe robot control. Such problems can be solved by off-the-shelf sparse non-linear solvers~\cite{ipopt}. However, higher efficiency can be achieved with more dedicated solvers like the aforementioned prioritized solvers based on the reduced Hessian formulation~\cite{pfeiffer2024}. Here, the right choice of nullspace basis reduces the number of variables (dense programming) or non-zeros (sparse programming). This offsets the computational burden of computing a basis of the nullspace. Several sparsity preserving bases of nullspace have been proposed both in the communities of structural mechanics and control theory. The authors in~\cite{topcu1979} described a sparsity preserving basis based on identifying linearly independent sub-sets. These arise due to the limited bandwidth of the blocks. Several improvements have been proposed. The work in~\cite{Gilbert1987} proposes sparsity enhancing improvements. A more efficient computation leveraging columns updates of the linearly independent sub-matrices is described in~\cite{pfeiffer2024}. A control theoretic approach has been developed for example in~\cite{Yang2019}.

\subsection{Overview}
This article is structured as follows. First, we describe the current state-of-the-art and our contributions in non-linear hierarchical least-squares programming and prioritized non-linear optimal control (Sec.~\ref{sec:probdef}). We then introduce a heuristic for adjusting the threshold for second-order information (Sec.~\ref{sec:epsadapt}). This promotes numerical stability when solving the HLSP sub-problems and avoids local regularized minima of lower priority levels. 
We develop the HLSP solver {\hbox{$\mathcal{N}$\hspace{-2pt}ADM$_2$}} based on the ADMM, see Sec.~\ref{sec:nadmm}. We present how to tune parameters of the ADMM for algorithmic efficiency (Sec.~\ref{sec:stepsize}) and make some considerations towards warm-starting (Sec.~\ref{sec:warmstart}) and the computation of dual variables (Sec.~\ref{sec:lagact}). Next, we design an efficient implementation of the turnback algorithm for the computation of nullspace basis of banded matrices. Specifically, we consider dynamics discretized by Euler integration and derive an upper bound on the bandwidth of the resulting nullspace basis (Sec.~\ref{sec:turnback}). We show that the algorithm can be highly parallelized, which is a distinguishing factor in comparison to recursive methods like the DDP (Sec.~\ref{sec:threads}).


\section*{Nomenclature}
\begin{itemize}[align=parleft,labelwidth=2.5cm,itemindent=2.2cm]
	\item[$l$] Current priority level
	\item[{$p$}] Overall number of priority levels, excluding the trust region constraint on $l=0$
	\item [{$n$}] Number of variables
	\item [{$r$}] Rank of matrix 
	\item [{$n_r$}] Number of remaining variables after nullspace projections 
	\item [{$m$}] Number of constraints 
	\item [{$x\in\mathbb{R}^{n}$}] Primal of NL-HLSP 
	\item [{$\Delta x\in\mathbb{R}^{n}$}] Primal of HLSP 
	\item [{$\Delta z\in\mathbb{R}^{n_r}$}] Primal of projected HLSP 
	\item [{$\Delta \hat{z}\in\mathbb{R}^{n_r}$}] Auxiliary primal of projected HLSP 
	\item [{$f(x)\in\mathbb{R}^{m}$}] Non-linear constraint function of variable vector $x\in\mathbb{R}^n$ 
	\item [{$\mathbb{E}_l$}] Set of $m_{\mathbb{E}}$ equality constraints (eq.)   of level $l$
	\item [{$\mathbb{I}_l$}] Set of $m_{\mathbb{I}}$ inequality constraints (eq.)  of level $l$
	\item [{$\mathcal{I}_l$}] Set of $m_{\mathcal{I}}$ inactive inequality constraints (ineq.)  of level $l$ 
	\item [{$\mathcal{A}_l$}] Set of $m_{\mathcal{A}}$ active equality and inequality constraints of level $l$
	\item [{$\mathbb{E}_{\cup l}$ (or ${\mathcal{E}}_{\cup l}$)}] Set union $\mathbb{E}_{\cup l}\coloneqq\bigcup_{i=1}^l \mathbb{E}_i= \mathbb{E}_1 \cup \cdots \cup \mathbb{E}_l$ with $m_{{\mathbb{E}}_{\cup l}}$ constraints
	\item [{$A_{\mathbb{E}}\in\mathbb{R}^{m_{\mathbb{E}}\times n}$}] Matrix  representing a set $\mathbb{E}$  of $m_{\mathbb{E}}$ linear constraints 
	\item [{$b_{\mathbb{E}}\in\mathbb{R}^{m_{\mathbb{E}}}$}] Vector representing a set $\mathbb{E}$ of $m_{\mathbb{E}}$ linear constraints 
	\item [{$\mathcal{N}(A_{\mathcal{A}_{l}})$}] Operator to compute the nullspace basis $Z_{\mA_{l}}$ and the rank $r$ of a matrix $A_{\mathcal{A}_{l}}$
	\item [{$Z_{\mA_{l}}\in\mathbb{R}^{n\times n_r}$}] Nullspace basis of matrix $A_{\mathcal{A}_{l}}\in\mathbb{R}^{m_{\mathcal{A}_{l}}\times n}$ with rank $r$, $n_r = n-r$ and $A_{\mathcal{A}_{l}}Z_{\mA_{l}}=0$ 
	\item [{$N_{l-1}\in\mathbb{R}^{n\times n_r}$}] Accumulated nullspace basis $N_{\mathcal{A}_{\cup l}} = Z_{\mA_{1}} \dots Z_{\mA_{l}}$ 
	\item [{$\tilde{M}\in\mathbb{R}^{m\times n_r}$}] Matrix $\tilde{M} = MN$ projected into the nullspace basis $N\in\mathbb{R}^{n\times n_r}$ of a matrix $A\in\mathbb{R}^{m\times n}$ of rank $r$; $n_r=n-r$  (variable elimination)
	\item [{$v\in\mathbb{R}^{m}$}] Slack variable of NL-HLSP
	\item [{$\hat{v}\in\mathbb{R}^{m}$}] HLSP equivalent of slack variable 
	\item [{$v^*\in\mathbb{R}^{m}$}] Optimal slack variable 
	\item [{$\mathcal{L}$}] Lagrangian 
	\item [{$K$}] Gradient of Lagrangian $K \coloneqq \nabla \mathcal{L}$
	\item [{$H_l$, $\hat{H}_l$}] Hierarchical Lagrangian Hessian, positive definite equivalent
	\item [{$\rho$, $\sigma$}] ADMM step-size parameter
	\item [{$\lambda\in\mathbb{R}^{{m}}$}] Lagrange multiplier  
	\item [{$\upsilon$}] Scaled Lagrange multiplier 
	\item [{$k$}] Outer iteration of NL-HLSP solver 
	\item [{$\iota$}] Inner iteration of HLSP solver 
	\item [{$\rho$}] Trust region radius 
	\item [{$\nu$}] Activation threshold of inequality constraints  
	\item [{$\chi$}] Convergence threshold of S-HLSP
	\item [{$\epsilon_{adaptive,l}$}] Adaptive second-order information (SOI) threshold of level $l$
\end{itemize}

\section{Problem definition and contributions}
\label{sec:probdef}

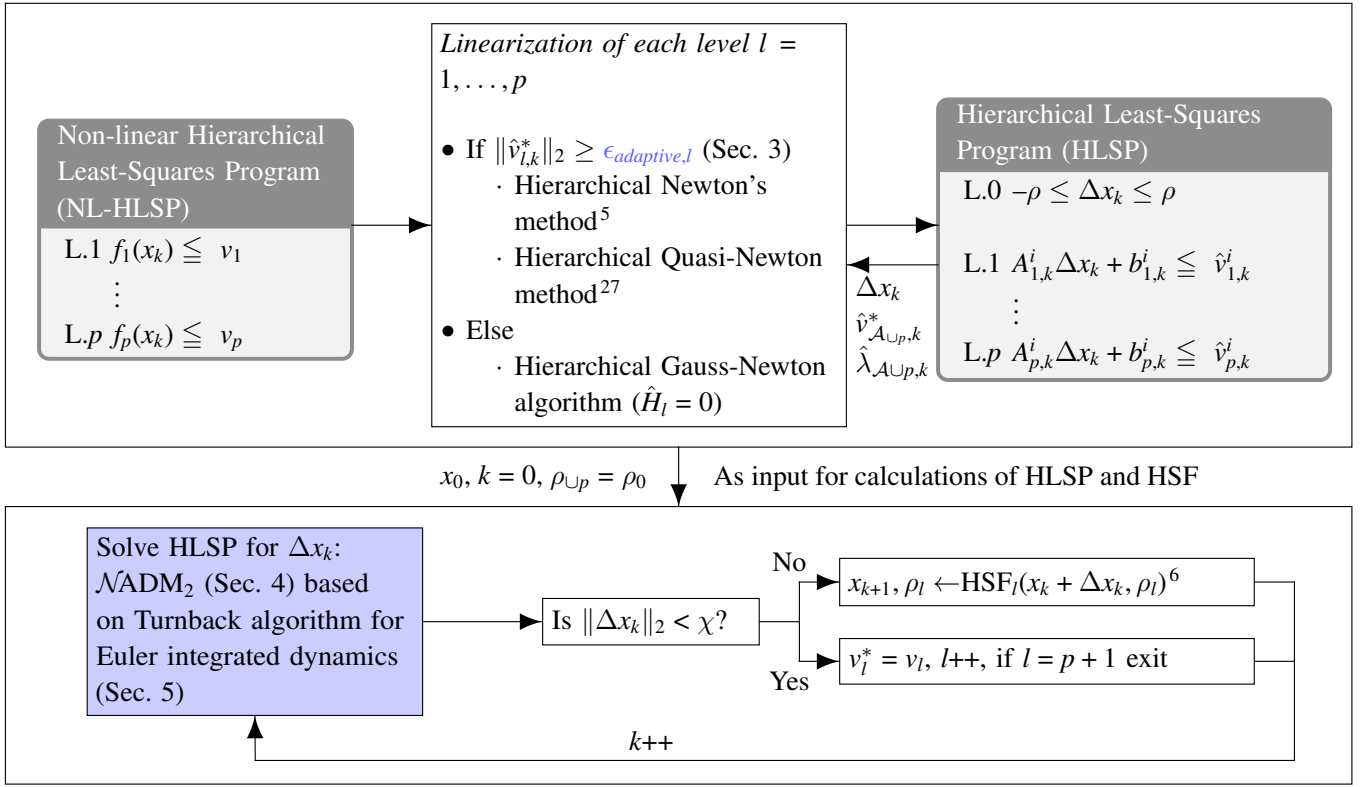
\begin{figure*}[htp!]
	\centering
	\resizebox{1\columnwidth}{!}{%
		\begin{tikzpicture}[line cap=rect]
			
			\node[align=left,text width=4cm] (NLHLSP) at (-0.75,-1) {
				\begin{tcolorbox}[colframe = gray!90,title=Non-linear Hierarchical \\Least-Squares Program\\ \eqref{eq:nlhlsp},boxsep=1pt,left=5pt,right=3pt,top=2pt,bottom=1pt] 
					\begin{enumerate}
						\item[L.$1$] $f_1(x_k) \leqq~v_1$\\
						{\centering$\vdots$}
						\item[L.$p$] $f_p(x_k) \leqq~v_p$
					\end{enumerate}	
				\end{tcolorbox}
			};
			
			\node[draw,align=left,text width=5cm] (Lin) at (4.85,-1) {
				\textit{Linearization of each level $l=1,\dots,p$}\\
				\begin{itemize}[leftmargin=*]
					\item If $\Vert \hat{v}_{l,k}^*\Vert_2 \geq {\color{blue!65}\epsilon_{adaptive,l}}$ (Sec.~\ref{sec:epsadapt})
					\begin{itemize}
						\item {{Hierarchical Newton's method~\cite{pfeiffer2023}}}
						\item Hierarchical Quasi-Newton method~\cite{pfeiffer2018}
					\end{itemize}
					\item Else 
					\begin{itemize}
						\item {Hierarchical Gauss-Newton algorithm ($\hat{H}_l = 0$)}
					\end{itemize}
				\end{itemize}
			};

			\node[text width=5cm] (HLSP) at (11.1,-1) {
				\begin{tcolorbox}[colframe = gray!90,title=Hierarchical Least-Squares\\ Program \eqref{eq:hlsp},boxsep=1pt,left=5pt,right=3pt,top=2pt,bottom=1pt] 
					\begin{enumerate}
						\item[L.$0$] $-\rho\leq\Delta x_k \leq \rho$\\
						\item[L.$1$] $A_{1,k}^i\Delta x_k+b_{1,k}^i\leqq~\hat{v}_{1,k}^i$\\
						{\centering$\vdots$}
						\item[L.$p$] $A_{p,k}^i\Delta x_k + b_{p,k}^i  \leqq~\hat{v}_{p,k}^i$
					\end{enumerate}	
				\end{tcolorbox}
			};
			
			\node[align=left] (feedback) at ($(HLSP.west) + (-0.425,-1.3)$) {$\Delta x_k$\\$\hat{v}_{\mA_{\cup p},k}^*$\\$\hat{\lambda}_{\mA{\cup p},k}$};
			
			\coordinate (b1)   at ($(Lin.north) + (-8,0.25)$);
			\coordinate (b2)   at ($(Lin.north) + (9,0.25)$);
			\coordinate (b3)   at ($(Lin.south) + (9,-0.25)$);
			\coordinate (b4)   at ($(Lin.south) + (-8,-0.25)$);	
			
			\draw[decoration={markings,mark=at position 1 with
				{\arrow[scale=2,>=latex]{>}}},postaction={decorate}] ($(NLHLSP.east) + (-0.15,0)$) -- (Lin.west);
			\draw[decoration={markings,mark=at position 1 with
				{\arrow[scale=2,>=latex]{>}}},postaction={decorate}] (Lin.east) -- ($(HLSP.west) + (0.15,0)$);
			\draw[decoration={markings,mark=at position 1 with
				{\arrow[scale=2,>=latex]{>}}},postaction={decorate}] ($(HLSP.west) - (-0.15,0.5)$) -- ($(Lin.east) - (0,0.5)$);	
			
			\node[] (input) at ($(Lin.south) + (-1.2,-0.65)$) { 
				$x_0$, $k=0$, $\rho_{\cup p} = \rho_0$
			};
			
			\node[] (input) at ($(Lin.south) + (4.,-0.65)$) { 
				As input for calculations of HLSP and HSF
			};
			
			\draw[] (b1) -- (b2) -- (b3) -- (b4) -- (b1);
			
			\node[draw,text width=4cm,fill=blue!20] (hlspsolve) at (-0.,-6) {
				{Solve HLSP for $\Delta x_k$:\\
					$\mathcal{N}$\hspace{-2pt}ADM$_2$ (Sec.~\ref{sec:nadmm}) based on
					Turnback algorithm 
					for Euler integrated dynamics (Sec.~\ref{sec:turnback})}
			};
			
			\node[draw,text width=2.5cm] (convtest) at ($(hlspsolve)+(5.,0)$) {
				{Is $\Vert\Delta x_k\Vert_2 < \chi$?}
			};		
			
			\node[draw,text width=5cm] (filter) at ($(convtest)+(5.,0.5)$) {
				{$x_{k+1},\rho_l\leftarrow$HSF$_l$($x_k+\Delta x_k,\rho_l$) \cite{pfeiffer2024}}
			};
			
			\node[draw,text width=5cm] (conv) at ($(convtest)+(5.,-0.5)$) {
				{$v_l^*=v_l$, $l$++, if $l = p+1$ exit}
			};

			\coordinate (b11)   at ($(Lin.south) + (-8,-1.)$);
			\coordinate (b12)   at ($(Lin.south) + (9,-1.)$);
			\coordinate (b13)   at ($(Lin.south) + (9,-4.5)$);
			\coordinate (b14)   at ($(Lin.south) + (-8,-4.5)$);	
			
			\coordinate (b15)   at ($(convtest.east) + (0.5,0)$);
			\coordinate (b16)   at ($(b15) + (0,0.5)$);
			\coordinate (b17)   at ($(b15) + (0,-0.5)$);
			
			\coordinate (b18)   at ($(filter.east) + (0.5,-0)$);
			\coordinate (b19)   at ($(conv.east) + (0.5,-0)$);	
			\coordinate (b20)   at ($(conv.east) + (0.5,-1.25)$);	
			\coordinate (b21)   at ($(b20) + (-13.12,0)$);

			\draw[] (b11) -- (b12) -- (b13) -- (b14) -- (b11);	
			
			\draw[decoration={markings,mark=at position 1 with
				{\arrow[scale=2,>=latex]{>}}},postaction={decorate}] ($(Lin.south) + (0.5,-0.25)$) -- ($(Lin.south) + (0.5,-1)$);
			
			\draw[decoration={markings,mark=at position 1 with
				{\arrow[scale=2,>=latex]{>}}},postaction={decorate}] ($(hlspsolve.east)$) -- (convtest.west);
			
			\draw[] ($(convtest.east)$) -- (b15);
			\draw[] (b15) -- (b16);
			\draw[] (b15) -- (b17);
			
			\draw[decoration={markings,mark=at position 1 with
				{\arrow[scale=2,>=latex]{>}}},postaction={decorate}] (b16) -- (filter.west);
			\draw[decoration={markings,mark=at position 1 with
				{\arrow[scale=2,>=latex]{>}}},postaction={decorate}] (b17) -- (conv.west);
			
			\draw[] ($(filter.east)$) -- (b18);
			\draw[] ($(conv.east)$) -- (b19);
			\draw[] (b18) -- (b20);
			\draw[] (b20) -- (b21);
			
			\draw[decoration={markings,mark=at position 1 with
				{\arrow[scale=2,>=latex]{>}}},postaction={decorate}] (b21) -- ($(hlspsolve.south)$);
			
			\node[] (inc) at ($(b21) + (5,0.25)$) { 
				$k$++
			};
			\node[] (inc) at ($(convtest) + (1.75,0.75)$) { No };
			\node[] (inc) at ($(convtest) + (1.75,-0.75)$) { Yes };
	\end{tikzpicture}}
	\caption{A symbolic overview of the sequential hierarchical least-squares programming (S-HLSP) with trust region and hierarchical step-filter (HSF) based on the SQP step-filter~\cite{fletcher2002b} to solve non-linear hierarchical least-squares programmings~\eqref{eq:nlhlsp} with $p$ levels. Our contributions, an adaptive threshold for second-order information and the HLSP sub-problem solver  {\hbox{$\mathcal{N}$\hspace{-2pt}ADM$_2$}} in combination with an efficient turnback algorithm for Euler integrated dynamics, are marked in blue.
	}
	\label{fig:scheme}
\end{figure*}

\subsection{Non-linear Hierarchical Least-Squares Programming}

In this article, we consider non-linear hierarchical least-squares problems (NL-HLSP) as preemptive transcription~\cite{Cococcioni2018} of~\ref{eq:lmoo} with $g=2$:
\begin{align}
	\mini_{{x},v_{\mathbb{E}_{l}},v_{\mathbb{I}_{l}}} \quad & \frac{1}{2} \left\|v_{\mathbb{E}_{l}}\right\|^2_2 + \frac{1}{2} \left\|v_{\mathbb{I}_{l}}\right\|^2_2 \qquad\quad\hspace{8pt}l = 1,\dots,p
	\nonumber\\
	\mbox{s.t.} \quad & f_{\mathbb{E}_{l}}(x) = v_{\mathbb{E}_{l}}\nonumber\\
	\qquad& f_{\mathbb{I}_{l}}(x) \leq v_{\mathbb{I}_l}\nonumber\\
	\qquad& f_{\mA_{\cup l-1}}(x) = v_{\mA_{\cup l-1}}^*\nonumber\\
	\qquad& f_{\mI_{\cup l-1}}(x) \leq 0\label{eq:nlhlsp}\tag{NL-HLSP}
\end{align}
Each problem corresponding to the levels $l=1,\dots,p$ is solved in order. At convergence of each level $l$ at the primal $x^*_l\in\mathbb{R}^{n}$, the feasible $v_l^*\in\mathbb{R}^{m_l} = 0$ or optimally infeasible points $v_l^* \neq 0$ of the sets of equality and inequality constraints $\vert\bE_{l}\vert = {m_{\bE_l}}$ and $\vert\bI_{l}\vert = {m_{\bI_l}}$ is identified. 
The slack variables $v_{\mA_{\cup l-1}}^*$ are  the optimal ones identified for the higher priority levels $1$ to $l-1$. 
The \textit{active set} $\mA_{\cup l-1}$ contains all constraints that are active at convergence of levels $1$ to $l-1$. The active set includes all equality constraints $\bE_{\cup_{l-1}}$, and furthermore all violated / infeasible ($v^* > 0$) or saturated ($v* = 0$) inequality constraints of $\bI_l$. In a similar vein, the \textit{inactive set} $\mI_{\cup l-1}$ contains all the remaining feasible inequality constraints ($v* = 0$) of the sets $\bI_{\cup l-1}$. 

\textbf{Contribution:}
we make some considerations towards resolving limitations of NL-HLSP's. Specifically, we demonstrate how we can identify local minima associated with negative function values, see Sec.~\ref{sec:eval:nonlinopt}.

\subsection{Sequential Hierarchical Least-Squares Programming}
\label{sec:introshlsp}

Sequential hierarchical least-squares programming (S-HLSP) is a method to resolve NL-HLSP. Here, the NL-HLSP is linearized around the current working point $x_k$ at every \textit{outer} iteration $k$ to a HLSP by virtue of the hierarchical Newton's method~\cite{pfeiffer2023}.
Hierarchical least-squares programs (HLSP) are problems of the form
\begin{align}
	\mini_{\Delta x,\hat{v}_{\mathbb{E}_l},\hat{v}_{\mathbb{I}_l}}& \qquad \frac{1}{2}\Vert \hat{v}_{\mathbb{E}_l} \Vert^2_2 + \frac{1}{2}\Vert \hat{v}_{\mathbb{I}_l} \Vert^2_2\qquad l=1,\dots,p \nonumber\\
	\text{s.t.}
	& \qquad A_{\mathbb{E}_l}\Delta x - b_{\mathbb{E}_l} = \hat{v}_{\mathbb{E}_l}\nonumber\\
	& \qquad A_{\mathbb{I}_l}\Delta x - b_{\mathbb{I}_l} \leq \hat{v}_{\mathbb{I}_l}\nonumber\\
	& \qquad A_{\mA_{\cup l-1}}\Delta x - b_{\mA_{\cup l-1}} = \hat{v}_{\mA_{\cup l-1}}^*\nonumber\\
	& \qquad A_{\mI_{\cup l-1}}\Delta x - b_{\mI_{\cup l-1}} \leq 0\label{eq:hlsp}\tag{HLSP}
\end{align}
Variables $\hat{\cdot}$ are the linear equivalents to the non-linear ones of the~\ref{eq:nlhlsp}.
Notably, the problem constraints are linear. The constraint matrices and vectors $A$ and $b$ represent this linearization (Jacobians and Hessians) of non-linear constraints~$f$. 

Utilizing Fletcher's filter method, the resulting primal steps from the HLSP sub-problems are accepted or rejected, depending on sufficient progress in terms of constraint infeasibility reduction and optimality.
The HLSP sub-problems are subject to a trust-region constraint in order to maintain the validity of the approximation. The trust-region radius is increased or decreased depending on step acceptance and rejection, respectively.
Linearization methods of the NL-HLSP to HLSP include the hierarchical Newton's method (using second order information (SOI) in form of the hierarchical Hessian) or the hierarchical Gauss-Newton algorithm (no SOI)~\cite{pfeiffer2023}. A switch between the two is decided upon the residual of the HLSP sub-problem. 

\textbf{Contribution:} An overview of S-HLSP is given in Fig.~\ref{fig:scheme}. In this work, we propose an adaptive thresholding strategy for SOI augmentation in the hierarchical Newton's method (Sec.~\ref{sec:epsadapt}). This promotes numerical stability when solving the HLSP sub-problems and enhances solution optimality of lower priority levels.
Furthermore, an efficient solver for HLSP based on the ADMM is presented (Sec.~\ref{sec:nadmm}). As a first-order method, the solver primarily relies on matrix-vector multiplications instead of matrix factorizations as for the IPM. This solver is efficient in approximating a solution of moderate accuracy with a limited number of iterations with respect to its IPM equivalent. The approximate primal guess then can be used to warm-start a high accuracy solver as we demonstrate in Sec.~\ref{sec:eval:nonlinopt}.

\subsection{Prioritized trajectory optimization}

A specific form of NL-HLSP's are prioritized non-linear trajectory optimization problems of the form
\begin{align}
	\mini_{x_T,v_{\cup l,t}} &\quad  \frac{1}{2} \left\| {v}_{\cup l,t} \right\|^2_2 \qquad l = 1,\dots,p  \label{eq:oc}\tag{PTO}\\
	\mbox{s.t.}  &\quad  f_l(x_{s_{t}:e_{t}}) \hspace{3pt} \leqq \hspace{3pt} {v}_{l,t} 	\qquad t = 0,\dots,T\nonumber\\
	&\quad  {f}_{\cup l-1}(x_{s_{t}:e_{t}}) \hspace{3pt} \leqq \hspace{3pt} {v}_{\cup l-1,t}^* \nonumber
\end{align}
Here, ${f}_{\cup l-1}$ represents equality and inequality constraints of lower priority levels $1$ to $l-1$, which is indicated by the symbol $\leqq$ (note that in the case of inactive inequality constraints $\mI_{\cup l-1}$, we have ${v}_{\cup l-1,t}^*=0$).
$T$ is the length of the control horizon. The individual time steps $t=0,\dots, T$ are also referred to as stages.
Constraints only depend on specific variable segments / intervals $x_{{s_t:e_t}}\coloneqq x\left[s_{t_0}:e_{t_1}\right]$. The indices $s_{t_0} \leq e_{t_1}$ with $t_0 \leq t_1$ are the start and end indices of the segments in $x$ corresponding to time steps $t_0$ and $t_1$.
The constraint Jacobians $J$ therefore exhibit a banded structure as follows
\begin{align}
	\setlength\arraycolsep{0pt}
	&J =
	\BIN \nabla_{x_{\left[0\right]}} f(x_{s_{0}:e_{1}}) & \nabla_{x_{\left[1\right]}} f(x_{s_{0}:e_{1}}) & \cdots & 0 \\
	0 & \nabla_{x_{\left[1\right]}} f(x_{s_1:e_2}) & \cdots & 0 \\
	\vdots & \vdots & \ddots & \vdots \\
	0 & 0 & \cdots & \nabla_{x_{\left[T\right]}} f(x_{s_{T-1}:e_{T}})
	\BOUT\label{eq:blockA}
\end{align}
$\left[t\right]$ indicates the interval $\left[s_t:e_t\right]$.
This optimal control problem structure handles multi-stage constraints (excluding the dynamics), unlike DDP.  Typically, we have some initial condition on parts of the variable vector $x_{:} = x_{:,0}$ with constant $x_{:,0}$, which can be seamlessly integrated as a high priority constraint.

\textbf{Contribution:} In this work, we present a sparse nullspace basis based on the turnback algorithm  for Euler integrated dynamics (Sec.~\ref{sec:turnback}). We provide an upper bound on its bandwidth. This enables us to design an efficient turnback algorithm which does not rely on a costly initial rank-revealing matrix factorization. We demonstrate how the high degree of sparsity in the case of full actuation can be transferred to the case of under-actuation

\section{Hierarchical step-filter with adaptive threshold for second order information}
\label{sec:epsadapt}

S-HLSP utilizes the hierarchical Newton's method~\cite{pfeiffer2023} (or the Quasi-Newton equivalent~\cite{pfeiffer2018}) or the hierarchical Gauss-Newton algorithm to linearize~\ref{eq:nlhlsp}.
Switching between the two can be based on the principle that at convergence, variables corresponding to infeasible constraints need to be `locked' in the~\ref{eq:hlsp} by a full rank Hessian in order to not disturb the optimal infeasibility of the non-linear constraints. Furthermore, in robotics constraint Jacobians in the~\ref{eq:hlsp} are typically rank deficient at infeasible points due to kinematic and algorithmic singularities~\cite{Chiaverini1997}. The Newton's method and its second-order information (SOI) then acts as a regularization and enables a global solution to the~\ref{eq:hlsp}. At the same time, deactivating SOI promotes solution optimality. SOI is typically full-rank on the variables that the corresponding constraints occupy. Therefore, these variables can not be used any more for the resolution of lower priority levels. If SOI is unnecessarily activated for feasible constraints, this results in less optimal local minima for constraints on lower priority levels.

The switching method proposed in~\cite{pfeiffer2018} adheres to the following strategy
\begin{align}
	H_l &= J^TJ + \text{SOI}_l \quad\text{ if }\Vert \hat{v}_l \Vert_2 \geq \epsilon \text{ (Newton's method)}\label{eq:adapteps}\\
	H_l &= J_l^TJ_l \phantom{+ SOIl} \quad\text{ otherwise (Gauss-Newton algorithm) }\nonumber
\end{align}
SOI is defined as
\begin{equation}
	\text{SOI}_l\coloneqq \hat{H}_l = R_l^TR_l\label{eq:soi}
\end{equation} 
$\hat{H}_l$ represents some positive definite regularization (for example Higham~\cite{Higham1986} or symmetric Schur regularization~\cite{Golub1996}, Broyden-Fletcher-Goldfarb-Shanno algorithm (BFGS)~\cite{Broyden1970}, weighted identity matrices~\cite{more1977}, \dots) of the hierarchical Lagrangian Hessian of a level $l$~\cite{pfeiffer2023}. It involves SOI and approximate Lagrange multipliers of the levels 1 to $l$.
$\epsilon$ is a constant threshold on the linear slacks $\hat{v}$ as an indicator for constraint infeasibility. The linear slacks $\hat{v}$ capture well-posed / compatible HLSP sub-problems and enable SOI deactivation even if an iterate $x^k \neq x^*$. In contrast, the non-linear slack $v_l$ of feasible constraints only vanishes at convergence $x^*$. 
Here, we propose an adaptive strategy for the SOI augmentation thresholds $\epsilon_{adaptive,l}$ of each level $l$ in~\eqref{eq:adapteps}, see Sec.~\ref{sec:soieps}. This avoids manual tuning of the SOI activation threshold which is oftentimes necessary for HLSP sub-problem solvers of different accuracy and in dependency of the problem configurations. The method is based on the HSF for S-HLSP globalization, which is recalled in Sec.~\ref{sec:hsf}.

\subsection{The hierarchical step-filter}
\label{sec:hsf}
The HSF~\cite{pfeiffer2024} based on the SQP step-filter~\cite{fletcher2002b} measures the progress in the approximated HLSP sub-problems with respect to the original NL-HLSP. 
Each filter $\mathcal{F}_l$ of the levels $l=1,\dots,p$ of the~\ref{eq:nlhlsp} consists of pairs $(h_{\cup l-1}$, $\Vert f_l^+\Vert_2^2)$ with
\begin{equation}
	h_{\cup l-1}(x_k+\Delta x_k) = \Vert f_{\bI_{\cup l-1}}^+ - v_{\bI_{\cup l-1}}^*\Vert_1 + \Vert f_{\bE_{\cup l-1}} - v_{\bE_{\cup l-1}}^*\Vert_1
\end{equation} 
Here
\begin{equation}
	f_{l}^+ \coloneqq \BIN f_{\bE_l} \\ \max(0,f_{\bI_l}) \BOUT
\end{equation} 
$h_{\cup l-1}$ reflects feasibility of the constraints while $\Vert f_l^+(x_k+\Delta x_k)\Vert_2^2$ indicates objective optimality. As can be seen, we use the non-linear slacks $v_l = f_l^+(x_k+\Delta x_k)$ instead of the linear ones $\hat{v}_l$ from the HLSP.

A new point $h_{\cup l-1}(x_k + \Delta x_k)$  and $\Vert f_l^+(x_k+\Delta x_k)\Vert_2^2$ resulting from a new primal step $\Delta x_k$ of the HLSP sub-problem  is acceptable to all filter points  $j\in\mathcal{F}_l$ if sufficient progress in feasibility or optimality has been achieved:
\begin{align}
	h_{\cup l-1} \leq \beta h_{\cup l-1}^j \qquad \text{or} \qquad \Vert f_l^+\Vert_2^2 + \gamma h_{\cup l-1} \leq \Vert f_l^{+j}\Vert_2^2
	\label{eq:acc}
\end{align}
$\beta$ is a value close to 1 and $\gamma$ is a value close to zero and adhere to the condition $0 < \gamma < \beta < 1$.
Since the model reliably represents the non-linear problem, the trust region radius is increased. Otherwise, the step is rejected and the trust region radius is reduced. The HSF of level $l$ converges once $\Vert \Delta x_k\Vert_2 <\chi$ falls below the threshold $\chi$. This process is repeated for each priority level $l=1,\dots,p$.

\subsection{Adaptive SOI thresholding}
\label{sec:soieps}
\begin{algorithm}[t!]
	\caption{\tt adaptEps}
	\begin{algorithmic}[1]
		\Statex \textbf{Input:} $\epsilon_{adaptive}$, $(h_{front},\Vert f^+\Vert^2_{2,front})$, $(h,\Vert f^+\Vert^2_{2})$, $c$, accepted, $\kappa$, $\underline{\epsilon}$, $\overline{\epsilon}$
		\Statex \textbf{Output:} $\epsilon_{adaptive}$, $(h_{front},\Vert f^+\Vert^2_{2,front})$, $c$
		\If{accepted}
		\If{$h \leq h_{front} \quad\&\quad \Vert f^+\Vert^2_{2} < \delta\Vert f^+\Vert^2_{2,front}$}
		\State
		$\epsilon_{adaptive} \leftarrow \min(\epsilon_{adaptive} \cdot \kappa,\overline{\epsilon}) $
		\State $h_{front} = h$
		\State $\Vert f^+\Vert^2_{2,front} = \Vert f^+\Vert^2_{2}$
		\State $c=0$
		\EndIf
		\ElsIf{$c > \zeta$}
		\State $\epsilon_{adaptive} \leftarrow  \max(\epsilon_{adaptive} / \kappa,\underline{\epsilon})$
		\EndIf
		\State $c\leftarrow c+1$
		\State \textbf{return} $\epsilon_{adaptive}$, $(h_{front},\Vert f^+\Vert^2_{2,front})$, $c$
	\end{algorithmic}
	\label{alg:epsadapt}
\end{algorithm}

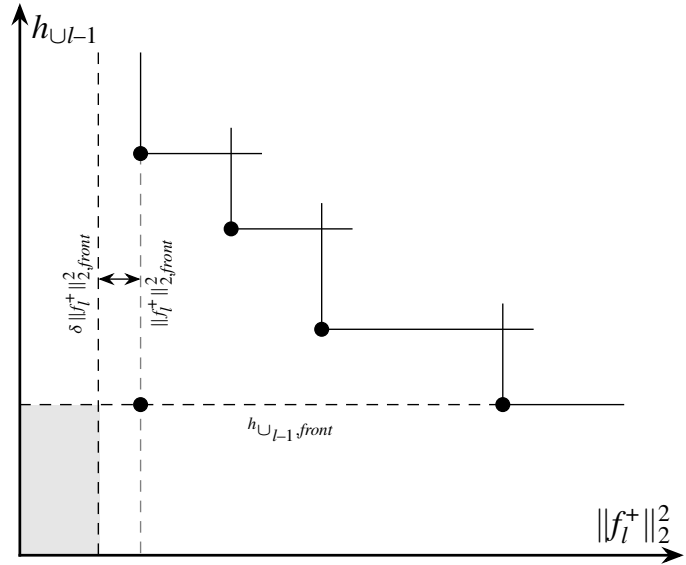
\begin{figure}[htp!]
	\centering
	\resizebox{0.5\columnwidth}{!}{%
		\begin{tikzpicture}[line cap=rect]
			\fill [gray!20] (0,0) rectangle (0.82,1.55);
			\node[rotate=90] at (0.6,2.8) {\tiny$\delta\Vert f^+_l\Vert_{2,front}^2$};
			\node[] at (2.8,1.3) {\tiny$h_{\cup_{l-1},front}$};
			\node[rotate=90] at (1.45,2.8) {\tiny$\Vert f^+_l\Vert_{2,front}^2$};
			
			\begin{axis}[
				axis lines=middle,
				axis line style={-Stealth, thick},
				xmin=0,xmax=5.5,ymin=0,ymax=5.5,
				xtick distance=100,
				ytick distance=100,
				xlabel=$\Vert f_l^+\Vert_2^2$,
				ylabel=$h_{\cup l-1}$,
				xticklabel=\empty,
				yticklabel=\empty,
				grid=major,
				grid style={thin,densely dotted,black!0}]
				\addplot[] coordinates {(1,5.) (1,4.)	(2,4.) };

				\addplot[no markers] coordinates {(1.75,4.25)	(1.75,3.25) };
				\addplot[no markers] coordinates {(1.75,3.25)	(2.75,3.25) };
				
				\addplot[no markers] coordinates {(2.5,3.5)	(2.5,2.25) };
				\addplot[no markers] coordinates {(2.5,2.25)	(4.25,2.25) };

				\addplot[no markers] coordinates {(5,1.5)	(4,1.5) };
				\addplot[no markers] coordinates {(4,1.5)	(4,2.5) };
				
				\addplot[no markers,dashed,gray] coordinates {(1,4)	(1,0) };
				\addplot[no markers,dashed] coordinates {(0.65,5)	(0.65,0) };
				\addplot[no markers,dashed] coordinates {(0,1.5)	(4,1.5) };

				\addplot[mark=*] coordinates {(1,4)  };
				\addplot[mark=*] coordinates {(1.75,3.25) };
				\addplot[mark=*] coordinates {(4,1.5) };
				\addplot[mark=*] coordinates {(2.5,2.25)};
				\addplot[mark=*] coordinates {(1,1.5)};
				
				\addplot[Stealth-Stealth] coordinates {(0.65,2.75) (1,2.75)};
				
			\end{axis}
			
	\end{tikzpicture}}
	\caption{Iterates $(h_{\cup l-1}(x_k),\Vert f_l^+(x_k)\Vert_2$ of level $l$. A new filter front needs to lie within the shaded area.
	}
	\label{fig:front}
\end{figure}

The threshold adaptation strategy for $\epsilon_{adaptive,l}$ is outlined in Alg.~\ref{alg:epsadapt}.
On each level $l=1,\dots,p$ and at every outer iteration, the filter \textit{front} $(h_{\cup{l-1},front}, \Vert f^+_l\Vert^2_{2,front})$ is updated by {\tt adaptEps}($\epsilon_{adaptive,l}$, $(h_{\cup l-1,front},\Vert f^+_l\Vert^2_{2,front})$, $(h_{\cup l-1},\Vert f^+_l\Vert^2_{2})$, $c_l$, $\text{accepted}_l$, $\kappa$, $\underline{\epsilon}$, $\overline{\epsilon}$). It represents the most optimal point by the margin $\delta$ that has been encountered so far on a level $l$.
\begin{definition}[Filter front]
	A point $(h_{\cup{l-1},front}, \Vert f^+_l\Vert^2_{2,front})$ is the front of a filter $\mathcal{F}_l$ to degree $\delta\leq1$, if it dominates all points of previous iterates $x_j$ with $j=1,\dots k$ according to
	\begin{align}
		&h_{\cup{l-1},front} \leq h^j_{\cup{l-1}} \quad\text{and}\quad
		\Vert f^+_l\Vert^2_{2,front} < \delta\Vert f^{+j}_l\Vert_2^2 \quad 
		\label{eq:filtfront}
	\end{align}
\end{definition}
An exemplary visualization is given in Fig.~\ref{fig:front}.
The choice $j=1,\dots k$ (and not considering the filter elements $j\in\mathcal{F}_l$) is motivated by the fact that a filter $\mathcal{F}_l$ does not include all iterates $x_j$ with $j=0,\dots,k$ due to a concept referred to as f-type iteration where the focus is put on optimality of the HSF level $l$, see~\cite{fletcher2002b}. Since this comes possibly at the cost of increase in constraint violation $h_{\cup l-1}$, the threshold adaptation takes place on all levels including the current HSF level $l$. 
We reinitialize the filter front of each level $i=1,\dots,p$ at the start of the step filter of a level $l$ by  $(h_{\cup{l-1}}(x_{l-1}^*),\Vert f^+_l(x_{l-1}^*)\Vert^2_{2})$. $x_{l-1}^*$ is the primal obtained at the KKT point of the previous level $l-1$. This handles cases where $x_0$ is a feasible ($f^+_l=0$), but $x_{l-1}^*$ is an infeasible point ($f^+_l\neq 0$) to constraints, since otherwise no other filter front can be identified (since the condition $\Vert f^+_l\Vert^2_{2} < \Vert f^{+}_l\Vert_{2,front}^2=0$ would need to be fulfilled).

The SOI augmentation threshold is relaxed / increased by a factor $\kappa > 1$ if a sub-step leads to a new filter front. For one, progress towards feasibility $\Vert f^{+}_l\Vert_2^2 = 0$ is required. This is ensured by the degree $\delta \leq 1$. At the same time, the condition $h_{\cup{l-1}} \leq h_{\cup{l-1},front}$ ensures that the current SOI sufficiently encapsulates SOI from constraints of previous levels $f^+_{\cup l-1}$ (since otherwise the ill-posed HLSP sub-step may increase constraint violation $h_{\cup{l-1}}$). 

On the other hand, the adaptation strategy tightens / decreases the threshold if a step is rejected according to~\eqref{eq:acc}. Furthermore, $\zeta$ steps must have been accepted with the current SOI threshold. This promotes trust-region reductions without SOI augmentation in order to escape local regularized minima. This is motivated by the analogies between trust-region methods and the Levenberg-Marquardt method~\cite{more1977} (smaller trust-region radius relates to higher regularization, whereby we prefer smaller trust-region radii without regularization / SOI).

The described procedure leads to a more moderate and slower adaptation strategy than for example directly coupling SOI activations to step acceptances and rejections. At the same time, the proposed heuristic for adapting the SOI threshold gives no guarantee that $0<\Vert\hat{v}_l\Vert_2^2 = \Vert{v}_l\Vert_2^2 < \epsilon_{adaptive,l}$  holds at an optimally infeasible KKT point $x^*$. This can be explained by the fact that by virtue of the  trust-region constraint, such a first-order point $x^*$ can always be obtained (the SQP filter convergence proof in~\cite{fletcher2002b} only requires the norm of the Hessian $\Vert \hat{H}_l\Vert_2$ to be bounded above; this is given for example for physically consistent systems like robots with bounded Jacobians; no rank requirements are made). Nonetheless, we show on test functions (Sec.~\ref{sec:eval:nonlinopt}) that SOI is reliably activated for infeasible constraints even if the initial threshold is chosen far away.

\section{Alternating direction method of multipliers for HLSP}
\label{sec:nadmm}

HLSP solvers both based on the active-set method~\cite{escande2014} and the interior-point method~\cite{pfeiffer2021} have been proposed. While the former is very efficient with little changes of the active-set due to warm-starting capabilities, the latter one exhibits numerical stability and constant computation times even in the case of ill-posed problem formulations. However, both methods rely on expensive matrix factorizations in every inner iteration. In recent years, the ADMM for solving constrained optimization problems has seen a sharp rise in popularity, for example in distribute optimization~\cite{Liu2023}. Here, we outline an ADMM for~\ref{eq:hlsp} which mostly relies on matrix-vector operations. Such first order methods typically approach a solution of moderate accuracy in few iterations~\cite{boyd2011}.
The proposed solver $\mathcal{N}$ADM$_2$ is based on nullspace projections of active constraints as described in Sec.~\ref{sec:nadmm}. We detail our choice of the step-size parameter (Sec.~\ref{sec:stepsize}) and our warm-starting strategy (Sec.~\ref{sec:warmstart}). Our derivations are finally concluded with some considerations regarding the computation of the Lagrange multipliers of the active constraints $\mA_{\cup l-1}$ (Sec.~\ref{sec:lagact}).

\subsection{Reduced Hessian based ADMM for HLSP}

Following the approach in~\cite{dang2017}, which first proposed an ADMM based on the reduced Hessian, we introduce the change of variables
\begin{align}
	\Delta x_l = \Delta x_{l-1}^* + N_{l-1}\Delta z_l
\end{align}
with the overall primal solution
\begin{align}
	\Delta x^*= \sum_{l=1}^p N_{l-1} \Delta z_l^*
	\label{eq:prim}
\end{align}
$N_{l-1}$ is a basis of the nullspace of the active constraints $\mA_{\cup l-1}$ such that 
\begin{equation}
	A_{\mA_{\cup l-1}}N_{l-1} = 0
\end{equation} 
and $N_0 = I_{n\times n}$. The particular solution $\Delta x_{l-1}^*$ (with $\Delta x_0=0$) fulfills the condition $A_{\mA_{\cup l-1}}\Delta x_{l-1}^* - b_{\mA_{\cup l-1}} - v_{\mA_{\cup l-1}}^* = 0$ (and which is obtained during the resolution of the higher priority levels). 
With appropriate choice of the nullspace basis $N$, this leads to either a decrease in variables (dense programming) or non-zeros (sparse programming). In this work, we rely on the turnback algorithm for the computation of sparse nullspace basis of banded matrices, see Sec.~\ref{sec:turnback}.

The change of variables leads to the following projected optimization problem, where $\tilde{A}$ is the projected variable $\tilde{A} = AN$
\begin{align}
	\mini_{ \substack{\Delta z_l,\Delta \hz_l,v_{\mathbb{E}_l},\\v_{\mathbb{I}_l},w_{\mathbb{I}_l},w_{\mI_{\cup l-1}} }}& \quad \frac{1}{2}\Vert  v_{\mathbb{E}_l}\Vert^2_2 + \frac{1}{2} v_{\mathbb{I}_l}^T \Vert^2_2 \nonumber\\
	\text{s.t.}
	& \quad \tA_{\mathbb{E}_l} \Delta z_l - \bb_{\mathbb{E}_l} = v_{\mathbb{E}_l} \qquad l=1,\dots,p \nonumber\\
	& \quad \tA_{\mathbb{I}_l} \Delta z_l - \bb_{\mathbb{I}_l} \leq v_{\mathbb{I}_l}  \nonumber\\
	& \quad \tA_{\mI_{\cup l-1}} \Delta z_l - \bb_{\mI_{\cup l-1}} \leq 0 \label{eq:hlspN}
\end{align}
The vector $\breve{b}_{\Xi}$ represents the expression
\begin{align}
	\breve{b}_{\Xi_l} \coloneqq b_{\Xi_l} - A_{\Xi_l}\Delta x_{l-1}^*
\end{align}
with the corresponding indices ${\Xi}_l = \{\bE_l,\bI_l,\mI_{\cup {l-1}}\}$. 

We introduce the slack variables $w_{\mI_{\cup l-1}}$ and $w_{\mathbb{I}_l}$, similarly to~\cite{pfeiffer2021}. Furthermore, the auxiliary variable $\Delta \hat{z}_l$ is added to the problem as in~\cite{osqp}. The HLSP then writes as
\begin{align}
	\mini_{ \substack{\Delta z_l,\Delta \hz_l,v_{\mathbb{E}_l},\\v_{\mathbb{I}_l},w_{\mathbb{I}_l},w_{\mI_{\cup l-1}} }}& \quad \frac{1}{2}\Vert \BIN v_{\mathbb{E}_l}^T & v_{\mathbb{I}_l}^T\BOUT^T \Vert^2_2 + \mathrm{I}_+(w_{\mathbb{I}_l}) + \mathrm{I}_+(w_{\mI_{\cup l-1}}) \nonumber\\
	\text{s.t.}
	& \quad \tA_{\mathbb{E}_l} \Delta z_l - \bb_{\mathbb{E}_l} = v_{\mathbb{E}_l} \qquad l=1,\dots,p \nonumber\\
	& \quad \tA_{\mathbb{I}_l} \Delta z_l - \bb_{\mathbb{I}_l} = v_{\mathbb{I}_l} + w_{\mathbb{I}_l} \nonumber\\
	& \quad \tA_{\mI_{\cup l-1}} \Delta z_l - \bb_{\mI_{\cup l-1}} =  w_{\mI_{\cup l-1}} \nonumber	\\
	& \quad \Delta \hz_l = \Delta z_l\label{eq:hlspadmm}
\end{align}
The slacks are penalized for negative values by the non-smooth indicator function $\mathrm{I}_+$
\begin{align}
	\mathrm{I}_{+}(w_{\Psi_l}) = 
	\BIN 
	0 \quad w_{\Psi_l} \geq 0 
	\\ +\infty \quad \text{otherwise}
	\BOUT
\end{align}
where 
\begin{equation}
	\Psi_l = \{\mI_{\cup l-1}, \mathbb{I}_{l}\}
\end{equation}
The augmented Lagrangian of level $l$ of~\eqref{eq:hlspN} writes as
\begin{align}
	&\tL_l(\Delta z_l,\Delta \hz_l,v_{\mathbb{E}_l},v_{\mathbb{I}_l},w_{\mathbb{I}_l},w_{\mI_{\cup l-1}})
	= \frac{1}{2}\Vert v_{\mathbb{E}_l} \Vert^2_2 + \frac{1}{2}\Vert v_{\mathbb{I}_l} \Vert^2_2 + \mathrm{I}_+(w_{\mathbb{I}_l}) + \mathrm{I}_+(w_{\mI_{\cup l-1}})
	+ \frac{\rho_{\mathbb{E}_l}}{2} \Vert \tA_{\mathbb{E}_l}\Delta z_l - \bb_{\mathbb{E}_l} - v_{\mathbb{E}_l} + \upsilon_{\mathbb{E}_l} \Vert_2^2\\
	+& \frac{\rho_{\bI_l}}{2} \Vert \tA_{\mathbb{I}_l}\Delta z_l - \bb_{\mathbb{I}_l} -v_{\mathbb{I}_l} - w_{\mathbb{I}_{l}} + \upsilon_{\mathbb{I}_l} \Vert_2^2\nonumber
	+ \frac{\rho_l}{2} \Vert \tA_{\mI_{\cup l-1}}\Delta z_l - \bb_{\mI_{\cup l-1}} - w_{\mI_{\cup l-1}} + \upsilon_{\mI_{\cup l-1}} \Vert_2^2
	+\frac{\sigma}{2} \Vert \Delta \hz_l - \Delta z_l + \sigma^{-1}\lambda_{\Delta z_l}\Vert^2_2\nonumber
\end{align}
where 
\begin{align}
	\upsilon \coloneqq \frac{1}{\rho} \lambda
\end{align}
The step-size parameters $\sigma>0$ and $\rho$, the distinctions and choices $\rho_{\bE_l}\rightarrow \infty$ and $\rho_{\mathbb{I}_l} = \rho_l$ are further explained in Sec.~\ref{sec:stepsize}.
$\lambda$ are the Lagrange multipliers associated with the corresponding problem constraints $\Xi$. 
Resulting from the Karush-Kuhn-Tucker (KKT) first order optimality conditions $\tK_{v_{\mathbb{E}_l}} = 0$ and $\tK_{v_{\mathbb{I}_l}} = 0$ (with $\tK\coloneqq \nabla \tL$), we obtain the primal substitutions
\begin{align}
	v_{\mathbb{E}_l} &= \tA_{\mathbb{E}_l}\Delta z_l - \bb_{\mathbb{E}_l} + \upsilon_{\mathbb{E}_l}\label{eq:subvel}\\
	v_{\mathbb{I}_l} &= \frac{\rho_l}{1+\rho_l}(\tA_{\mathbb{I}_l}\Delta z_l - \bb_{\mathbb{I}_l} -w_{\mathbb{I}_{l}} + \upsilon_{\mathbb{I}_l})\label{eq:subvil}
\end{align}
We then successively compute the alternating steps
\begin{align}
	\Delta\hz^{k+1}_l &\leftarrow \argmin_{\hz_l} \tL_l(\Delta z_l,\Delta \hz_l,v_{\mathbb{E}_l},v_{\mathbb{I}_l},w_{\mathbb{I}_l},w_{\mI_{\cup l-1}}) \\
	\Delta z^{k+1}_l &\leftarrow \alpha \Delta\hz^{k+1} + (1-\alpha)\Delta z^k\\
	v^{k+1}_{\mathbb{E}_l} &\leftarrow \eqref{eq:subvel}\\
	v^{k+1}_{\mathbb{I}_l} &\leftarrow \eqref{eq:subvil}\\
	\hw_{\bI_{l}}^{k+1} &\leftarrow \tA_{\bI_{l}}\Delta \hz^{k+1} - v^{k+1}_{\bI_l}\\
	\hw_{\mI_{\cup l-1}}^{k+1} &\leftarrow \tA_{\mI_{\cup l-1}}\Delta\hz^{k+1} \\
	w^{k+1}_{\Psi} &\leftarrow \max(\bb_{\Psi}, \alpha 	\hw_{\Psi}^{k+1} + (1-\alpha) w_{\Psi}^k + \upsilon_{\Psi}^k)\\
	\upsilon^{k+1}_{\Psi} &\leftarrow u^{k}_{\Psi} + \alpha\hw_{\Psi}^{k+1} + (1-\alpha)w_{\Psi}^k - w^{k+1}_{\Psi}\label{eq:duil}\\
	\upsilon^{k+1}_{\bE_l} &\leftarrow u^{k}_{\bE_l} + \alpha(\tA_{\mathbb{E}_l}\Delta\hz^{k+1} - v^{k+1}_{\bE_l}) + (1-\alpha)\bb_{\bE_l} - \bb_{\bE_l}\label{eq:duel}
\end{align}
The parameter $\alpha\in(0,2)$ is the over-relaxation parameter (typically $\alpha = 1.6$)~\cite{Eckstein1992}. For the computation of the primal $\hz^{k+1}$, we consider the optimality condition $\tK_{\hz_l} = 0$
which leads to the expression
\begin{equation}
	C_l\Delta\hz^{k+1}_l=r_l
\end{equation} 
The positive definite matrix $C_l$ is defined as
\begin{align}
	C_l &= \tA_{\bE_l}^T\tA_{\bE_l} + \frac{\rho_l}{1+\rho_l}\tA_{\bI_l}^T\tA_{\bI_l}
	+ \rho_l\tA_{\mI_{\cup l-1}}^T\tA_{\mI_{\cup l-1}} + \sigma I
	\label{eq:primalstep}
\end{align}
$I$ is an identity matrix.
The right hand side writes as
\begin{align}
	r_l=&
	\tA_{\bE_l}^T(\bb_{\bE_l} - \upsilon_{\bE_l}) + \frac{\rho_{l}}{1+\rho_{l}}\tA_{\bI_l}^T(\bb_{\bI_l} +w_{\bI_{l}} - \upsilon_{\mathbb{I}_l})
	+ {\rho_l}\tA_{\mI_{\cup l-1}}^T(\bb_{\mI_{\cup l-1}} + w_{\mI_{\cup l-1}} - \upsilon_{\mI_{\cup l-1}}) + \sigma z^k_l 	
\end{align}
Once the alternating steps of level $l$ have converged with $\Vert \tK_l\Vert_2<\eta$, the active constraint sets $\mA_{l^*}$ and $\mA_l$ corresponding to $\mI_{\cup l-1}$ and $\bI_{l}$ need to be composed. $\eta$ is a positive numerical threshold. The level $l^*$ is referred to as `virtual' priority level and maintains the prioritization between active sets of $\mI_{\cup l-1}$ and $\bI_{l}$~\cite{pfeiffer2021}. We use the following decision criteria to determine active constraints
\begin{align}
	w_{\mI_{\cup l-1}} < \nu &\qquad \text{and} \qquad	\lambda_{\mI_{\cup l-1}} > \nu  \\
	w_{\mathbb{I}_l} < \nu &\qquad \text{and} \qquad v_{\mathbb{I}_l} < -\nu
\end{align}
The resolution of the HLSP is then continued with the ADMM of the next level $l+1$ projected into the nullspace $N_l$ of the new active set  $\mA_{\cup l} = \mA_{\cup l-1}\cup\mA_l$. The remaining inactive constraints are contained in the updated inactive set $\mI_{\cup l}$.

\subsection{Choice of the step-size parameters $\rho$}
\label{sec:stepsize}

Since equality constraints $\bE_l$ are necessarily active at convergence, we choose $\rho_{\bE_l}\rightarrow\infty$~\cite{Ghadimi2024}. It can be seen that this leads to a more efficient algorithm since the dual update $\upsilon_{\bE_l}$~\eqref{eq:duel} is zero and therefore does not need to be computed. 

Similarly, the choice $\rho_{\bI_l}\rightarrow \infty$ for the inequality constraints $\bI_l$ would render its corresponding equation in~\eqref{eq:duil} obsolete. In this case, the inequality constraints are treated as equalities. Consequently, at ADMM convergence, feasible inequality constraints are saturated (with $\tA_{\bI_l} z_l - \bb_{\bI_l} = 0$) and infeasible constraints are active ($v_{\bI_l} < 0$). However, we noticed that this leads to increased and unnecessary constraint activations due to the limited convergence accuracy of the ADMM (see also Sec.~\ref{sec:warmstart}).
Instead, we set the step-size parameter $\rho_{\bI_l} = \rho_l$ according to~\cite{osqp}.

\subsection{Warm-starting HLSP's}
\label{sec:warmstart}

Oftentimes, a slowly evolving sequence of programs (parametric program) needs to be resolved, for example in the context of S-HLSP (see Sec.~\ref{sec:introshlsp}). In this case, and in contrast to interior-point methods, the ADMM can be easily warm-started, i.e., a good initial guess for the primal and dual variables reduces the number of alternating iterations until convergence.
We store the optimal primal and dual values $z_l^*$, $w_{\Psi_l}^*$ and $\upsilon_{\Psi_l}^*$ and the step-size parameter $\rho^*_l$ after convergence of each level $l=1,\dots,p$. In the next problem instance, the primal and dual variables are then warm-started with these values.
If exactly the same HLSP is solved, our algorithm therefore converges as expected with zero iterations with $\Delta x_{k+1}^* = \Delta x^*_k = \sum_{l=1}^p N_{l-1} \Delta z_{l,k}^*$~\eqref{eq:prim}.

Nonetheless, we observed that by warm-starting the primal and dual variables, constraints previously activated tend to be activated again in the next iteration. Potentially, this is caused by the inherently moderate accuracy of ADMM. This can artificially delay convergence of outer methods like S-HLSP if these constraints are not actually active in the corresponding non-linear program. We therefore reset the primal and dual sub-steps to zero in every new HLSP instance. One argumentation for this procedure is that at S-HLSP convergence, the primal sub-step $\Vert \Delta x\Vert_2 \leq \chi$ vanishes and therefore poses a good initial guess when a non-linear parametric program is solved.

\subsection{Lagrange multipliers of active constraints}
\label{sec:lagact}

Considering the dual ascent step $\upsilon_{\mA_{\cup l-1}} = \upsilon_{\mA_{\cup l-1}} - \nabla_{\upsilon_{\mA_{\cup l-1}}}\mathcal{L}$ for the update of the Lagrange multipliers $\upsilon_{\mA_{\cup l-1}}$,
we can see that the gradient $\nabla_{\upsilon_{\mA_{\cup l-1}}} \mathcal{L} = 0$ since $A_{\mA_{\cup l-1}} \Delta x_0 - b_{\mA_{\cup l-1}} - v_{\mA_{\cup l-1}}^* = 0$ and $A_{\mA_{\cup l-1}}(\Delta x_0 + N_{l-1}\Delta z)- b_{\mA_{\cup l-1}} - v_{\mA_{\cup l-1}}^*= A_{\mA_{\cup l-1}}N_{l-1}\Delta z = 0$ as well. Therefore, the Lagrange multipliers associated with the active constraints  ${\mA_{\cup l-1}}$ (and whose nullspace the problem of level $l$ is projected into) are not updated.

As noted in~\cite{pfeiffer2021}, the Lagrange multipliers of the active constraints are not necessary as none of the other primal or dual variables depend on it. However, the Lagrange multipliers may be needed within a non-linear solver based on Newton's method. Here, the Lagrange multipliers are used for the hierarchical Hessian. We use a fast conjugate gradient method to compute the Lagrange multipliers if required by the non-linear solver.
In case that we use the turnback nullspace bases (see Sec.~\ref{sec:turnback}), we use the $L$ factor of the LU decomposition of $A_{\mA_{\cup l-1}}$ for preconditioning the CG algorithm for accelerated convergence. Note that with the choice of other nullspace basis (for example based on the QR decomposition), matrix factorizations can be re-used for efficient computation of the Lagrange multipliers~\cite{pfeiffer2023}.

\section{Turnback algorithm for Euler integrated dynamics}
\label{sec:turnback}

One critical element of the above nullspace method based HLSP solver is to efficiently compute a basis of the nullspace of the active constraints.
The appropriate choice of the nullspace basis $N$ leads to either a decrease in variables (dense programming) or non-zeros (sparse programming). In this work, we rely on the turnback algorithm for the computation of sparse nullspace basis for banded matrices, which arise in discrete optimal control problems~\cite{gpops}. The main computational step of the turnback algorithm is to determine linearly independent subsets in the matrix $A$, to which a certain number of columns of $A$ is linearly dependent. These columns are then used to compute a basis of the nullspace.
Additionally, in our desired context of~\ref{eq:oc}, it is important to preserve resulting banded structures of the constraints as much as possible. The turnback algorithm is able to do so by considering nullspace vectors which are computed with respect to subsets of the block diagonal matrix instead of the whole one. In this work, we introduce some computational shortcuts to the turnback algorithm tailored to dynamics discretized by Euler integration. Importantly, we avoid a costly initial rank-revealing matrix factorization.

First, we formulate our system dynamics discretized by Euler integration (Sec.~\ref{sec:eulerintdyn}). We then outline the algorithmic details of the original turnback algorithm (Sec.~\ref{sec:tbalg}). It is based on identifying linearly independent column subsets in the matrix $A$. In Sec.~\ref{sec:subsets}, we show how to identify these subsets in the case of Euler integrated dynamics and derive an upper bound on the number of columns in the subsets. This enables us in Sec.~\ref{sec:tbalged} to design an efficient turnback algorithm without the need of an expensive initial rank-revealing matrix factorization. Finally, we address the full-rank property of the resulting basis of nullspace (Sec.~\ref{sec:tbrank}), demonstrate how the high degree of sparsity in the case of full actuation can be transferred to the case of under-actuation (Sec.~\ref{sec:tbua}) and comment on the parallelization of our algorithm (Sec.~\ref{sec:threads}).

\subsection{Euler integrated dynamics}
\label{sec:eulerintdyn}

\begin{figure}[t!]
	\begin{align}
		\nabla_{x}f_{dyn} =&\resizebox{.5\hsize}{!}{$ \left[\begin{array}{@{}c|cc|cc|cc|cc|cc|cc|cc|cc|cc|cc|c@{}}
				\ddots &&&&&&&&&&&&\\
				\hline
				\hdots &&& \bm{E}_{3,t} & \cellcolor{yellow!25}\bm{E}_{4,t} &&&&&&&&& \\
				\hdots &&F_{t}^{ua} &\cellcolor{yellow!25}D_{3,t}^{ua}& D_{4,t}^{ua} &&&&&&&& \\
				\hdots &\cellcolor{blue!25}\bm{B}_t & F_{t} &\cellcolor{yellow!25}D_{3,t}& D_{4,t} &&&&&&&& \\
				\hline
				&&& \bm{E}_{1,t} &\cellcolor{orange!25} \bm{E}_{2,t} &&& \bm{E}_{3,t+1} &\cellcolor{yellow!25}\bm{E}_{4,t+1}&& &&&\\
				&&&\cellcolor{orange!25}D_{1,t}^{ua} & D_{2,t}^{ua} & &  F_{t+1}^{ua} &\cellcolor{yellow!25}D_{3,t+1}^{ua}& D_{4,t+1}^{ua} &&&&&\\
				&&&\cellcolor{orange!25}D_{1,t} & D_{2,t} & \cellcolor{blue!25}\bm{B}_{t+1} & F_{t+1} &\cellcolor{yellow!25}D_{3,t+1}& D_{4,t+1} &&&&& \\
				\hline
				&&&&&&& \bm{E}_{1,t+1} & \cellcolor{orange!25}\bm{E}_{2,t+1} &&& \bm{E}_{3,t+2} &\cellcolor{yellow!25}\bm{E}_{4,t+2}& \\
				&&&&&&&\cellcolor{orange!25}D_{1,t+1}^{ua} & D_{2,t+1}^{ua} & & F_{t+2}^{ua} &\cellcolor{yellow!25}D_{3,t+2}^{ua}&  D_{4,t+2}^{ua} \\
				&&&&&&&\cellcolor{orange!25}D_{1,t+1} & D_{2,t+1} & \cellcolor{blue!25}\bm{B}_{t+2} &  F_{t+2} &\cellcolor{yellow!25}D_{3,t+2}& D_{4,t+2} &\\	
				\hhline{=========||============+}
				&&&&&&&&&&& \bm{E}_{1,t+2} & \cellcolor{orange!25}\bm{E}_{2,t+2}\\
				&&&&&&&&&&&\vdots&\vdots&\ddots
			\end{array}\right]$}\label{eq:dxeid}\tag{DED}\\
		P_t^T\nabla_{x}f_{dyn}Q_t^T =&\resizebox{.5\hsize}{!}{$ 
			\left[\begin{array}{@{}cc|cc|cc||cccc||cc|c@{}}
				\cellcolor{blue!25}\bm{B}_t  &&&&&& {D}_{4,t}&F_{t} &&&&\\
				& \bm{E}_{3,t} & \hphantom{I} &&&&&&&&& \\
				\hline
				& \cellcolor{orange!25}D_{1,t} & \cellcolor{blue!25}\bm{B}_{t+1} &&& {D}_{4,t+1} & D_{2,t} &&F_{t+1} &&&& \\
				& \bm{E}_{1,t} && \bm{E}_{3,t+1} &&& \cellcolor{orange!25}\bm{E}_{2,t} &&&&&\\
				\hline
				&&& \cellcolor{orange!25}D_{1,t+1} & \cellcolor{blue!25}\bm{B}_{t+2} & D_{2,t+1} &&&& F_{t+2} && D_{4,t+2} & \\
				&&& \bm{E}_{1,t+1} && \cellcolor{orange!25}\bm{E}_{2,t+1} &&&&& \bm{E}_{3,t+2} && \\
				\hhline{======||=======}
				&&&&&& {D}_{4,t}^{ua}&F_{t}^{ua}&&&&\\
				&\cellcolor{orange!25}D_{1,t}^{ua} &&&&D_{4,t+1}^{ua}& D_{2,t}^{ua}&&F_{t+1}^{ua}&&&\\
				&&&\cellcolor{orange!25}D_{1,t+1}^{ua} &&D_{2,t+1}^{ua}&&&&F_{t+2}^{ua}&& D_{4,t+2}^{ua} \\
				\hhline{======||=======}
				&&&&&&&&& & \bm{E}_{1,t+2} &\cellcolor{orange!25}\bm{E}_{2,t+2} &\\
				&&&&&&&&&& \vdots & \vdots &\ddots
			\end{array}\right]$}\label{eq:pdxeid}\tag{PDXED}\\
		P_t^T\nabla_{x}f_{dyn}Q_t^T = &\resizebox{.5\hsize}{!}{$
			\left[\begin{array}{@{}cc|cc|cc|cc||cccccc||cc|c@{}}
				\cellcolor{blue!25}\bm{B}_t  &D_{4,t}&&&&&&& \cellcolor{yellow!25}{D}_{3,t}&&F_{t} &&&&\\
				& \cellcolor{yellow!25}\bm{E}_{4,t} & \hphantom{I} &&&&&&\bm{E}_{3,t}&&&&& \\
				\hline
				&&\cellcolor{blue!25}\bm{B}_{t+1}  &D_{4,t+1}&&&&&& \cellcolor{yellow!25}{D}_{3,t+1}&&F_{t+1} &&&&\\
				&&& \cellcolor{yellow!25}\bm{E}_{4,t+1} & \hphantom{I} &&&&\bm{E}_{1,t}&\bm{E}_{3,t+1}&&&&& \\
				\hline
				&&& D_{2,t+1} & \cellcolor{blue!25}\bm{B}_{t+2} &D_{4,t+2}&& \cellcolor{yellow!25}D_{3,t+2}&  &&&&F_{t+2} &&&& \\
				&&&  && \cellcolor{yellow!25}\bm{E}_{4,t+2} &&\bm{E}_{3,t+2}&  &\bm{E}_{1,t+1}&&&&\\
				\hline
				&&&&& D_{2,t+2} & \cellcolor{blue!25}\bm{B}_{t+3} & &&&&&& F_{t+3} &\cellcolor{yellow!25}D_{3,t+3}& D_{4,t+3} & \\
				&&&&&  && \bm{E}_{1,t+2} &&&&&&& \bm{E}_{3,t+3} &\bm{E}_{4,t+3} & \\
				\hhline{======||===========}
				&{D}_{4,t}^{ua}&&&&&&&\cellcolor{yellow!25}{D}_{3,t}^{ua}&&F_{t}^{ua}&&&&\\
				&D_{2,t}^{ua}&&{D}_{4,t+1}^{ua}&&&&&&\cellcolor{yellow!25}{D}_{3,t}^{ua}&&F_{t+1}^{ua}&&&&\\
				&&&D_{2,t+1}^{ua} &&{D}_{4,t+2}^{ua}&&\cellcolor{yellow!25}D_{3,t+2}^{ua}&&&&&F_{t+2}^{ua}&&&\\
				&&&&&D_{2,t+2}^{ua} &&&&&&&&F_{t+3}^{ua}&\cellcolor{yellow!25}D_{3,t+3}^{ua}& D_{4,t+3}^{ua} \\
				\hhline{======||===========}
				&&&&&&&&&&&&& & \bm{E}_{1,t+3} &&\\
				&&&&&&&&&&&&&& \vdots & \vdots &\ddots
			\end{array}\right]$}\label{eq:pdieid}\tag{PDIED}\\
	\end{align}
	\caption{Gradient and permuted gradients of the Euler integrated dynamics. The top matrix shows the un-permuted case. Matrices which only appear in the explicit case and in the implicit case are colored in orange and in yellow, respectively. 
		The control matrices $B$ are colored in blue. Matrices of full column rank according to theorem~\ref{th:band} are printed in bold.
		The middle and bottom matrices show the permuted subsets for $\mu=0$ in the explicit and for $\mu=1$ in the implicit case, respectively. These permuted column subsets are linearly independent to all other columns of $\nabla_x f_{dyn}$.}
\end{figure}

The dynamics of a rigid-body system are described by the inverse dynamics Newton-Euler equations~\cite{Li2013}
\begin{align}
	\mathcal{I}\mathcal{D}(q,\dot{q}, \tau, \gamma) \coloneqq M\ddot{q} =  S^T\tau - V(q,\dot{q}) + J^T\gamma
\end{align}
The joint torques $\tau\in\mathbb{R}^{n_{\tau}}$ and contact forces $\gamma\in\mathbb{R}^{n_{\gamma}}$ are considered the input variables of the system. $S\in\mathbb{R}^{n_{\tau}\times n_q}$ is a full-rank selection matrix describing under-actuation of the system $n_{\tau} < n_q$. The joint angles $q\in\mathbb{R}^{n_{q}}$, velocities $\dot{q}\in\mathbb{R}^{n_{\dot{q}}}$ and accelerations $\ddot{q}\in\mathbb{R}^{n_{\ddot{q}}}$ describe the system state. $M(q)\in\mathbb{R}^{n_q\times n_q}$ is the whole-body inertia matrix. $V(q,\dot{q})\in\mathbb{R}^{n_q}$ describes linear and non-linear force effects like Coriolis, centrifugal, gravitational and frictional forces. The Jacobian $J(q)\in\mathbb{R}^{n_{\gamma}\times n_q}$ is associated with the contact points.
It has been noted in~\cite{carpentier2018} that the inverse dynamics form (explicit joint torques) is computationally advantageous compared to the forward dynamics equations (in contrast to explicit joint accelerations). 

In the following, for visualization purposes, we introduce the change of variables
\begin{equation}
	\tilde{q} = \frac{1}{\Delta t}q \qquad \text{and} \qquad \tilde{\tau} = \Delta t\tau 
\end{equation}
The states $s\in\mathbb{R}^{Tn_s}$ (with $n_{s} = n_q + n_{\dot{q}}$) and controls $u\in\mathbb{R}^{Tn_u}$ (with $n_u=n_{\tau} + n_{\gamma}$) are defined as 
\begin{align}
	s =& \BIN  \tilde{q}_1^T & \dot{q}_1^T & \cdots  & \tilde{q}_{T}^T & \dot{q}_{T}^T \BOUT^T\quad\text{and}\quad
	u = \BIN \tilde{\tau}_0^T & \gamma_0^T & \cdots & \tilde{\tau}_{T-1}^T & \gamma_{T-1}^T & \BOUT^T\label{eq:subs}
\end{align}
We assume known constant $\tilde{q}_0$ and $\dot{q}_0$.

We discretize the dynamics by the direct multiple-shooting method~\cite{Giftthaler2017}, namely by Euler integration. 
The resulting Euler integrated dynamics (ED) write as
\begin{align}
	f_{dyn}(t) &= 	\BIN f_{dyn,1}^T(t)&f_{dyn,2}^T(t)\BOUT^T \coloneqq  s_{t+1} - s_t - \Delta t \dot{s}_{t(+1)} = 
	\BIN \tilde{q}_{t+1} - \tilde{q}_t - \dot{q}_{t(+1)} 	\\
	L_t(\dot{q}_{t+1} - \dot{q}_t) - \Delta t G_t{\mathcal{I}\mathcal{D}}(q_{t(+1)},\dot{q}_{t(+1)},{\tau}_t,\gamma_t)
	\BOUT\label{eq:xeid}\tag{ED}
\end{align}
We set $L_t\coloneqq M_t$, $L_t\coloneqq I$ and $G_t\coloneqq I $, $G_t\coloneqq M_t^{-1}$ in the case of inverse and forward dynamics, respectively. The index $(+1)$ indicates implicit Euler integrated dynamics.
In case of under-actuation $n_{\tau}<n_{\dot{q}}$, the corresponding degrees of freedom (freely swinging pendulum or the `free-flyer' / base of a humanoid robot) are described in linear coordinates (and not for example with quaternions) to facilitate the linear integration scheme above. For the remainder of this work, we therefore assume $n_q = n_{\dot{q}}$. Gimbal lock can be avoided for example as described in~\cite{pfeiffer2023}.

The derivatives of ${\mathcal{I}\mathcal{D}}$ with respect to $q$ and $\dot{q}$ can be computed according to~\cite{Singh2021}. 
Similarly, the first and second order derivatives of a function $f(q)$ with respect to $\tilde{q}$ writes as
\begin{align}
	\partial f(q) / \partial \tilde{q} &= \Delta t\partial f(q) / \partial q\\
	\partial^2 f(q) / \partial \tilde{q}^2 &= \Delta t^2\partial^2 f(q) / \partial q^2
\end{align}
This results in the partial derivatives
\begin{align}
	E_{1,t} \coloneqq& \frac{\partial f_{dyn,1}(t)}{\partial {q}_t}=-I,\quad
	E_{2,t} \coloneqq \frac{\partial f_{dyn,1}(t)}{\partial \dot{q}_t}=-I, \quad
	E_{3,t} \coloneqq \frac{\partial f_{dyn,1}(t)}{\partial {q}_{t+1}}=I,\quad
	E_{4,t} \coloneqq \frac{\partial f_{dyn,1}(t)}{\partial \dot{q}_{t+1}}=-I\\
	B_t\coloneqq&\frac{\partial f_{dyn,2}}{\partial \tilde{\tau}_t} = \frac{\partial f_{dyn,2}}{\partial {\tau}_t} \frac{\partial\tau_t}{\partial\tilde{\tau_t}}= -G_tS^T,\quad
	F_{t}	\coloneqq\frac{\partial f_{dyn,2}(t)}{\partial \gamma_t} = G_tJ_t^T\\
	D_{1,t} \coloneqq& 
	\frac{\partial f_{dyn,2}(t)}{\partial \tilde{q}_t} 
	=
	\frac{\partial f_{dyn,2}(t)}{\partial {q}_t} \frac{\partial q_t}{\partial \tilde{q}_t} =
	\frac{\partial f_{dyn,2}(t)}{\partial q_t} \Delta t, \quad
	D_{2,t} \coloneqq \frac{\partial f_{dyn,2}(t)}{\partial \dot{q}_t} = L_t \hspace{1pt}(-\Delta t\cdots)\\ 
	D_{3,t} \coloneqq&\frac{\partial f_{dyn,2}(t)}{\partial {\tilde{q}}_{t+1}},\quad
	D_{4,t} \coloneqq\frac{\partial f_{dyn,2}(t)}{\partial \dot{q}_{t+1}} = L_t\hspace{1pt}(-\Delta t\cdots)
\end{align}
It can be observed that due to the substitutions~\eqref{eq:subs}, $\Delta t$ does not appear as denominator. This is numerically advantageous for small time steps $\Delta t \ll 1$~s due to better matrix conditioning. Ruiz equilibration $\hat{A} = S_lAS_r$~\cite{ruizscaling} can equally be employed but comes at a higher computational cost. The nullspace basis of the original matrix $A$ becomes $Z = S_r\hat{Z}$ with $\hat{A}\hat{Z}=0$.

\subsection{Turnback algorithm}
\label{sec:tbalg}

The turnback algorithm based on the LU decomposition to compute a nullspace basis for a banded matrix $A\in\mathbb{R}^{m\times n}$ consists of the following steps~\cite{kaneko1982}:
\begin{enumerate}
	\item Compute rank revealing $P^TLUQ^T$ decomposition of $A$ (rank~$r_A)$. Then, $r_Z = n-r_A$.
	\item Determine the index vector $b\in\mathbb{R}^{r_Z}$, which indicates the first non-zero entry of each column of
	\begin{equation}
		Z_{LU} = Q\BIN-U_1^{-1}U_2\\I\BOUT\label{eq:Z}
	\end{equation}
	$Z$ is upper block triangular due to the block-diagonal structure of $A$. 
	\item Determine the {turnback pivot columns} $\pi\in\mathbb{R}^{r_Z}$. They are the row indices of the permuted identity matrix in~\eqref{eq:Z}.
	\item For each index $i=1,\dots,r_Z$ in $b$, add columns to the sub-matrix $G_i\in\mathbb{R}^{n\times r_Z}$ to the right of column $b_i$ of $A$ until linear dependency is detected. The turnback pivot column $\pi(i)$ is not added to the sub-matrix.
	\item Compute the null-vector 
	\begin{align}
		z_i = Q\BIN U_1^{-1}u_2 \\ 0_{\pi(i)-r_A-1} \\ 1 \\ 0_{n - \pi(i)}\BOUT
		\label{eq:luz}
	\end{align}
	$Q$, $U_1$ and $u_2$ result from the LU decomposition of the sub-matrix $G_i$. $u_2$ corresponds to the column $\pi(i)$ of $A$.
\end{enumerate}
The resulting turnback nullspace basis is full-rank since each pivot-column is chosen only once during the submatrix augmentation and therefore has a similar structure to~\eqref{eq:Z} with a permuted identity matrix ensuring full column rank.

\subsection{Subset determination for turnback algorithm}
\label{sec:subsets}

In the following, we derive a conservative bound for the number of columns which are needed for linearly independent sub-sets of the Euler integrated dynamics. We structure the permuted matrices~\eqref{eq:pdxeid} and~\eqref{eq:pdieid} as
\begin{equation}
	G_t\coloneqq\left[\begin{array}{@{}c||c@{}}
		G_{1,t} & G_{2,t} \\
		\hhline{=||=}
		G_{1,t}^{ua} & G_{2,t}^{ua} \\
	\end{array}\right]
	\label{eq:gt}
\end{equation}
The operator $\left\lceil a \right\rceil$ rounds the scalar $a$ to its nearest upper integer.

\begin{theorem}
	If $B_t$ and $E_t$ (or namely, $M_t$ and $S_t^T$) with $t=0,\dots,T$ are of full column rank $r_{B_t} = n_{\tau}$ and $r_{E_t} = n_{q}$, the basis of nullspace of $A\coloneqq \nabla_xf_{dyn}$~\eqref{eq:dxeid} is of rank $r_{Z} = T(n_{\tau}+n_{\gamma})$. The linear independent sub-sets of $A$ are banded within width of $\beta \leq  (2+\mu)Tn_{s} + (3+\mu)(n_{\tau} + n_{\gamma})$. The subset augmentation factor $\mu$ is given by
	\begin{align}
		\mu &= \left(\left\lceil\frac{2n_{ua}}{n_q-n_{ua}}\right\rceil\right) \quad \text{for} \quad  0 \leq n_{ua} < n_q
		\label{eq:mu}
	\end{align}
	\label{th:band}
\end{theorem}

\begin{proof}
	
	\textbf{Full actuation}
	First, we consider the case of computing a nullspace basis of~\ref{eq:dxeid} in the case of full actuation (empty matrices $G_{1,t}^{ua}$, $G_{2,t}^{ua}$ with $n_{ua}=0$). The rank of~\ref{eq:dxeid} is $r_{A} = Tn_{s}$ (number of rows, with full row rank). The dimension of the nullspace basis follows with $r_{Z} = T(n_{s}+n_{\tau}+n_{\gamma}) - Tn_{s} = T(n_{\tau}+n_{\gamma})$ (number of columns minus rank of matrix $A$). 
	The bandwidth can be identified by finding row and column permutations $P_t$  and $Q_t$, such that the column subset corresponding to time step $t$ is permuted to the upper left, see~\ref{eq:pdxeid} for explicit and~\ref{eq:pdieid} for implicit Euler integrated dynamics, respectively.
	The $3(n_{\tau} + n_q)$ leftmost columns $G_{1,t}$ are clearly full rank due to the block-diagonal consisting of full-rank elements $B$ and $E$. The rightmost $n_q + 3n_{\gamma}$ columns $G_{2,t}$ are linearly dependent of them. The linearly independent subsets of~\ref{eq:dxeid} are maximally of length $\beta =  2n_{s} + 3(n_{\tau} + n_{\gamma})$. The above is successively applied to all time steps $t=0,\dots,T$.

	\textbf{Under-actuation} We now consider the case of under-actuation of degree $n_{ua}>0$, such that $n_{\tau} + n_{ua} = n_{\dot{q}}$. 
	In the following, we do not consider turnback pivot columns corresponding to contact forces $\gamma$. These columns are already used in the nullspace basis corresponding to the contact forces itself, while repeated use would violate the full-rank property, see Sec.~\ref{sec:tbrank}.
	
	Considering the permutations~\ref{eq:pdxeid} or~\ref{eq:pdieid}, we see that $G_t$ has at most rank (number of rows of the subset)
	\begin{equation}
		r_{G_t} = (3+\mu)(n_{\tau} + n_q + n_{ua})
	\end{equation}
	The subset augmentation factor $\mu$ adds additional time steps to the sub-set.
	The number of columns is (columns of subset minus the pivot columns that need to be in the linear subset) 
	\begin{equation}
		c_{G_t} = (3+\mu)n_{\tau} + (4+2\mu)n_q
	\end{equation}
	The maximum dimension of the nullspace of $G_t$ is then
	\begin{align}
		n_{G_t} = c_{G_t} - r_{G_t} = (1+\mu)n_q - (3+\mu)n_{ua}\label{eq:ndim}
	\end{align}
	The number of linearly dependent columns within the given subset $n_{G_t}$ needs to be larger than the number of pivot columns (as these are used to form the basis of the nullspace)
	\begin{equation}
		n_{G_t} \geq n_{\tau}
	\end{equation}
	Inserting~\eqref{eq:ndim}, the expression for $\mu$~\eqref{eq:mu} follows.
	
	With this choice of $\mu$, we find a subset which is linear dependent to our $n_{\tau}$ pivot columns. The $n_{\gamma}$ columns corresponding to the contact forces are already linearly dependent of $G_t$ as discussed above. The bandwidth of the linearly independent matrix sub-sets therefore becomes
	\begin{equation}
		\beta = (2+\mu)n_s + (3+\mu)(n_{\tau} + n_{\gamma})\label{eq:beta}
	\end{equation}	
\end{proof}
The augmentation factor $\mu$ is a conservative measure since $r_{G_t}$ is an approximation of the exact rank $\hat{r}_{G_t}$ of $G_t$, with $r_{G_t}\geq\hat{r}_{G_t}$ (therefore, the bandwidth is most likely smaller with $n_{G_t} \leq \hat{n}_{G_t}$). In case of full under-actuation $n_{ua} = n_{\dot{q}}$, the nullspace basis becomes dense with $\mu\rightarrow\infty$ as expected. This means that the system response of each time $t_i$ is fully dependent on the system state at any other given time $t_j$ with $j \neq i$. 

As we show in Sec.~\ref{sec:tbua}, the bands of the turnback nullspace basis for dynamics integrated by the Euler method exhibit internal sparsity patterns. Still, the bandwidth $\beta$~\eqref{eq:beta} is in contrast to an effective bandwidth of $n_{\tau}+n_{\gamma}$ for DDP (neglecting the cost of the forward roll-out for the state calculation requiring operations in $n_s^2$). In future work, we desire to incorporate DDP principles into HLSP for further computational efficiency. Nonetheless, the computation of the nullspace basis can be highly parallelized, as we describe in Sec.~\ref{sec:threads}. This is not possible for DDP due to its recursive nature. Also, the backward recursion would need to be computed for every priority level. In contrast, the projection into the nullspace of the dynamics only needs to be done once. This might be more efficient for a high number of priority levels. Furthermore, multi-stage constraints involving states and controls from several time steps (aside from dynamics constraints) are handled due to the broad optimization point of view. Sparsity of such constraints is preserved by relying on less specialized formulations of the turnback algorithm as described in~\cite{pfeiffer2023}.

\subsection{Turnback algorithm for Euler integrated dynamics}
\label{sec:tbalged}

\algnewcommand{\IIf}[1]{\State\algorithmicif\ #1\ \algorithmicthen}
\algnewcommand{\EndIIf}{\unskip\ \algorithmicend\ \algorithmicif}

\begin{algorithm}[t!]
	\caption{{\tt turnbackParam}}
	\begin{algorithmic}[1]
		\Statex \textbf{Input:} $T$, $n$, $n_{\tau}$, $n_{\gamma}$, $n_{s}$, $n_{ua}$, $\beta$
		\Statex \textbf{Output:} $r_{A}$, $r_{Z}$, $b\in\mathbb{R}^{T}$, $b^+\in\mathbb{R}^{T}$,  $\pi\in\mathbb{R}^{r_{Z}}$
		\State $r_{A} = Tn_{s}$
		\State $r_{Z} = T(n_{\tau}+n_{\gamma})$
		\State $n_{\pi} = 0$
		\State $n_b = 0$
		\For{$t=0:T-1$}
		\State $b(t) = n_{b}$
		\State $b^+(t) = \min(n_{b} + \beta,n)$
		\For{$j=0:n_{\gamma}$}
		\State $\pi(n_{\pi}) = n_{b} + n_{\tau} + j$
		\State $n_{\pi}\leftarrow n_{\pi}+1$
		\EndFor
		\For{$j=0:n_{\dot{q}}-n_{ua}$}
		\If{Explicit Euler integrated dynamics}
		\State $\pi(n_{\pi}) = n_{b} + n_{\tau} + n_{\gamma} + n_q + n_{ua} + j$
		\ElsIf{Implicit Euler integrated dynamics}
		\State $\pi(n_{\pi}) = n_{b} + n_{\tau} + n_{\gamma} + n_{ua} + j$
		\EndIf
		\State $n_{\pi}\leftarrow n_{\pi}+1$
		\EndFor
		\State $n_{b}\leftarrow n_{b} + n_{\tau} + n_{\gamma}  + n_{s}$
		\EndFor
		\State \textbf{return} $b$, $b^+$, $\pi$, $r_{A}$, $r_{Z}$
	\end{algorithmic}
	\label{alg:unrbp}
\end{algorithm}	

Based on theorem~\ref{th:band}, we can implement a computationally efficient version of the  turnback algorithm. Foremost, the linearly dependent column subsets of the matrix can be chosen according to the known bandwidth of $Z$. This means that an initial rank revealing LU decomposition of the matrix is not necessary. Concretely, the indices $b$ indicate the first and $b^+=b+\beta$ the last column of the sub-matrix of $A$. Furthermore, the turnback pivot columns are set as the columns corresponding to $D_{1,t}$ and the last $n_{\tau}$ columns of $M_t$. The reasoning is that the under-actuated part typically describes the free-flyer dynamics of the system which are well conditioned as they represent the full linear and rotational inertia of the system. Note that in theorem~\ref{th:band}, we assume full-rank of $M$. This is typically given for physically consistent systems~\cite{Udwadia2010}.
Algorithm~\ref{alg:unrbp} details the computation of above values. The modified turnback algorithm then consists of following steps:
\begin{enumerate}
	\item $b$, $b^+$, $\pi$ $\leftarrow$ Alg.~\ref{alg:unrbp}.
	\item For each index $t=1,\dots,T$, compute the LU decomposition of the column submatrix $G_t\coloneqq A(\mathcal{C}_t)$ of $A$. The column set $\mathcal{C}_t$ is given by the column range from $b(t)$ to $b^+(t)$ without the turnback pivot columns contained within. This leads to the set $\mathcal{C}_t = \left[b(t),b^+(t)\right]\setminus \pi(i)$ with $i=t(n_s+n_{\tau}+n_{\gamma})+n_s,\dots,(t+1)(n_s+n_{\tau}+n_{\gamma})$.
	\item Compute the null-vector according to~\eqref{eq:luz}.
\end{enumerate}

In case of SOI augmentation $V\coloneqq\BIN \nabla_xf_{dyn}^T & R^T\BOUT^T$, where $R$ is a factor of the hierarchical Hessian $\hat{H} \coloneqq \nabla^2_xf_{dyn}^T = R^TR$~\cite{pfeiffer2023}, we apply following two-step computation of a basis of the nullspace: 
first, $\mathcal{N}_{tb,ed}$ computes a basis of the nullspace  according to the turnback algorithm for Euler integrated dynamics described above ($N_{\nabla_xf_{dyn}}\leftarrow \mathcal{N}_{tb,ed}(\nabla_xf_{dyn})$), and secondly, $\mathcal{N}_{tb}$ according to the turnback algorithm as described in~\cite{pfeiffer2024} ($N_2\leftarrow\mathcal{N}_{tb}(RN_{\nabla_xf_{dyn}})$). We then have $VN_{\nabla_xf_{dyn}}N_2=0$. Note, that $\mathcal{N}_{tb}$ does not provide any sparsity guarantees but has been observed to reliably deliver sparse bases on a wide variety of sparsity patterns~\cite{dang2017}. At the same time, due to the high variable occupancy of the equation of motion, lower levels typically are not resolved anymore since most variables are eliminated by the projections.

\subsection{Full-rank property of turnback nullspace basis}
\label{sec:tbrank}
Due to numerical inaccuracies, it can turn out that the pivot columns of a time step $t$ are linearly independent of the corresponding column sub-matrix $A(\mathcal{C}_t)$ to a small error $\Vert \hat{u}_{2}\Vert_2 \leq \delta$ with $\delta \ll 1$ such that
\begin{equation}
	A(\mathcal{C}_t) = P^T L \BIN U_1 & u_2 \\ 0 & \hat{u}_{2}\BOUT Q^T
\end{equation}
Furthermore, nullspace vectors may have an error higher than a tolerance $\Vert Az\Vert_2 > \delta$.

In these cases, we further augment the sub-matrix with blocks corresponding to timesteps $t^+ > t$ and $t^- < t$ to the `left' and `right' of stage $t$. The full-rank property of the resulting nullspace basis is preserved by not adding columns of $A$ that correspond to turnback pivot columns of lower time-steps $t^- < t$.
The basis of the nullspace then exhibits the following structure (we depict the extreme case of full augmentation)
\begin{align}
	Z_{tb} = 
	\BIN X_{1,1} & X_{1,1} & \hdots & X_{1,T-1} & X_{1,T}\\
	I & \cellcolor{green!25} & \hdots & \cellcolor{green!25}& \cellcolor{green!25}\\
	X_{2,1} & X_{2,2} & \hdots & X_{2,T-1} & X_{2,T}\\
	X_{3,1} & I & \hdots & \cellcolor{green!25} & \cellcolor{green!25} \\
	\vdots & \vdots & \ddots & \vdots & \vdots\\
	X_{n-2,1} & X_{n-2,2} &\hdots & X_{n-2,T-1} & X_{n-2,T}\\
	X_{n-1,1} & X_{n-1,2} & \hdots& I & \cellcolor{green!25}\\
	X_{n,1} & X_{n,2} & \hdots& X_{n,T-1} & I\\
	X_{n,1} & X_{n,2} & \hdots& X_{n,T-1} & X_{n,T}\BOUT
\end{align}
The identity matrices correspond to the turnback pivot columns of $A$. These ensure full-column rank of the turnback nullspace.
Furthermore, it can be easily confirmed that the above is full-rank as 
\begin{itemize}
	\item columns to the left are not a linear combination of each of its columns to the right, as this would destroy the sparsity (green)
	\item columns  to the right are not a linear combination of each of its columns to the left, as they can not eliminate the entries on the same rows as the sparse rows (green).
\end{itemize}

\subsection{Under-actuated systems}
\label{sec:tbua} 

We consider the basis of nullspace of~\eqref{eq:pdxeid} for fully-actuated systems ($n_{ua}=0$) in the case of explicit Euler integrated dynamics
\begin{equation}
	\BIN G_{1,t} & G_{2,t}\BOUT Z_t = 0\quad\text{with}\quad 	Z_t = \BIN -G_{1,t}^{-1} G_{2,t}\\I\BOUT
\end{equation}
Using block-wise inversion~\cite{blockinv} of the matrix $G_{1,t}$ with full-rank and invertible $B$ and $E$ (see theorem~\ref{th:band}), we get

\begin{align}
	&G_{1,t}^{-1} 	G_{2,t} = \\
	&\BIN B_t^{-1}\\
	& E_{3,t}^{-1} &  \\
	& \Upsilon & B_{t+1}^{-1} & \Upsilon & \Upsilon & \Upsilon\\
	& \Upsilon &  & E_{3,t+1}^{-1} & \\
	& \Upsilon &  & \Upsilon & B_{t+2}^{-1} & \Upsilon\\
	& \Upsilon &  & \Upsilon &  & E_{2,t+1}^{-1}
	&
	\BOUT
	\hspace{-2pt}
	\BIN D_{4,t} & F_t\\
	\\
	D_{2,t} &  \\
	E_{2,t}\\
	& \\
	&
	\BOUT 	\hspace{-3pt}= 	\hspace{-3pt}	\BIN
	B_t^{-1}D_{3,t} & B_t^{-1} F_t\\
	\\
	B_{t+1}^{-1}D_{2,t} + \Upsilon E_{2,t} & \\
	{\color{gray!50}E_{3,t+1}^{-1} E_{2,t}}\\
	\Upsilon E_{2,t}\\
	\Upsilon E_{2,t}
	\BOUT\nonumber
\end{align}
$\Upsilon$ are place-holders for dense matrix blocks. Elements in gray are zero blocks. This means that the effective bandwidth of the null-vectors~\eqref{eq:luz} is $\beta - n_s$ in the case of full actuation.
In contrast, such sparsity is not reproducible if $B\in\mathbb{R}^{n_u\times n_u-n_{nua}}$ is not invertible due to under-actuation $n_{ua} > 0$. Instead of sparse block-wise inversion, row permutations of $G_{1,t}$ need to be applied in order to permute invertible pivot elements onto the diagonal. 

In order to avoid this, we adapt the robot dynamics by introducing `virtual' controls $u^*$ such that the modified control matrix $\hat{B} = \BIN B^* & B \BOUT\in\mathbb{R}^{n_{\dot{q}}\times n_{\dot{q}}}$ is full-rank and invertible as in the fully actuated case. At the same time, the virtual controls are set to zero by two sets of inequality constraints $\vert u^*\vert \leq 0$ in order to not influence the robot behavior. The choice of inequality instead of equality constraints prevents that these constraints enter the active set (and create non-zero fill-in as without virtual controls by nullspace projections). This method effectively increases the number of variables but this is offset by the reduced number of non-zeros in the turnback nullspace. Such a scheme has been devised in the context of a sparse nullspace basis for optimal control of linear time-invariant systems~\cite{pfeiffer2021b}.

\subsection{Multi-threaded computation}
\label{sec:threads}

The turnback algorithm can be highly parallelized. In fact, in the case of Euler integrated dynamics and with the availability of $T(n_{\tau} + n_{\gamma})$ threads, each subset of $Z$ could be computed in parallel, casting the effective computational complexity of the turnback algorithm to approximately
$O(\beta^3 + \beta^2)$ with $\beta$ as defined in~\eqref{eq:beta} (in detail: $T$ subsets are factorized in parallel by $T$ threads, and the $T(n_{\tau}+n_{\gamma})$ individual nullvectors of bandwidth $\beta$ are then computed in parallel by $T(n_{\tau}+n_{\gamma})$ threads). This is in contrast to DDP, where $T$ decompositions of complexity $O((n_{\tau}+n_{\gamma})^3)$ need to be computed in sequence. Therefore, a projector based S-HLSP based on a sparsity retaining turnback algorithm may be preferred in the presence of high number of cores (and high number of priority levels, as noted in Sec.~\ref{sec:subsets}).

\section{Evaluation}

We use the presented solver \hbox{$\mathcal{N}$\hspace{-2pt}ADM$_2$} in combination with the turnback algorithm for Euler integrated dynamics within the solver  S-HLSP~\cite{pfeiffer2024} for NL-HLSP. The HLSP solver solves the HLSP sub-problems which arise from the linearization of the NL-HLSP at its current working point $x$.
First, we evaluate the efficiency of the modified turnback algorithm for Euler integrated dynamics (see Sec.~\ref{sec:eval:turnback}). Secondly, S-HLSP in combination with our proposed HLSP solver is run on a hierarchy composed of test-functions, Sec.~\ref{sec:eval:nonlinopt}. 
For one, we investigate whether a high accuracy solution can be efficiently obtained by first approximating an optimal primal point with our proposed solver of lower accuracy. We then switch to a high-accuracy sub-solver and continue the local search. In Sec.~\ref{sec:eval:nonlinopt}, we can see how a high-accuracy solution of a NL-HLSP composed of test-functions is obtained with less computational effort compared to a S-HLSP without an initial primal guess of lower accuracy. 
Furthermore, we evaluate the SOI augmentation threshold strategy developed in Sec.~\ref{sec:epsadapt}.
We then evaluate our methods to the following robot scenarios:
\begin{itemize}
	\item Inverse kinematics of a humanoid robot HRP-2Kai (Sec.~\ref{sec:eval:hrp2})
	\item Time-optimal control of the manipulator UR3e under multi-stage constraint\\ (Sec.~\ref{sec:eval:topm})
	\item Swing-up of inverted pendulum (Sec.~\ref{sec:eval:invpend})
	\item Jump of robot dog Solo12 including multi-stage constraint (Sec.~\ref{sec:eval:solo12}) 
\end{itemize}
The last three examples (inverted pendulum, manipulator and robot dog) are~\ref{eq:oc} where we use the turnback algorithm for dynamics integrated by Euler integration (Sec~\ref{sec:turnback}). The latter two simulations are thereby concerned with under-actuated systems where we follow the developments outlined in Sec.~\ref{sec:tbua}.

The simulations are run on an 11th Gen Intel Core i7-11800H \@ 2.30GHz $\times$ 16 with 23 GB RAM.
The implementations of  \hbox{$\mathcal{N}$\hspace{-2pt}ADM$_2$} and the turnback algorithm are based on the Eigen library~\cite{eigenweb} and implemented in C++.
The matrix $C$~\eqref{eq:primalstep} is factorized ($O(n^3)$) only if the step-size parameter $\rho$ is updated. We use a LDLT decomposition with low non-zero fill-in (for example compared to the QR decomposition). Otherwise, \hbox{$\mathcal{N}$\hspace{-2pt}ADM$_2$} relies on matrix-vector operations ($O(n^2)$) as a first-order method. 
The HLSP solver $\mathcal{N}$\hspace{-1pt}IPM$_2$ based on the IPM (matrix factorizations in every iteration), which we proposed in our previous work~\cite{pfeiffer2021}, is modified by incorporating the proposed turnback nullspace basis for dynamics integrated by the Euler method. We can expect computational advantage for  \hbox{$\mathcal{N}$\hspace{-2pt}ADM$_2$} if 
\begin{equation}
	\iota_{\text{$\mathcal{N}$\hspace{-2pt}ADM$_2$}} +\iota_{\text{$\mathcal{N}$\hspace{-2pt}ADM$_2$},\rho}\cdot n < 	\iota_{\text{$\mathcal{N}$IPM}}\cdot (1+n)
\end{equation}
$\iota$ is the number of inner iterations of the respective solvers. $\iota_{\mathcal{N}\text{\hspace{-2pt}ADM$_2$},\rho}$ is the number of factorization updates of  \hbox{$\mathcal{N}$\hspace{-2pt}ADM$_2$}. $n$ is the number of problem variables.
In case of an increase of the KKT norm (which is not related to a change of $\rho$), we increase the regularization factor $\sigma$ ($\sigma_0=1\cdot 10^{-6}$) and reset $\rho$ ($\rho_0=0.1$). 
The number of inner iterations of  {\hbox{$\mathcal{N}$\hspace{-2pt}ADM$_2$}} is limited to 1500, or 2000 for the test in Sec.~\ref{sec:eval:invpend} . We use the analytical hierarchical Hessian~\cite{pfeiffer2023} as needed for the Newton's method in~Fig.~\ref{fig:scheme}. In the robotics examples, the NL-HLSP's and HLSP's are computed by the pinocchio library~\cite{carpentier2019pinocchio} and the automatic differentiation package CppAD~\cite{cppad}. The Lagrange multipliers are computed according to Sec.~\ref{sec:lagact} by the Conjugate Gradient solver LSQR~\cite{Paige1982}. The turnback algorithm is based on the rank-revealing LU decomposition provided by the library LUSOL~\cite{gill1987}. Note that we only depict the computation times of the HLSP sub-solvers. Due to memory limitations, all simulations except for the turnback algorithm evaluation are run on a single thread. 
\hbox{$\mathcal{N}$\hspace{-2pt}ADM$_2$} and \hbox{$\mathcal{N}$\hspace{-1pt}IPM$_2$} are compared to the hierarchical versions of the off-the-shelf solvers H-MOSEK~\cite{mosek}, H-GUROBI~\cite{gurobi} and H-OSQP~\cite{osqp}. All solvers rely on the same framework for active and inactive set composition.


\subsection{Turnback algorithm for Euler integrated dynamics}
\label{sec:eval:turnback}

\begin{figure*}[t!]
	
	\centering
	\includegraphics[width=.32\textwidth]{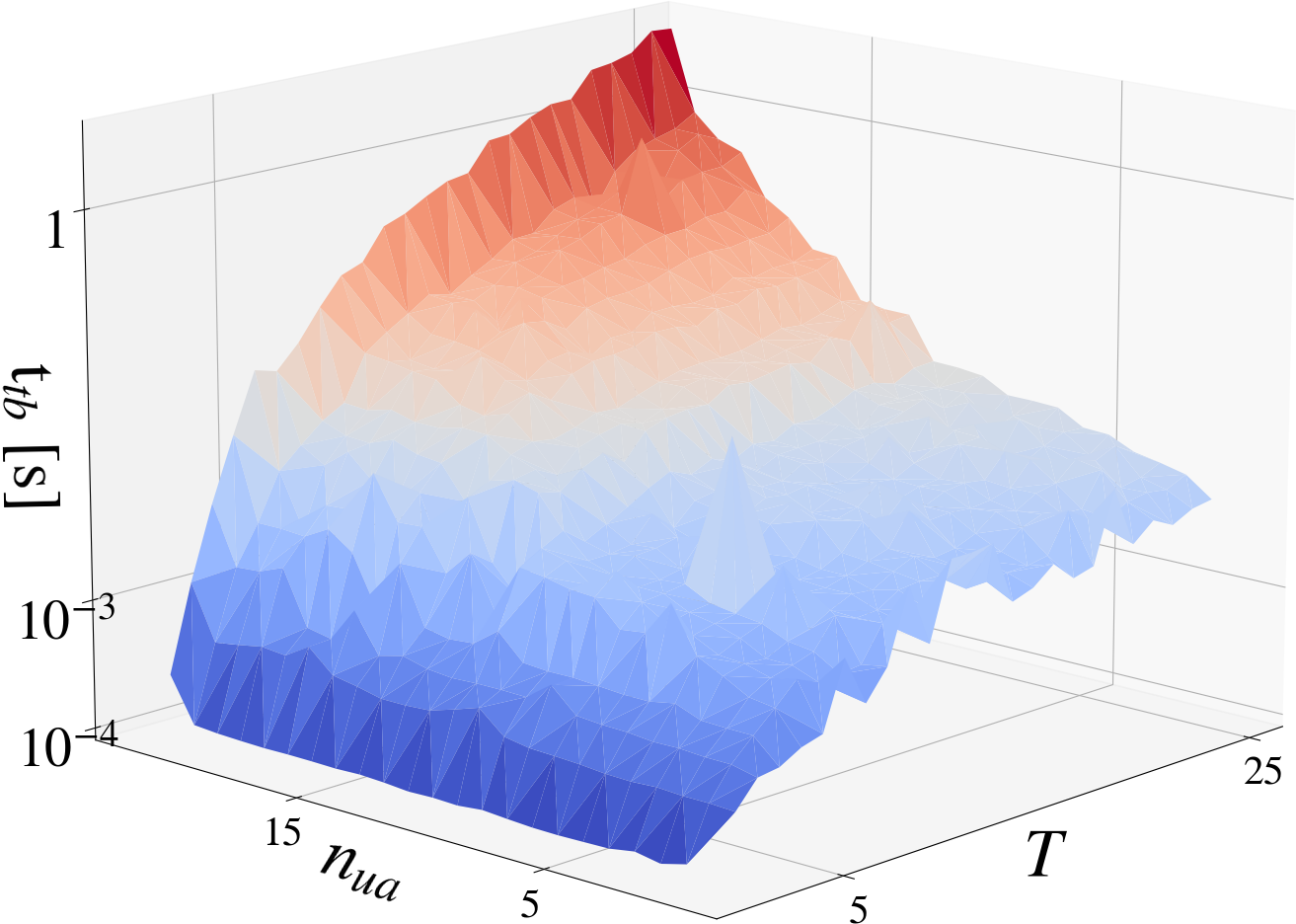}\hfill
	\includegraphics[width=.32\textwidth]{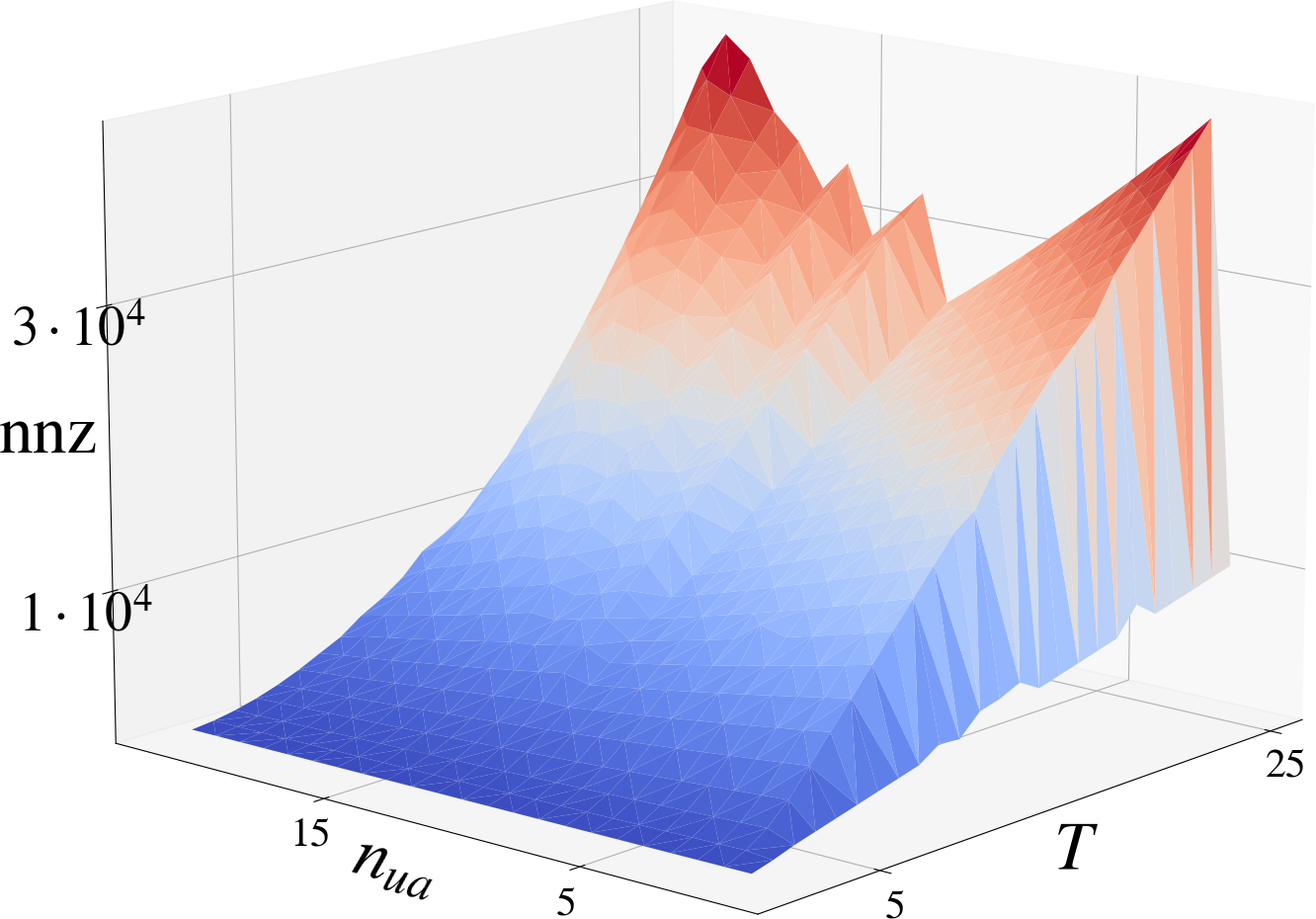}\hfill
	\includegraphics[width=.3\textwidth]{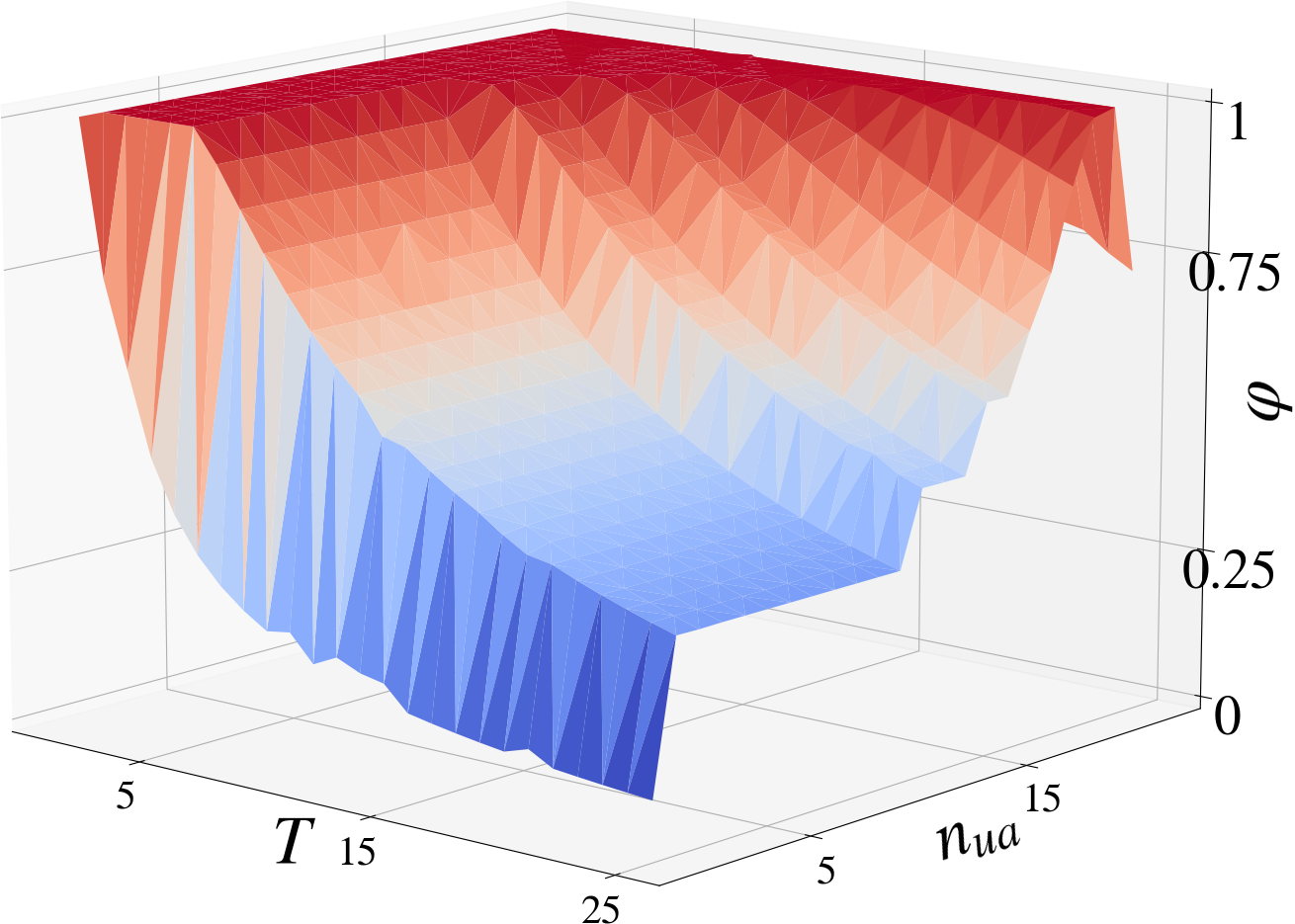}
	
	\caption{Computation time $t_{tb}$, number of non-zeros (nnz) and density ($\varphi$) of $Z^TZ$ of the turnback nullspace $Z$ for Euler integrated dynamics with $n_q = n_{\dot{q}} = 22$ and $n_{\gamma} = 24$ in dependence of control horizon $T$ and under-actuation $n_{ua}$.}
	\label{fig:turnback}
	
\end{figure*}

First, we evaluate the computational efficiency of the turnback algorithm adapted to discrete Euler integrated dynamics. We compose randomized matrices $M$, $D$ and $B$ in~\eqref{eq:dxeid}. We choose $n_q=n_{\dot{q}}=22$, $n_{\gamma}=24$ and $n_{\tau}=n_{\dot{q}} - n_{ua}$ with variable $n_{ua}$ (this corresponds to the dimensions of the robot dog Solo12 with four feet exerting forces and torques, such that $n_{\gamma} = 4\cdot 6$, and $n_{ua}=6$). We use 8 threads for the parallel computations of the turnback algorithm as described in Sec.~\ref{sec:threads}.

The results are given in Fig.~\ref{fig:turnback}. Depicted are the computation time and the number of non-zeros and density ($\varphi = nnz(A) / (n\cdot m)$ with $A\in\mathbb{R}^{m\times n}$) of the normal form $Z^TZ$ of the turnback nullspace depending of the time horizon $T=0,\dots,25$ and under-actuation $n_{ua} = 0,\dots,n_q$. We first consider the fully actuated case $n_{ua} = 0$. For $T=25$, the turnback computation time is $t_{tb}=4.9\cdot10^{-3}$~s. The resulting normal form $Z^TZ$ contains 103412 non-zeros with a density of 0.078. It can be observed that with under-actuation $n_{ua}>0$, there is a sharp incline in non-zeros and density of $Z^TZ$. For example for $n_{ua}=1$, the density increases to 0.328 with four times as many non-zeros (415125). This can be explained by the higher coupling within the diagonal blocks as demonstrated in Sec.~\ref{sec:tbua}. Nonetheless, the computation time does not increase as dramatically to $t_{tb}=6.9\cdot10^{-3}$~s. 
For full under-actuation $n_{ua} = n_q$, the resulting nullspace basis is dense ($\varphi=1$) as expected. At the same time, the increase in non-zeros is quadratic. In contrast, for low $n_{ua}$, the linear increase of non-zero entries in $T$ is clearly distinguishable.

\subsection{NL-HLSP test-functions}
\label{sec:eval:nonlinopt}

\begin{table*}[htp!]
	\centering
	\resizebox{\columnwidth}{!}	{%
		\begin{tabular}{@{} cccccccccccccccc @{}}  
			\toprule
			& & &  \multicolumn{2}{c}{$\mathcal{N}$\hspace{-2pt}ADM$_2$}  & \multicolumn{2}{c}{$\mathcal{N}$\hspace{-2pt}ADM$_2$ $\rightarrow$ H-MOSEK} & \multicolumn{2}{c}{H-MOSEK} & \multicolumn{2}{c}{H-OSQP}\\
			& & &  	
			\multicolumn{2}{c}{(0.01 s)}  & \multicolumn{2}{c}{(0.01 s$\rightarrow$0.11 s: 0.12 s)} & \multicolumn{2}{c}{(0.19 s)} & \multicolumn{2}{c}{(0.23 s)}\\		
			\cmidrule(lr){4-5}	\cmidrule(lr){6-7}	\cmidrule(lr){8-9}	\cmidrule(lr){10-11}
			$l$ & & $f_l(x)\leqq v_l$ & $\Vert v_l^* \Vert_2$ &  Iter.  & $\Vert v_l^* \Vert_2$ & Iter.  & $\Vert v_l^* \Vert_2$ & Iter. & $\Vert v_l^* \Vert_2$ & Iter.\\
			\midrule
			1 & Disk ineq. & $x_1^2 + x_2^2 - 1.9 \leq v_1$ & $1.0\cdot 10^{-5}$ & 6  & $9.8\cdot 10^{-6}$ & 1 & $9.8\cdot 10^{-6}$ & 12 & $1.0\cdot 10^{-5}$  & 7\\
			2 & Ros. eq. & $(1-x_1)^2 + 100(x_2 - x_1^2)^2 = v_2$ & $2.9\cdot 10^{-4}$ & 14  &  $2.9\cdot 10^{-4}$ & 14 & $2.9\cdot 10^{-4}$ & 17 & $2.9\cdot 10^{-4}$ & 13\\
			3 & Disk eq. & $x_1^2 + x_2^2 - 0.9 = v_3$& 1 & 2  & 1 & 1 & 1 & 1 & 1 & 2\\
			4 & Disk eq. & $x_2^2 + x_3^2 - 1 = v_4$& $1.6\cdot 10^{-6}$ & 2  & $1.2\cdot 10^{-16}$ & 1 & $1.9\cdot 10^{-16}$ & 1 & $1.1\cdot 10^{-11}$ & 2\\
			5 & Disk ineq. & $x_4^2 + x_5^2 + 1 \leq v_5$& 1 & 1 &  1 & 1 & 1 & 1 & 1 & 1\\
			6 & Disk eq. & $x_6^2 + x_7^2 + x_8^2 - 4 = v_6$& $7.1\cdot 10^{-7}$  & 4   & $1.7\cdot 10^{-8}$ & 1 & $1.6\cdot 10^{-10}$ & 8 & $1.9\cdot 10^{-8}$ & 1\\
			7 & Ros. eq. & $(1-x_6)^2 + 100(x_7 - x_6^2)^2 = v_7$& $4.2\cdot 10^{-4}$ & 1  &  $7.6\cdot 10^{-8}$ & 24 & $7.4\cdot 10^{-8}$ & 35 & $1.8\cdot 10^{-4}$ & 111\\
			8 & McC. eq. &  $\sin(x_{_9} + x_{10}) + (x_9 - x_{10})^2$ & 18.1 & 1 &  18.1 & 1 & 18.1 & 1 & 24.4 & 1\\
			&  & $- 1.5x_9 + 2.5x_{10}+1 + M = v_8$ \\
			9 & Reg. eq. & $x_{1:10} = v_9$& 2.9 & 0  & 2.9 & 0 & 2.9 & 0 & 7.8 & 0\\
			\midrule
			$\Sigma$ & & & & 31 &  &  (31$\rightarrow$44) 75&  & 76& & 138\\
			\bottomrule
		\end{tabular}
	}
	\caption{Non-linear test functions: optimal slacks $v^*$ and number of outer iterations (Iter.) per priority level for a~\ref{eq:nlhlsp} with $p=9$ and $n=10$. The hierarchy is composed of disk, Rosenbrock (Ros.), McCormick (McC.) and regularization (Reg.) equality (eq.) and inequality (ineq.) constraints.}
	\label{tab:p8nl}
\end{table*}

\begin{figure}[htp!]
	\includegraphics[width=0.7\columnwidth]{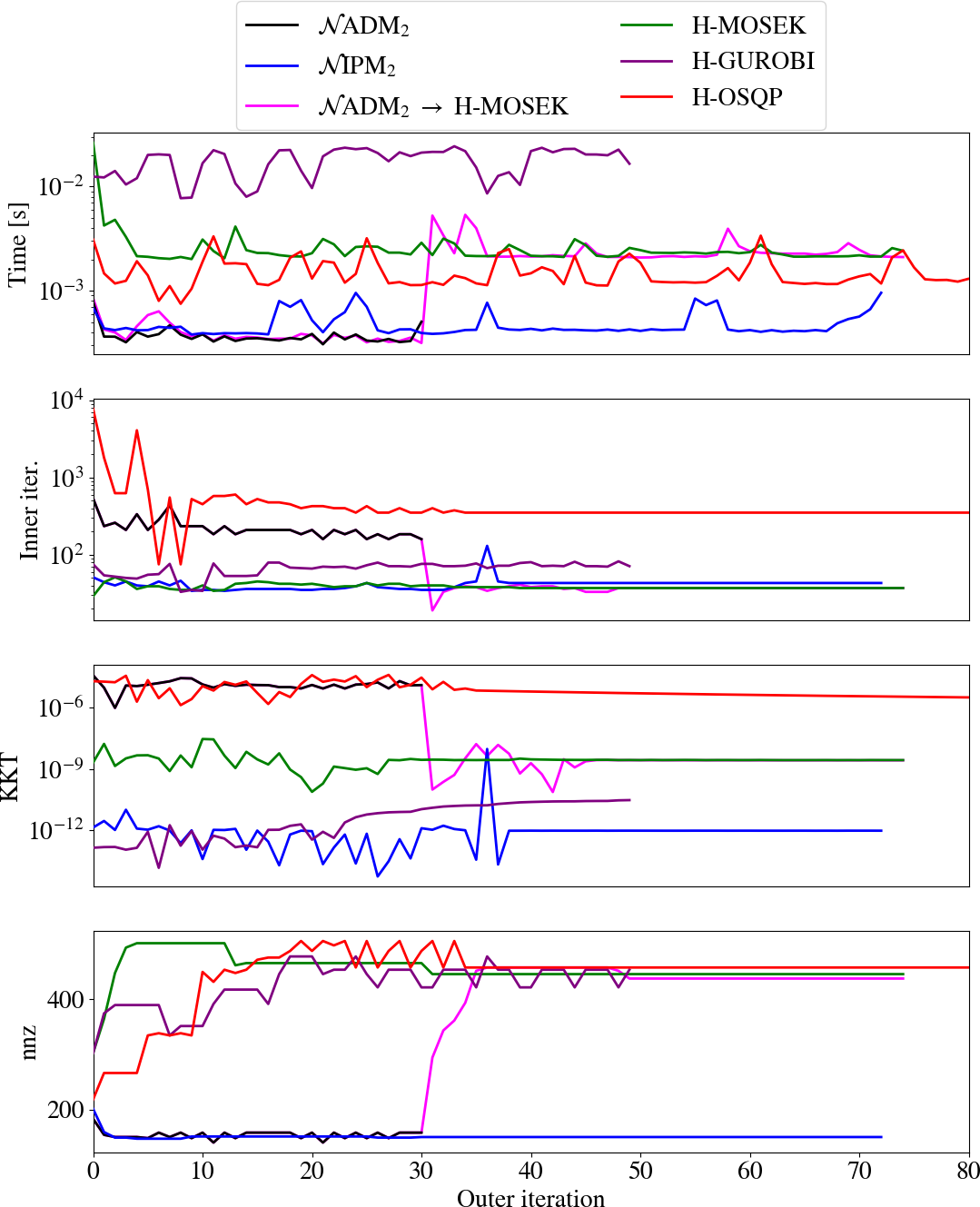}
	\centering
	\caption{Non-linear test functions, data for the different HLSP sub-solvers over S-HLSP outer iteration: computation times per HLSP solve, number of inner iterations, KKT residuals and overall number of non-zeros handled throughout the whole hierarchy.}
	\label{fig:nonlinsolverdata}
\end{figure}

\begin{figure}[htp!]
	\includegraphics[width=0.6\columnwidth]{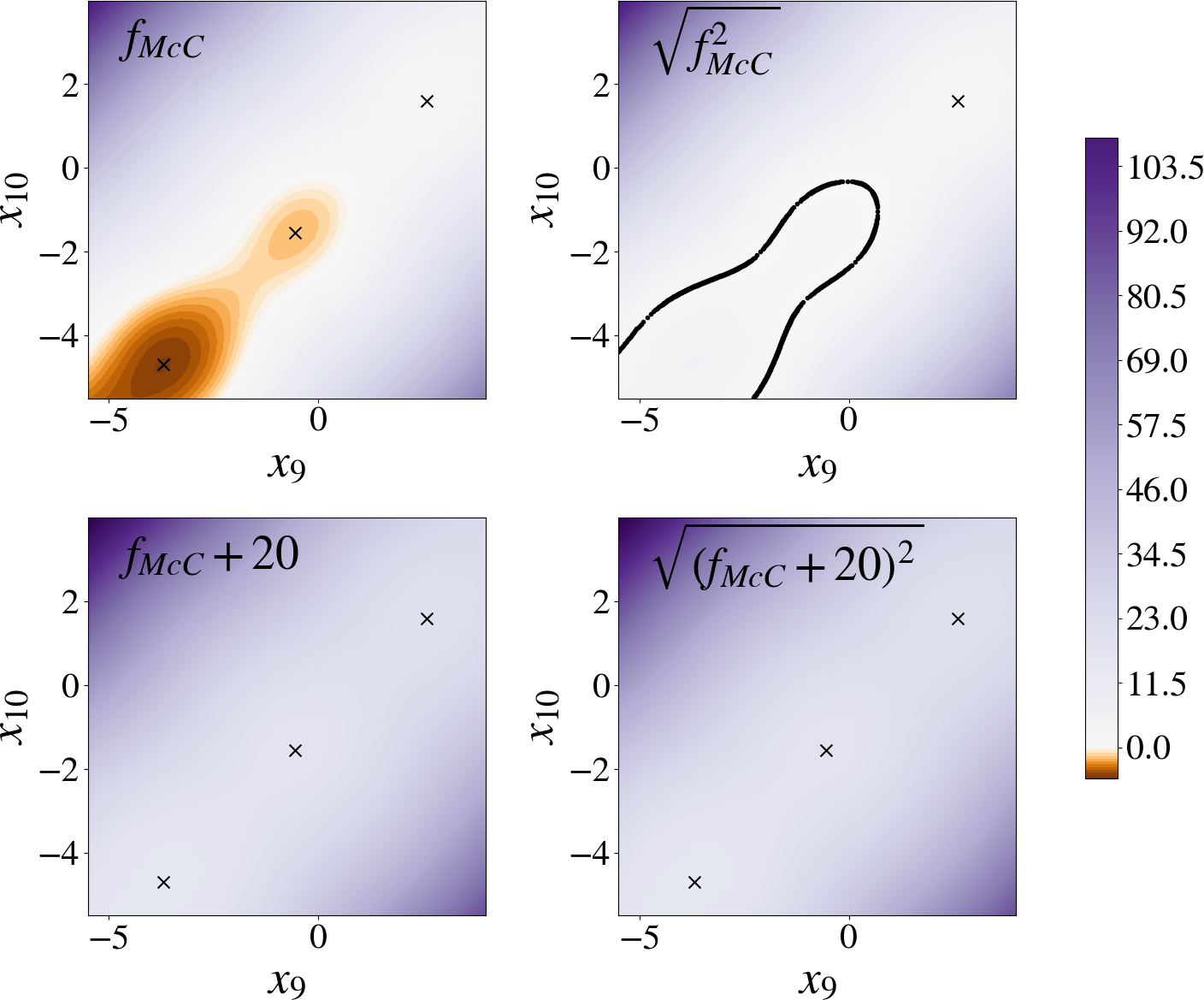}
	\centering
	\caption{Different formulations of McCormick function: original $f_{\text{McC}}$, $\ell_2$-norm $\sqrt{f_{\text{McC}}^2}$, with offset $f_{\text{McC}} + M$, $\ell_2$-norm with offset  $(f_{\text{McC}} + M)^2$ ($M=20$). Negative function values are colored in orange tones. Local minima are marked in black.}
	\label{fig:mccormick}
\end{figure}

\begin{figure*}[htp!]
	
	\centering
	\includegraphics[width=1\textwidth]{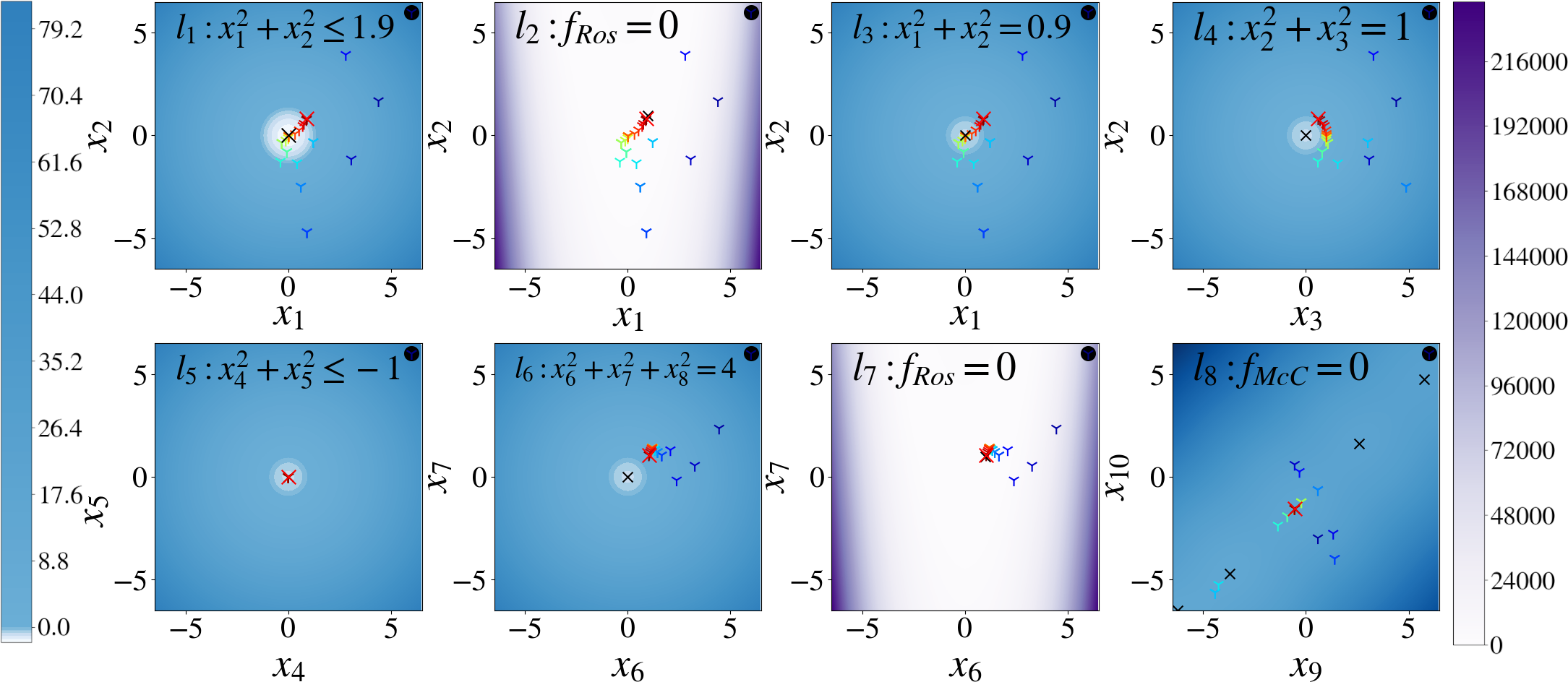}
	\caption{Non-linear test functions: primal $x$ over S-HLSP outer iteration (colored triangles) from start point (black dot) to converged point (large red cross). Local minima are marked with black crosses.}
	\label{fig:nonlinoptX}
	
\end{figure*}

\begin{figure}[htp!]
	\includegraphics[width=0.7\columnwidth]{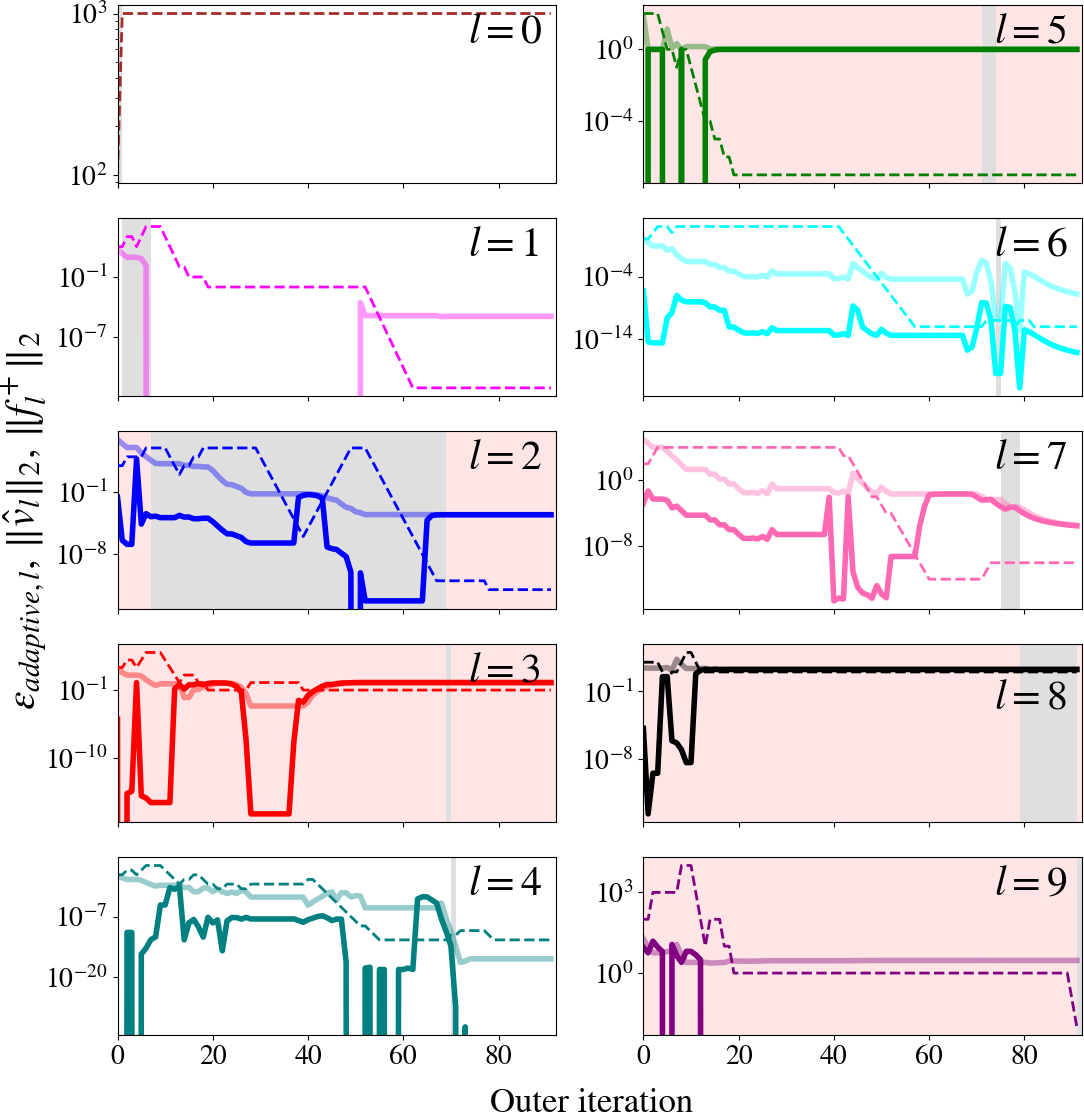}
	\centering
	\caption{Non-linear test functions,  {\hbox{$\mathcal{N}$\hspace{-2pt}ADM$_2$}}: linear slacks $\Vert \hat{v}_l\Vert_2$, non-linear error $\Vert f_l^+\Vert_2$ (light color) and $\epsilon_{adaptive,l}$ (dashed) for the levels $l=1,\dots,p$ over outer iteration. The initial value is chosen as $\epsilon_{adaptive,l} = 100$. Gray background color indicates the outer iterations where the respective level has been resolved by the HSF. Infeasible levels are indicated by light red background color.}
	\label{fig:epsadapt}
\end{figure}

We apply  {\hbox{$\mathcal{N}$\hspace{-2pt}ADM$_2$}} as HLSP sub-solver for S-HLSP to solve a NL-HLSP composed of test functions as listed in Tab.~\ref{tab:p8nl}. This problem constellation tests feasible and infeasible equality and inequality constraints. This also includes infeasibility arising from conflict with constraints from higher priority levels. Note that the McCormick function includes a large positive offset $M$. This enables S-HLSP to identify negative function minima (for example the minimum $f_{\text{McC}}(-0.547,-1.547) + M = -1.9133 + M$) despite its least-squares formulation.
This can be easily confirmed with the following theorem:
\begin{definition}
	A factor $M \geq 0$ is sufficiently large on the domain $\{x\in\mathbb{R}^n:x\in \mathcal{S}\}$ if a function $\hat{f}(x,M) = f(x) + M : \mathcal{S} \rightarrow \mathbb{R}_{> 0}$.
\end{definition}
\begin{theorem}
	If $M \geq 0$ is sufficiently large on the domain $x\in \mathcal{S}$, then first-order optimality of $\hat{f}(x,M)^2$ applies at the same points $x^*\in \mathcal{S}$ as the twice continuously differentiable function~$f(x)$.
\end{theorem}
\begin{proof}
	We consider the first-order derivative of $\hat{f}(x,M)^2$ which writes as
	\begin{equation}
		\frac{\partial \hat{f}^2}{\partial x} = 2\hat{f}\frac{\partial \hat{f}}{\partial x} = 2\hat{f}\frac{\partial f}{\partial x}
	\end{equation}
	Clearly, ${\partial \hat{f}^2}/{\partial x}$ has the same zeros as ${\partial f}/{\partial x}$ for $\hat{f}(x,M)>0$ on the domain $x\in\mathcal{S}$.
\end{proof}
Figure~\ref{fig:mccormick} shows how the squared McCormick function with offset $(f_{\text{McC}}+M)^2$ has the same minima $(x_9^*,x_{10}^*)$ as the original function in the range $x_9, x_{10}\in\left[-5.5,4\right]$. Here, we choose $M=20$. For even polynomials, the global value $M^*$, such that $\hat{f}(x^*,M^*)=0$, can be found by polynomial optimization~\cite{Chen2021}. In contrast, the squared McCormick function without offset has additional minima at the zeros of $f_{\text{McC}}^2$, while missing any minima that are associated with negative function values $f_{\text{McC}} < 0$. 

S-HLSP with {\hbox{$\mathcal{N}$\hspace{-2pt}ADM$_2$}} identifies the primal solution $x = $ 0.983,
0.966, 0.257, 0,0, 1.02, 1.04, 1.37, -0.547, -1.547. The evolution of the primal over the S-HLSP outer iterations is depicted in Fig.~\ref{fig:nonlinoptX}.
{\hbox{$\mathcal{N}$\hspace{-2pt}ADM$_2$}} is able to solve the NL-HLSP to moderate accuracy. For example, according to Tab.~\ref{tab:p8nl}, the Rosenbrock equality on level 7 is solved to a residual error of $\Vert v^*_7 \Vert^2_2 = 4.2\cdot 10^{-4}$ while H-MOSEK solves the same level to $\Vert v^*_7 \Vert^2_2 = 7.4\cdot 10^{-8}$. At the same time, the previous levels are solved to comparable accuracy (note that error comparisons in hierarchies need to consider that a higher error norm on a higher priority level can lead to lower error norm on a lower priority level).
While being less accurate,  {\hbox{$\mathcal{N}$\hspace{-2pt}ADM$_2$}} (0.011 s) solves the problem the fastest out of all the solvers, see Fig.~\ref{fig:nonlinsolverdata}. This is partly due to the low  number of outer S-HLSP iterations (31, about half as many as for H-MOSEK with 76). Still, from Fig.~\ref{fig:nonlinsolverdata}, it can be observed that the HLSP sub-problems are solved in about $4\cdot10^{-4}$~s. $\mathcal{N}$\hspace{-1pt}IPM$_2$ solves the sub-problems slightly slower in about $5\cdot10^{-4}$~s. This is in accordance with the number of inner iterations of $\sim\hspace{-2pt}300 <\hspace{2pt} \sim\hspace{-2pt}50\cdot n = \hspace{2pt}\sim\hspace{-2pt}500$ of  \hbox{$\mathcal{N}$\hspace{-2pt}ADM$_2$} and $\mathcal{N}$\hspace{-1pt}IPM$_2$, respectively. In contrast, the next fastest solver H-OSQP solves the inner iterations in about $1\cdot10^{-3}$~s. This clearly demonstrates the advantage of solving the KKT system projected into the nullspace of active constraints.

We furthermore consider the combination of both the low and high accuracy solvers  {\hbox{$\mathcal{N}$\hspace{-2pt}ADM$_2$}} and H-MOSEK. It can be observed that a high accuracy solution is obtained when compared to the low accuracy solver  {\hbox{$\mathcal{N}$\hspace{-2pt}ADM$_2$}} alone (level 7 at $7.6\cdot 10^{-8}$ compared to $4.2\cdot 10^{-4}$). At the same time, the computation time is lower (0.12 s) compared to the high accuracy solver H-MOSEK alone (0.19 s). Consequently, a sub-problem solver with moderate accuracy like  {\hbox{$\mathcal{N}$\hspace{-2pt}ADM$_2$}} can be used to warm-start the S-HLSP with a lower accuracy primal guess. The reduced overall computation time follows due to the reduced number of high accuracy sub-problem solutions (44 compared to 76 for H-MOSEK alone).

Finally, we evaluate the adaptive SOI thresholding strategy developed in Sec.~\ref{sec:epsadapt}. For this, we set the initial value to $\epsilon_{adaptive,l} = 100$ (note that in all other examples, we initially set $\epsilon_{adaptive,l} = 1\cdot10^{-12}$) and the lower limit to $1\cdot 10^{-12}$. The corresponding parameters are chosen as $\zeta = 1$ and $\delta=0.95$ (see Alg.~\ref{alg:epsadapt}). As can be seen from Fig.~\ref{fig:epsadapt}, a similar error reduction as in Tab.~\ref{tab:p8nl} is achieved. However, more iterations are necessary (92 instead of 31) as $\epsilon_{adaptive,l}$ is adjusted to the infeasibility of the constraints. Importantly, at convergence of the HSF of the infeasible levels $l=2,3,5,8$, we have $\Vert f_l^+\Vert_2^2 > \epsilon_{adaptive,l}$ and the SOI is activated. At the same time, the heuristic relaxes the threshold for example for the infeasible Rosenbrock constraint on level 2 in instances of sufficient progress in terms of optimality. Nonetheless, the SOI is erroneously activated for the feasible level 7. This motivates further investigation with respect to constraint optimality by avoiding regularized minima, for example based on machine learning methods for feasibility detection.

\subsection{Inverse kinematics of humanoid robot HRP-2}

\label{sec:eval:hrp2}

\begin{table*}[htp!]
	\centering
	\setlength\tabcolsep{5pt} 
	\resizebox{1\columnwidth}{!}	{%
		\begin{tabular}{@{} cccccccccccccc @{}} 
			\toprule
			& &\multicolumn{2}{c}{$\mathcal{N}$\hspace{-2pt}ADM$_2$} & \multicolumn{2}{c}{$\mathcal{N}$\hspace{-1pt}IPM$_2$} & \multicolumn{2}{c}{$\mathcal{N}$\hspace{-2pt}A. $\rightarrow$ H-M.} & \multicolumn{2}{c}{H-MOSEK}& \multicolumn{2}{c}{H-GUROBI}& \multicolumn{2}{c}{H-OSQP} \\ 	
			&&	\multicolumn{2}{c}{ (0.06 s)}  & \multicolumn{2}{c}{(0.40 s)} & \multicolumn{2}{c}{(0.06 s$\rightarrow$0.39 s: 0.46 s)} &  \multicolumn{2}{c}{(0.59 s)}  &  \multicolumn{2}{c}{(1.06 s)} &  \multicolumn{2}{c}{(0.089 s)}\\				
			\cmidrule(lr){3-4}	\cmidrule(lr){5-6}	\cmidrule(lr){7-8}\cmidrule(lr){9-10}	\cmidrule(lr){11-12} \cmidrule(lr){13-14}
			$l$ & $f_l(x) \leqq v_l$ & $\Vert v_l^* \Vert_2$ & Iter. & $\Vert v_l^* \Vert_2$ & Iter. & $\Vert v_l^* \Vert_2$& Iter.& $\Vert v_l^* \Vert_2$& Iter.& $\Vert v_l^* \Vert_2$& Iter.& $\Vert v_l^* \Vert_2$& Iter.\\
			\midrule
			1 & J. lim. ineq. & $5.1\cdot 10^{-5}$ & 2 & $3.4\cdot 10^{-8}$ &  1 &$3.4\cdot 10^{-5}$& 0 & 0 & 0 & 0 & 0 & 0 & 0\\
			2 & LF, RF, LH eq. & $3.1\cdot 10^{-6}$ & 14  & $2.6\cdot 10^{-10}$ &  18 & $7.8\cdot 10^{-9}$& 10& $9.9\cdot 10^{-9}$ & 9 & $7.4\cdot 10^{-8}$ & 27 & $4.8\cdot 10^{-8}$ & 6\\
			3 & CoM ineq. & $4.1\cdot 10^{-6}$ & 1  & $1.0\cdot 10^{-8}$ &  1 & $1.0\cdot 10^{-5}$& 65&  $1.0\cdot10^{-5}$ & 30 & 0 & 28 & $2.6\cdot 10^{-6}$ & 1\\
			4 & Right hand eq. & 1.04 & 9 & 0.98&  128  &1.02&2& 1.02  & 56 & 1.13 & 19 & 0.98 & 12\\
			5 & Reg. eq. & 5.5 & 10 & 4.40 &  7 &4.1&0&  4.98 & 0 & 4.70 & 0 & 4.39 & 0\\
			\midrule
			$\Sigma$ & & & 37 & & 156  && (48$\rightarrow$77) 125 & & 95 & & 74 & & 19\\
			\bottomrule
	\end{tabular}}
	\caption{HRP-2 inverse kinematics: optimal slacks $v^*$ and number of outer iterations (Iter.) per priority level for a~\ref{eq:nlhlsp} with $p=5$ and $n=38$. J. lim.: Joint limits, LF: left foot, RF: right foot, LH: left hand. $\mathcal{N}$\hspace{-2pt}A.:  {\hbox{$\mathcal{N}$\hspace{-2pt}ADM$_2$}}; H-M: H-MOSEK.}
	\label{tab:hrp2}
\end{table*}

\begin{figure}[htp!]
	\includegraphics[width=0.7\columnwidth]{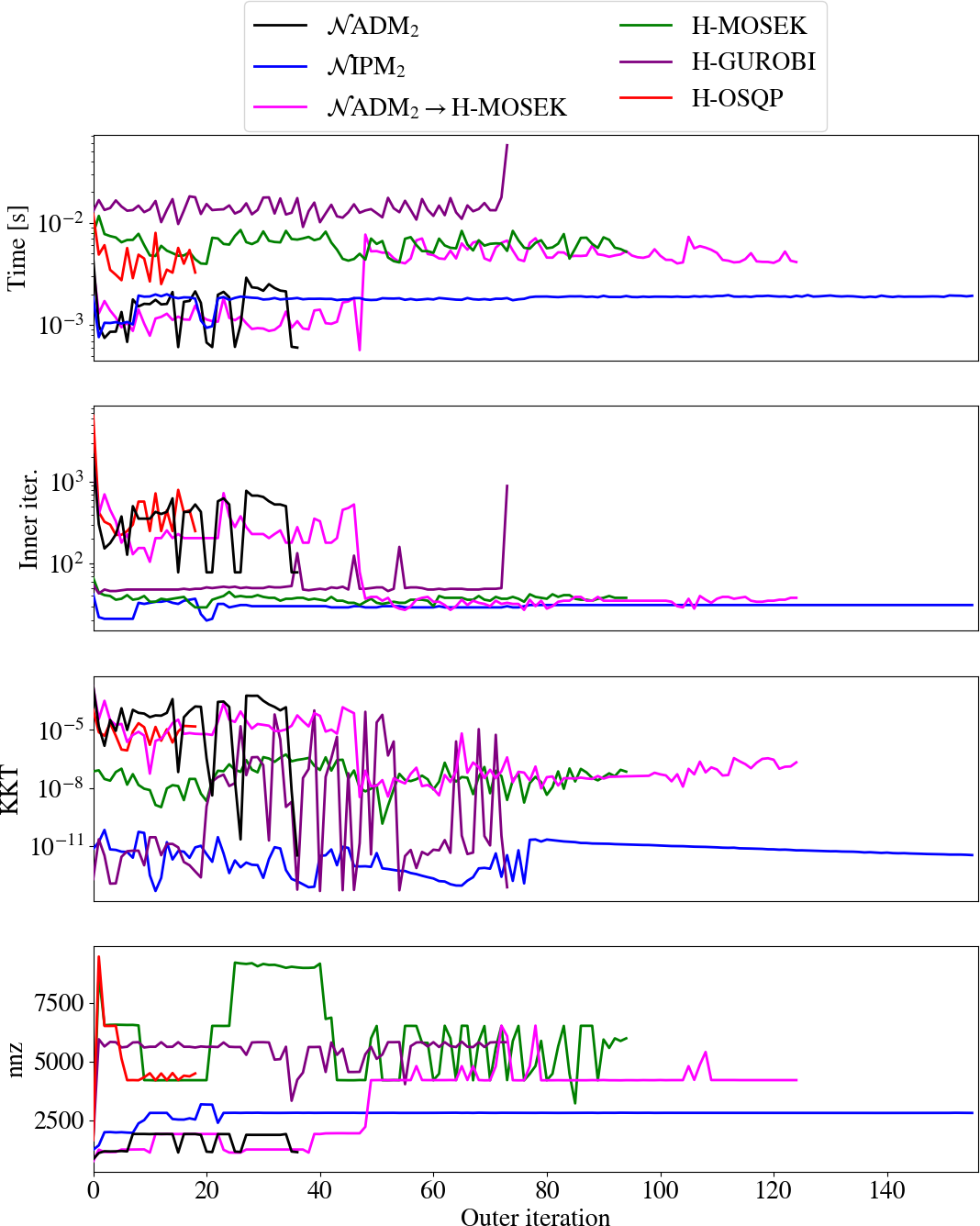}
	\centering
	\caption{HRP-2 inverse kinematics, data for the different HLSP sub-solvers over S-HLSP outer iteration: computation times per HLSP solve, number of inner iterations, KKT residuals and overall number of non-zeros handled throughout the whole hierarchy.}
	\label{fig:hrp2data}
\end{figure}

\begin{figure}[htp!]
	\includegraphics[width=0.4\columnwidth]{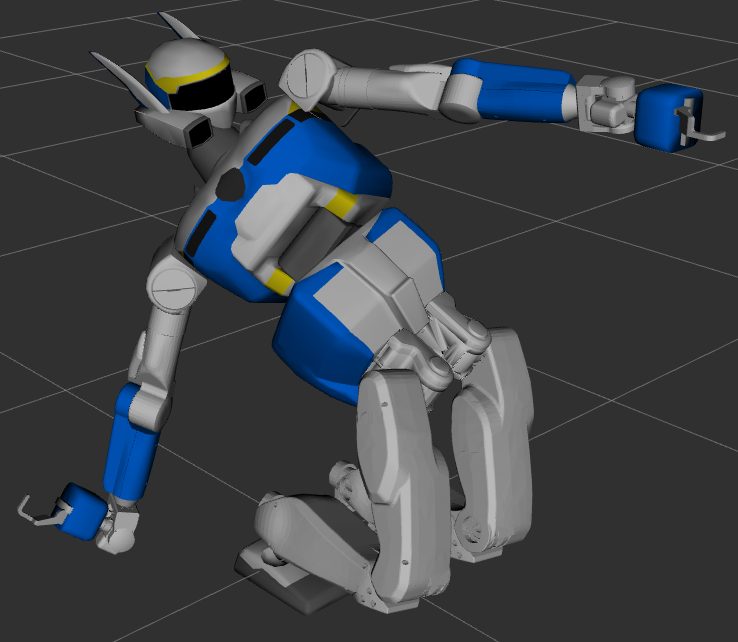}
	\centering
	\caption{HRP-2 inverse kinematics,  {\hbox{$\mathcal{N}$\hspace{-2pt}ADM$_2$}}: converged robot posture.}
	\label{fig:hrp2}
\end{figure}

This simulation is concerned with solving an inverse kinematics problem for the humanoid robot HRP-2 with $n=38$ degrees of freedom.  The corresponding hierarchy is given in Tab.~\ref{tab:hrp2} ($p=5$).
The first level limits the joint angles. The second level positions the left and right feet and the left hand. The third level limits the CoM position to a bounding box. The fifth level positions the right hand towards an out-of-reach target $\BIN -0.5 & -0.5 & -1\BOUT$~m below its feet. The right foot is positioned at $\BIN 0.015 &  -0.1 & 0.1\BOUT$~m. The $z$ component is approximately at ground level $0.1$~m. Lastly, all variables are regularized to zero.

The results are given in Tab~\ref{tab:hrp2}. The converged robot posture for  {\hbox{$\mathcal{N}$\hspace{-2pt}ADM$_2$}} is depicted in Fig.~\ref{fig:hrp2}. Our proposed solver  {\hbox{$\mathcal{N}$\hspace{-2pt}ADM$_2$}} solves the HLSP sub-problems the fastest at about $6\cdot 10^{-2}$~s. Fluctuations in computation time and non-zeros (Fig.~\ref{fig:hrp2data}) is due to activation and deactivations of SOI on the second level.
At the same time, moderate accuracy is achieved. For example, the end-effector positioning of the left and right foot and the left hand on level 2 is resolved to an error of $3.1\cdot 10^{-6}$~m while the error is reduced to less than $4.8\cdot 10^{-8}$~m (H-OSQP) for the other solvers. The right hand task on level 3 is resolved to an error of 1.04~m. While this is worse than for example $\mathcal{N}$\hspace{-1pt}IPM$_2$ (0.98~m), the adaptive SOI threshold strategy enables SOI deactivation on the higher priority level 2. Without it, only manual tuning for each individual solver prevented the SOI activation which causes worse error convergence on lower priority levels due to high variable occupation.
H-OSQP achieves a solution in the lowest number of outer iterations (19 compared to 48 for  {\hbox{$\mathcal{N}$\hspace{-2pt}ADM$_2$}}) but is slower than  {\hbox{$\mathcal{N}$\hspace{-2pt}ADM$_2$}} due to the slow resolution of the HLSP sub-problems at about $4\cdot 10^{-2}$~s. 

\subsection{Time-optimal control of manipulator}

\label{sec:eval:topm}

\begin{table}[htp!]
	\centering
	\resizebox{0.8\columnwidth}{!}	{%
		\begin{tabular}{@{} cccccccccc @{}}  
			\toprule
			& &\multicolumn{2}{c}{$\mathcal{N}$\hspace{-2pt}ADM$_2$} (4.8 s) &\multicolumn{2}{c}{$\mathcal{N}$\hspace{-1pt}IPM$_2$} (4.8 s) &  \multicolumn{2}{c}{H-MOSEK} (72.1 s) & \multicolumn{2}{c}{H-GUROBI} (3.2 s)\\ 			
			\cmidrule(lr){3-4}	\cmidrule(lr){5-6}	 \cmidrule(lr){7-8}\cmidrule(lr){9-10}
			$l$ & $f_l(x) \leqq v_l$ & $\Vert v_l^* \Vert_2$ & Iter. & $\Vert v_l^* \Vert_2$ & Iter. & $\Vert v_l^* \Vert_2$ & Iter.& $\Vert v_l^* \Vert_2$ & Iter.\\
			\midrule
			1 & $q$, $\tau$ lim. ineq. & $1.7\cdot 10^{-5}$ & 1 & $5.4\cdot 10^{-8}$ & 1 & 0 & 1 & 0 &  1 \\
			2 &$f_{\text{dyn}}(x) = v_2$ & $4.9\cdot 10^{-8}$ & 17 & $1.6\cdot 10^{-4}$ & 29  & $2.3\cdot 10^{-9}$ & 48& $8.5\cdot 10^{-8}$ &  10\\
			3 & $f_{\text{ef,adtoc}}(x) = v_3$ & $0.11$ & 87 & $8.1\cdot 10^{-2}$ & 320  &  $7.8\cdot 10^{-2}$ & 317& $1.8$ &  53\\
			4 & $\dot{h}(x) = v_4$ & 1509 & 1 & 1319 & 1 & 1363 & 1& 211 &  1  \\
			5 & $\BIN q & \dot{q}\BOUT^T=v_5$ & 231 & 0 & 245 & 1 &  274 & 1& 91 &  7 \\
			6 & $\tau=v_6$ & 306& 1 & 274 & 1 &  276 & 1 & 18 &  13 \\
			\midrule
			$\Sigma$ & & & 110 & & 353 & & 370 & & 86\\
			\bottomrule
	\end{tabular}}
	\caption{UR3e time-optimal control: optimal slacks $v^*$ and number of outer iterations (Iter.) per priority level for a~\ref{eq:nlhlsp} with $p=6$ and $n=361$.}
	\label{tab:adtoc}
\end{table}

\begin{figure}[htp!]
	\includegraphics[width=0.7\columnwidth]{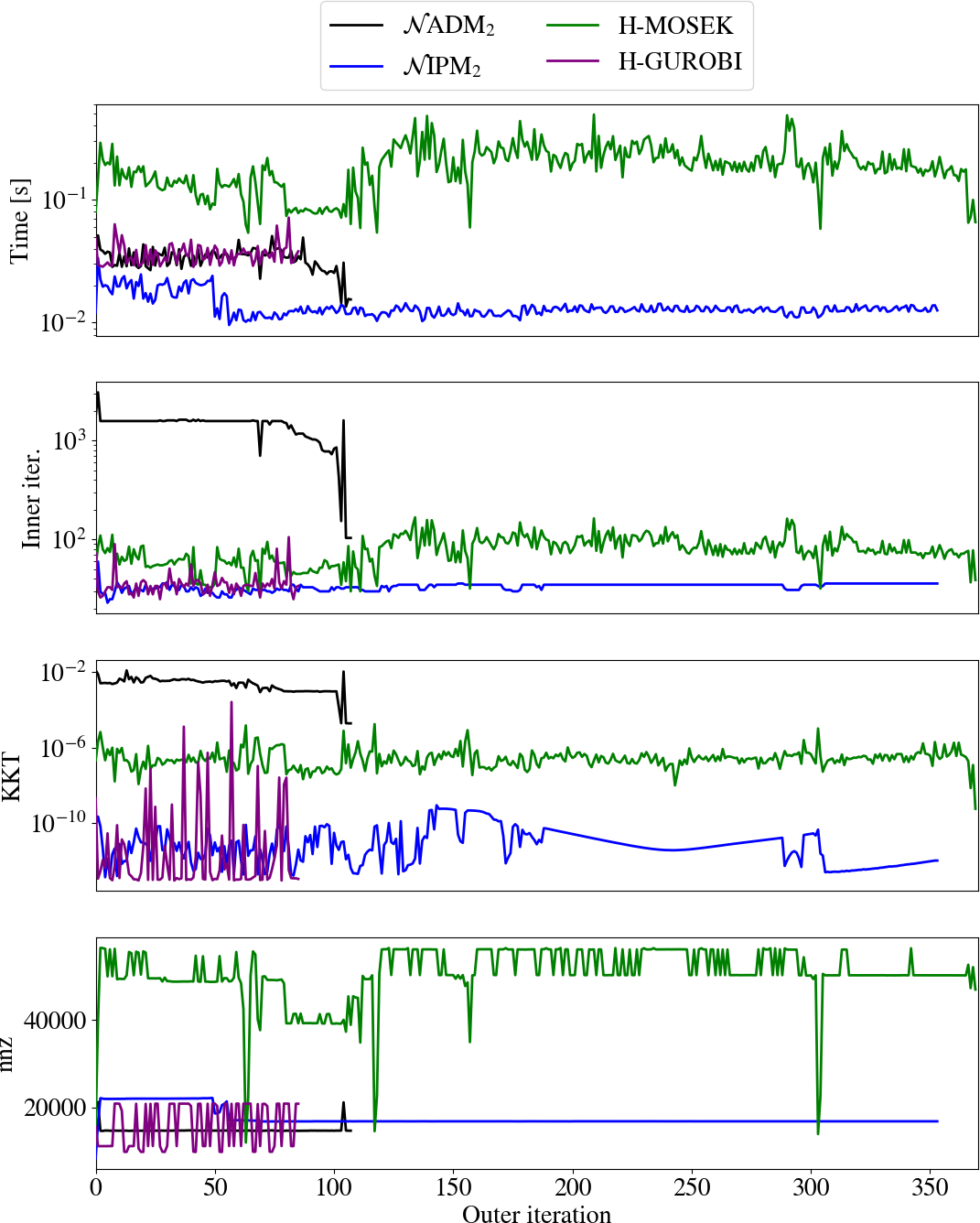}
	\centering
	\caption{UR3e time-optimal control, data for the different HLSP sub-solvers over S-HLSP outer iteration: computation times per HLSP solve, number of inner iterations, KKT residuals and overall number of non-zeros handled throughout the whole hierarchy.}
	\label{fig:adtocdata}
\end{figure}

\begin{figure}[htp!]
	\includegraphics[width=0.6\columnwidth]{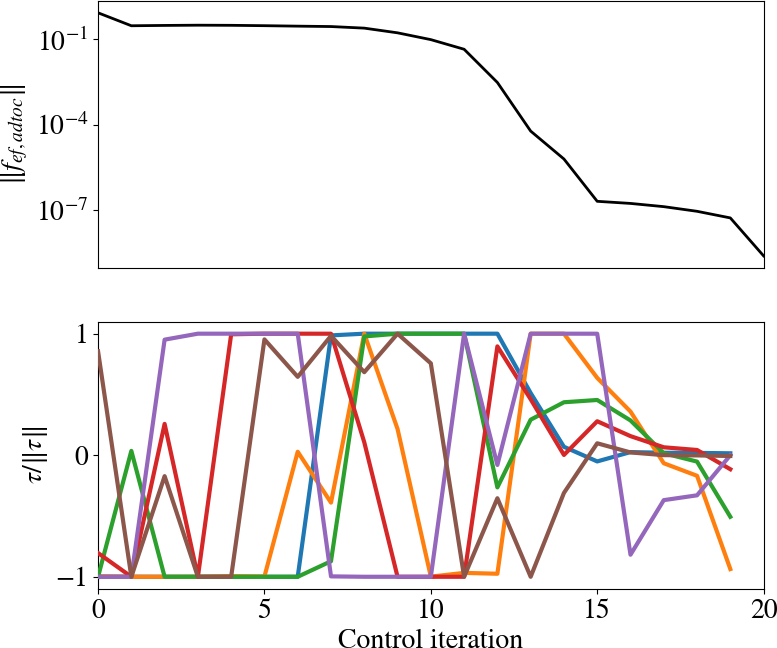}
	\centering
	\caption{UR3e time-optimal control,  {\hbox{$\mathcal{N}$\hspace{-2pt}ADM$_2$}}. Error reduction of the time-optimal reaching task (top) and joint torques normalized by their limits (bottom).}
	\label{fig:ur3eUE}
\end{figure}


In this simulation setting, we aim to identify a time-optimal control under actuation limits for a kinematic reaching task of the fully actuated manipulator UR3e. Furthermore, we impose a regularization task of the momentum evolution for a safe robot movement. Such a constraint includes variables from several stages, which is a form of constraint typically not handled by recursive methods like DDP.
Time-optimal control in least-squares programming can be achieved by a continuous approximation of the discrete optimal time $t^*$~\cite{pfeiffer2023}. Note that this requires a feasible `resting' goal point, i.e. the robot can physically remain at this point until the end of the control horizon~$T$.

The~\ref{eq:oc} hierarchy with $p=6$ is given in Tab.~\ref{tab:adtoc}. It is composed of state and control limits, explicit inverse Euler integrated dynamics, time-optimal control $f_{\text{dyn,adtoc}}$ and finally momentum time evolution $\dot{h}$, joint velocity, angle and torque regularization. The control horizon is chosen as $T=20$ such that the number of variables is $n=361$. The control time step is $\Delta t=0.01$~s.

The results in Tab.~\ref{tab:adtoc} show S-HLSP convergence in 4.8~s and within 110 and 353 iterations for {\hbox{$\mathcal{N}$\hspace{-2pt}ADM$_2$}} and {\hbox{$\mathcal{N}$\hspace{-1pt}IPM$_2$}}, respectively. This is in contrast to H-MOSEK which delivers a time-optimal solution in only 72.1~s. H-GUROBI fails to resolve the time-optimal control $f_{\text{dyn,adtoc}}$ and converges quickly due to SOI activation of the dynamics constraints (which occupies most of the variables such that lower levels show little variable activity).

Fig.~\ref{fig:adtocdata} shows how the projector based solvers  {\hbox{$\mathcal{N}$\hspace{-2pt}ADM$_2$}} and $\mathcal{N}$\hspace{-1pt}IPM$_2$ both solve the single S-HLSP iterations the fastest with run-times of around $0.05$~s and $0.02$~s (note that H-GUROBI only resolves the hierarchy up to the dynamics constraints due to SOI activation). In comparison, H-MOSEK solves the HLSP sub-problems in around $0.2$~s. This is due to the significantly lower number of non-zeros throughout the hierarchy for  {\hbox{$\mathcal{N}$\hspace{-2pt}ADM$_2$}} and $\mathcal{N}$\hspace{-1pt}IPM$_2$ (20000, about half as many as for H-MOSEK). 

Figure~\ref{fig:ur3eUE} shows the resulting  joint torques normalized by their limits (lower graph) for  {\hbox{$\mathcal{N}$\hspace{-2pt}ADM$_2$}}. A time-optimal bang-bang control profile (controls at their limits) can clearly be distinguished. This leads to a sharp drop-off of the task error $f_{\text{ef,adtoc}}$ at around control iteration 12 (upper graph).

\subsection{Swing-up of inverted pendulum}
\label{sec:eval:invpend}

\begin{table*}[htp!]
	\centering
	\resizebox{\columnwidth}{!}	{%
		\begin{tabular}{@{} cccccccccc @{}}  
			\toprule
			& &\multicolumn{2}{c}{$\mathcal{N}$\hspace{-2pt}ADM$_2$} (19.7 s) &\multicolumn{2}{c}{$\mathcal{N}$\hspace{-1pt}IPM$_2$} (4.1 s) &  \multicolumn{2}{c}{H-MOSEK} (11.8 s) & \multicolumn{2}{c}{H-GUROBI} (20.3 s)\\ 			
			\cmidrule(lr){3-4}	\cmidrule(lr){5-6}	 \cmidrule(lr){7-8}\cmidrule(lr){9-10}
			$l$ & $f_l(x) \leqq v_l$ & $\Vert v_l^* \Vert_2$ & Iter. & $\Vert v_l^* \Vert_2$ & Iter. & $\Vert v_l^* \Vert_2$ & Iter.& $\Vert v_l^* \Vert_2$ & Iter.\\
			\midrule
			1 & $\vert q\vert \leq \overline{q}$, $\vert F_{cart}\vert\leq \overline{F}$, $\vert\tau^*\vert\leq0$ (only $\mathcal{N}$\hspace{-2pt}ADM$_2$ and $\mathcal{N}$\hspace{-1pt}IPM$_2$) & $6.6\cdot 10^{-2}$ & 1 & 0 & 1 & $1.7\cdot 10^{-6}$  & 33 &  $5.0\cdot 10^{-5}$ & 1\\
			2 &$f_{\text{dyn}}(x) = v_2$ & $2.9\cdot 10^{-5}$ & 2 & $4.8\cdot 10^{-7}$ & 81 & $2.8\cdot 10^{-8}$ & 129 &  $2.0\cdot 10^{-8}$ & 152\\
			3 & $f_{\text{ef}}(q) = v_3$ & 3.90 & 371 & 3.19 & 78 & 6.0 & 2 & 3.2 & 1\\
			4 & $\BIN q & \dot{q}\BOUT^T=v_5$ & 268.6 & 1 & 263.0 & 22 & 175.2 & 2 & 274.2 & 18\\
			5 & $F_{cart}=v_6$ &  683.5 & 1 & 600.1 & 1 & 416.7 & 1 & 577.9 & 74\\
			\midrule
			$\Sigma$ & & & 378 & & 184 & & 168 & & 247\\
			\bottomrule
	\end{tabular}}
	\caption{Inverted pendulum swing-up: optimal slacks $v^*$ and number of outer iterations (Iter.) per priority level for a~\ref{eq:nlhlsp} with $p=5$ and $n=375$ ($n^*=450$ for  {\hbox{$\mathcal{N}$\hspace{-2pt}ADM$_2$}} and $\mathcal{N}$\hspace{-1pt}IPM$_2$).}
	\label{tab:invpend}
\end{table*}

\begin{figure}[htp!]
	\includegraphics[width=0.7\columnwidth]{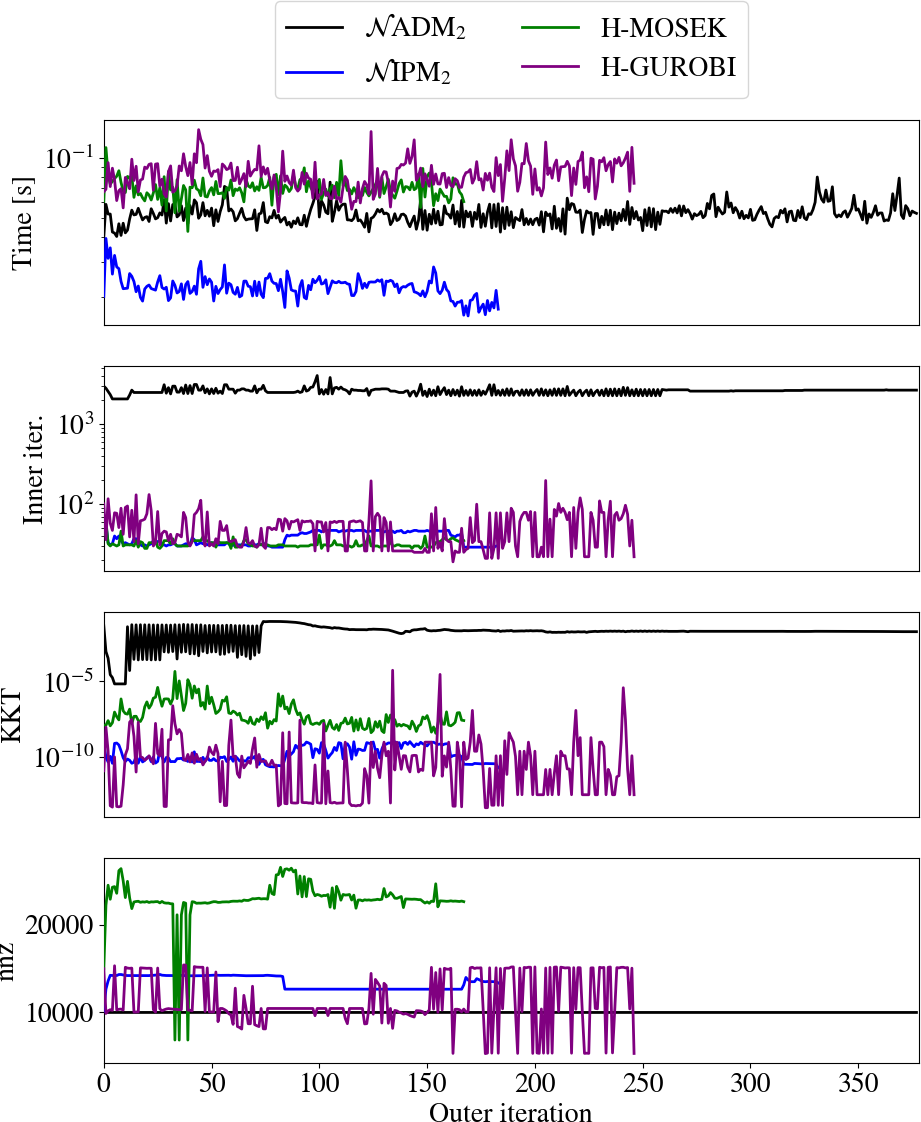}
	\centering
	\caption{Inverted pendulum swing-up, data for the different HLSP sub-solvers over S-HLSP outer iteration: computation times per HLSP solve, number of inner iterations, KKT residuals and overall number of non-zeros handled throughout the whole hierarchy.}
	\label{fig:invpend}
\end{figure}

\begin{figure}[htp!]
	\includegraphics[width=0.6\columnwidth]{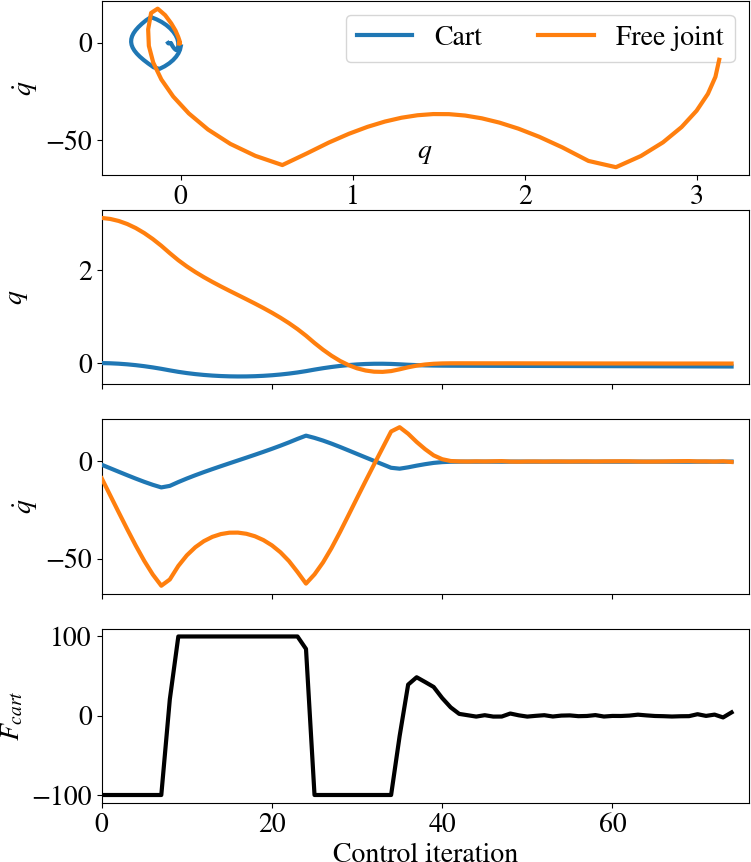}
	\centering
	\caption{Inverted pendulum swing-up, $\mathcal{N}$\hspace{-1pt}IPM$_2$: position and velocity $q$ and $\dot{q}$ of the cart and the freely swinging pendulum, force $F_{cart}$ applied to the cart.}
	\label{fig:invpendqdqf2}
\end{figure}

In this two-dimensional simulation setting, we aim to compute a swing up motion of a freely rotating pendulum mounted to a horizontally moving cart~\cite{Cavdaroglu2008}. The control input is given by the force $F_{cart}$ (limit $\overline{F} = 100~N$) applied horizontally to the cart of weight 0.1~kg. The pendulum is of length 0.25~m and of mass 0.1~kg. The coordinates $q = \BIN q_1 & q_2\BOUT^T = \BIN 0 & \pi\BOUT^T$ describe the horizontal cart position $q_1$ and the pendulum angle $q_2$. $q_2 = 0$ corresponds to the upright pendulum 	position. 
The planning horizon is $T=75$ ($\Delta t = 0.0025$~s) with $n=375$ ($n^*=450$ according to Sec.~\ref{sec:tbua} with under-actuation $n_{ua}=1$, for $\mathcal{N}$ADM$_2$ and $\mathcal{N}$\hspace{-1pt}IPM$_2$).

The hierarchy of this trajectory optimization problem is given in Tab~\ref{tab:invpend}. On the first two levels, joint angle and torque limits and the dynamics equations integrated by the explicit Euler method are defined. The third level contains a positioning task where the desired Cartesian position of the pendulum tip is set to 0~m and 0.5~m for the horizontal and vertical axis, respectively. 

The solver data is given in Tab.~\ref{tab:invpend} and Fig.~\ref{fig:invpend}. $\mathcal{N}$\hspace{-1pt}IPM$_2$ is able to compute a swing-up motion in 4.1~s. The corresponding robot values are depicted in Fig.~\ref{fig:invpendqdqf2}. It can be seen from the bottom graph that $F_{cart}$ exhibits slight jittering at the upright position from control iteration 45 onwards. This is due to SOI approximations (linear Lagrange multipliers, regularization, numerical errors, ...).
A slightly worse solution is delivered by H-GUROBI (position task error of $\Vert v_3\Vert_2^2 = 3.2$ instead of $\Vert v_3\Vert_2^2 = 3.19$ for $\mathcal{N}$\hspace{-1pt}IPM$_2$). The computation time is significantly longer at 20.3~s and 247 outer iterations. This is due to the higher number of non-zeros handled throughout the hierarchy and confirms the efficiency of the turnback algorithm in the under-actuated case. Both solvers $\mathcal{N}$ADM$_2$ and H-MOSEK struggle to resolve level 3 of the HLSP sub-problem (see graph of KKT norm in Fig.~\ref{fig:invpend}), giving rise to worse convergence in the corresponding NL-HLSP ($\Vert v_3^*\Vert_2^2 = 3.9$ for $\mathcal{N}$ADM$_2$ and $\Vert v_3^*\Vert_2^2 = 6.0$ for H-MOSEK; note that we increased the maximum number of inner iterations to 2000 for $\mathcal{N}$ADM$_2$).

\subsection{Jump of robot dog Solo12}
\label{sec:eval:solo12}

\begin{table*}[htp!]
	\centering
	\resizebox{\columnwidth}{!}	{%
		\begin{tabular}{@{} cccccccccc @{}}  
			\toprule
			& &\multicolumn{2}{c}{$\mathcal{N}$\hspace{-2pt}ADM$_2$ (7.0 s)} & \multicolumn{2}{c}{$\mathcal{N}$\hspace{-1pt}IPM$_2$(12.2 s)}  & \multicolumn{2}{c}{H-MOSEK} (80.8 s) & \multicolumn{2}{c}{H-GUROBI} (7.4 s)\\ 			
			\cmidrule(lr){3-4}	\cmidrule(lr){5-6}	 \cmidrule(lr){7-8} \cmidrule(lr){9-10}
			$l$ & $f_l(x) \leqq v_l$ & $\Vert v_l^* \Vert_2$ & Iter. & $\Vert v_l^* \Vert_2$ & Iter. & $\Vert v_l^* \Vert_2$ & Iter. & $\Vert v_l^* \Vert_2$ & Iter.\\
			\midrule
			1 & $\vert q\vert \leq \overline{q}$, $\vert\tau\vert\leq \overline{\tau}$, $\vert\tau^*\vert\leq0$ (only $\mathcal{N}$\hspace{-2pt}ADM$_2$ and $\mathcal{N}$\hspace{-1pt}IPM$_2$), $\gamma_z\geq 0$  & $4.8\cdot 10^{-4}$ & 1 & $1.4\cdot 10^{-7}$ & 1 & $0$ & 1 & $4.7\cdot 10^{-5}$ & 33 \\
			2 &$f_{\text{dyn}}(x) = v_2$ & $1.2\cdot 10^{-8}$ & 5 & $2.8\cdot 10^{-6}$ & 2 & $4.8\cdot 10^{-9}$& 55 & $3.1\cdot 10^{-7}$ & 1 \\
			3 & $\sqrt{\gamma_x^2 + \gamma_y^2}\leq \mu \gamma_z$  & $1.0\cdot 10^{-5}$ & 4 & 0 & 28 & 0 & 30 & $1.4\cdot 10^{-8}$ & 1\\
			4 & ${f_{\text{ef}}}(q) = v_{3,1}$ & $4.9\cdot 10^{-3}$ & 6 & $5.2\cdot 10^{-4}$ & 12 & $2.5\cdot 10^{-4}$  & 30 & $1.8\cdot 10^{-2}$ & 3 \\
			& $10^{-3}\cdot q_{act} = v_{3,2}$ & $2.1\cdot 10^{-3}$ & & $4.4\cdot 10^{-4}$ & & $1.7\cdot 10^{-3}$ & & $2.4\cdot 10^{-3}$\\
			5 & $\dot{h}(x) = v_4$ & 83.0 & 35 & 130.5 & 14 & 176.7 & 51 & 89.5 & 3 \\
			6 & $\BIN q^T&\dot{q}^T\BOUT^T=v_5$ & 49.0 & 1 & 36.8 & 4 & 87.3 & 5 & 38.6 & 1 \\
			7 & $\BIN\tau^T & \gamma^T\BOUT=v_6$ & 36.1 & 1 & 108.1 & 12 & 315.5 & 7 & 19.0 & 1 \\
			\midrule
			$\Sigma$ & & & 54 & & 74 & & 180 & & 44\\
			\bottomrule
	\end{tabular}}
	\caption{Solo12 jump: optimal slacks $v^*$ and number of outer iterations (Iter.) per priority level for a~\ref{eq:nlhlsp} with $p=6$ and $n=900$ ($n^*=990$ for  {\hbox{$\mathcal{N}$\hspace{-2pt}ADM$_2$}} and $\mathcal{N}$\hspace{-1pt}IPM$_2$).}
	\label{tab:solo12}
\end{table*}

\begin{figure}[htp!]
	\includegraphics[width=0.7\columnwidth]{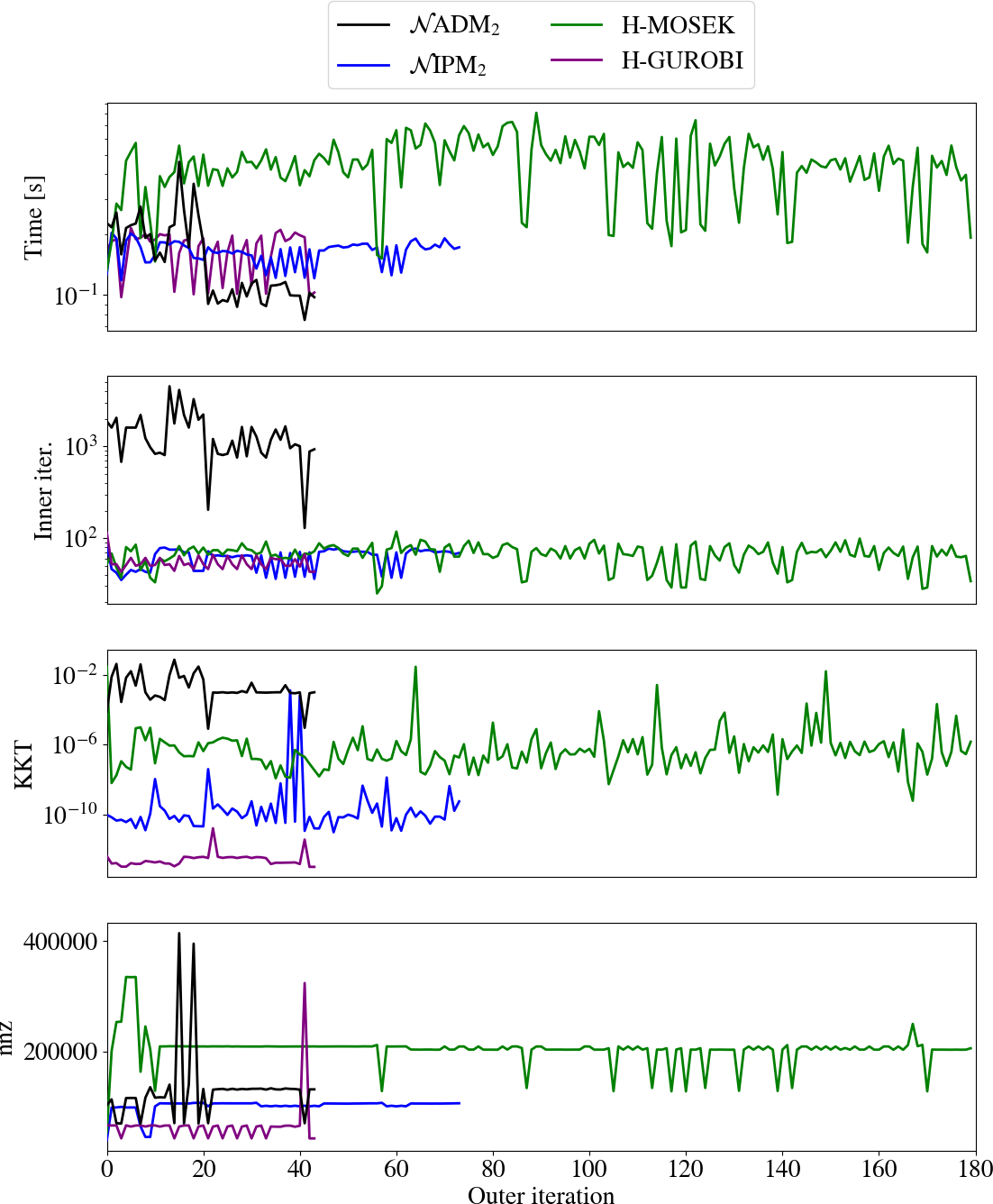}
	\centering
	\caption{Solo12 jump, data for the different HLSP sub-solvers over S-HLSP outer iteration: computation times per HLSP solve, number of inner iterations, KKT residuals and overall number of non-zeros handled throughout the whole hierarchy.}
	\label{fig:solo12data}
\end{figure}

\begin{figure}[htp!]
	\includegraphics[width=0.4\columnwidth]{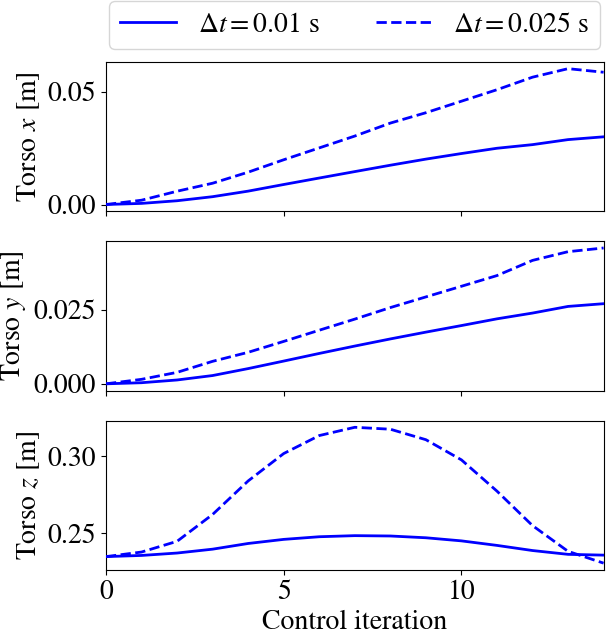}
	\centering
	\caption{Solo12 jump, $\mathcal{N}$\hspace{-1pt}IPM$_2$, torso trajectory, dashed for $\Delta t = 0.025$~s.}
	\label{fig:solo12x}
\end{figure}

\begin{figure}[htp!]
	\includegraphics[width=0.5\columnwidth]{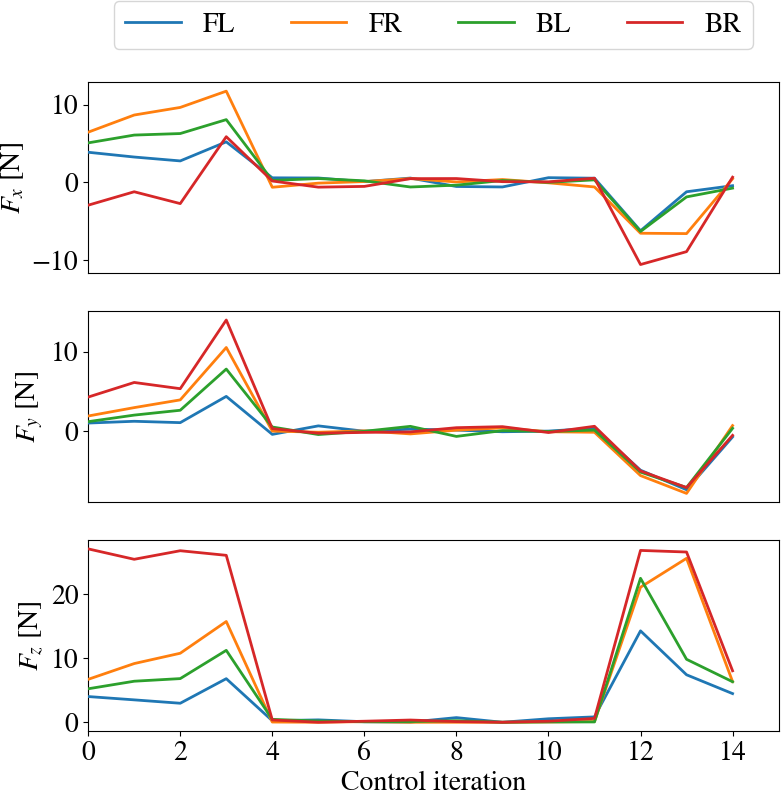}
	\centering
	\caption{Solo12 jump, $\mathcal{N}$\hspace{-1pt}IPM$_2$, contact forces.}
	\label{fig:solo12f}
\end{figure}

In this example, we compute a whole-body jumping motion of the robot dog Solo12 over a horizon of $T=15$ with $\Delta t=0.01$~s. The base of the robot is not actuated ($n_{ua}=6$). We use the turnback algorithm for under-actuated systems as described in Sec.~\ref{sec:tbua}. This effectively increases the number of variables from $n=900$ to $n^*=990$.

The control hierarchy is given in Tab.~\ref{tab:solo12}. Joint, torque and contact friction cone ($\mu=1$) constraints enforce robot safety and physicality. The jumping motion is enforced by removing contacts from the dynamics equation $f_{\text{dyn}}$ in the time interval from $t=4$ to $t=13$. The contacts $f_{\text{ef}}$ after landing are shifted by $\BIN 0.05 & 0.05 & 0\BOUT$~m compared to the initial stance. On the same level, we add a posture regularization task with low weight on the actuated robot joints. Furthermore, the evolution of the angular momentum $\dot{h}$ is regularized to zero. This promotes a stable flight phase. 

The results are given in Fig.~\ref{fig:solo12data}. Our solver $\mathcal{N}$\hspace{-2pt}ADM$_2$ (7.0 s, 54 outer iterations) solves the NL-HLSP faster than  $\mathcal{N}$\hspace{-1pt}IPM$_2$ (12.2~s, 74 outer iterations), H-MOSEK (80.8~s, 180 outer iterations) and H-GUROBI (7.4 s, 44 outer iterations). However, the HLSP solve times fluctuate with SOI activations which introduce a large number of non-zeros due to the SOI of the dynamics equations (see nnz peaks in bottom graph of Fig.~\ref{fig:solo12data}).
This could for example be avoided by generating strictly dynamically feasible outer iterates, which has been proposed in a SQP trust-region method~\cite{Tenny2004}.
In contrast, SOI activations of the dynamics constraints do not occur for $\mathcal{N}$\hspace{-1pt}IPM$_2$, which can be contributed to the high-accuracy nature of the solver.  
Combined with the lower number of iterations ($\sim \hspace{-2pt}200$ times lower than the ones of $\mathcal{N}$\hspace{-2pt}ADM$_2$), the HLSP sub-problem solve times are limited to under 0.2~s. The HLSP times for H-GUROBI are competitive.. The solver is efficiently warm-started by the primal of the full-rank level 5, such that levels below converge with zero iterations. In comparison, both the solvers $\mathcal{N}$\hspace{-2pt}ADM$_2$ and $\mathcal{N}$\hspace{-1pt}IPM$_2$ exhibit significantly faster HLSP solve times than H-MOSEK, which resolves every level of the hierarchy. This emphasizes the advantage of projection based methods for active constraints elimination in order to resolve sparse problems with lower number of non-zeros. This also indicates the high sparsity introduced by the turnback nullspace for under-actuated systems with virtual controls as described above.

The high accuracy solver $\mathcal{N}$\hspace{-1pt}IPM$_2$ achieves the greatest error reduction consistently throughout the priority levels, see Tab.~\ref{tab:solo12}. The corresponding robot torso trajectory and contact forces are depicted in Fig.~\ref{fig:solo12x} and Fig.~\ref{fig:solo12f}, respectively. It can be observed that for a shorter time horizon of 0.14~s, the posture task is infeasible and the torso is only moved around 0.025~m into the desired direction. The robot feet however are translated to their desired position, as can be seen from the corresponding error reduction of $f_{\text{ef}}$ in Tab.~\ref{tab:solo12}. With a longer time horizon of 0.35~s with $\Delta t=0.025$~s, the robot manages a larger torso transfer and shifts its body by the desired amount of 0.05~m.

\section{Conclusion}

In this article, we proposed several tools for efficiently solving prioritized trajectory optimization problems in robot motion planning.
We designed a threshold adaptation strategy in order to appropriately activate or deactivate SOI. This promotes optimality of the NL-HLSP and numerical stability when solving its HLSP approximation.  
We proposed the ADMM solver {\hbox{$\mathcal{N}$\hspace{-2pt}ADM$_2$}} for efficiently solving HLSP's. It is based on a reduced Hessian formulation with nullspace basis projections of active constraints. We directed our attention to problems of block-diagonal structure arising from optimal control formulations. Accordingly, we designed a sparsity leveraging turnback nullspace basis of upper bounded bandwidth for dynamics discretized by Euler integration. 

The proposed HLSP solver's efficiency was demonstrated within the S-HLSP framework to solve NL-HLSP's. The reduced Hessian formulation significantly reduces the number of non-zeros handled throughout the hierarchy with a sufficient number of priority levels. With a limited number of inner iterations, this enables a fast search for an optimal point of lower accuracy. 
We showed how such a point can be used to warm-start S-HLSP to find a high accuracy solution. Applicability to \ref{eq:oc} for robot scenarios with multi-stage constraints was demonstrated for a fully-actuated and under-actuated robots. In the latter case, we showed how the high sparsity of the turnback nullspace for fully-actuated robots can be maintained in the case of under-actuation. The SOI threshold adaptation strategy was shown to adjust to constraint infeasbility even from a far off starting point. 

Our method sparsely and once and for all eliminates the dynamics constraints from the HLSP sub-problems. In future work, we aim to extend the turnback algorithm to higher-order integration methods. Furthermore, as a non-recursive method, our method is associated with a high memory footprint as all control time-steps need to be handled in one big solution system. We therefore aim to implement recursive formulations of our method to achieve higher solver maturity, for example based on DDP principles. While this is associated with low bandwidth and high degree of sparsity, special attention needs to be paid with respect to multi-stage constraints and high number of priority levels, as the recursions need to be computed for every level. 

\section*{Funding}
This work is partly supported by the Schaeffler Hub for Advanced Research at Nanyang
Technological University, under the ASTAR IAF-ICP Programme ICP1900093. This work is
partly supported by the Research Project I.AM. through the European Union H2020 program
(GA 871899).

\section*{Conflict of interest disclosure}
The authors have no relevant financial or non-financial interests to disclose.

\section*{Data availability}
Data generated by our algorithms S-HLSP and the HLSP sub-problem solvers are available from the corresponding author on request.


\bibliography{../bib.bib}

\end{document}